\documentclass{article}

\usepackage{microtype}
\usepackage{graphicx}
\usepackage{subfig}
\usepackage{booktabs} 
\usepackage{comment}
\usepackage{hyperref}

\usepackage[accepted]{icml2022}

\usepackage{amsmath}
\usepackage{amssymb}
\usepackage{mathtools}
\usepackage{amsthm}

\usepackage[capitalize,noabbrev]{cleveref}

\theoremstyle{plain}
\newtheorem{theorem}{Theorem}[section]

\newtheorem{lemma}[theorem]{Lemma}
\newtheorem{corollary}[theorem]{Corollary}
\theoremstyle{definition}
\newtheorem{definition}[theorem]{Definition}

\theoremstyle{remark}
\newtheorem{remark}[theorem]{Remark}

\usepackage[textsize=tiny]{todonotes}

\usepackage{pgfplots}
\pgfplotsset{compat=1.16}

\usepackage{amsfonts}       
\usepackage{nicefrac}       
\usepackage{microtype}      

\usepackage{mathrsfs}
\usepackage{bm}
\usepackage{bbm}
\usepackage{comment}
\usepackage{multirow}
\usepackage{varwidth}
\usepackage{enumitem}

\usepackage{tabularx}

\usepackage{standalone}

\usepackage{letltxmacro}
\usepackage{dsfont}
\usepackage[scaled=1.15,mathscr]{urwchancal}
\usepackage{nccmath}

\usepackage[belowskip=-15pt,aboveskip=-15pt]{caption}

\usepackage{eqparbox}
\renewcommand\algorithmiccomment[1]{%
  \hfill\#\ \eqparbox{COMMENT}{#1}%
}

\usepackage{etoolbox}  
\makeatletter
\patchcmd{\algorithmic}{\addtolength{\ALC@tlm}{\leftmargin} }{\addtolength{\ALC@tlm}{\leftmargin}}{}{}
\makeatother

\renewcommand{\algorithmicrequire}{\textbf{Input: }} 
\renewcommand{\algorithmicensure}{\textbf{Output: }} 

\newcommand\algorithmicprocedure{\textbf{procedure}}
\newcommand{\algorithmicendprocedure}{\algorithmicend\ \algorithmicprocedure}
\makeatletter
\newcommand\PROCEDURE[3][default]{%
  \ALC@it
  \algorithmicprocedure\ \textsc{#2}(#3)%
  \ALC@com{#1}%
  \begin{ALC@prc}%
}
\newcommand\ENDPROCEDURE{%
  \end{ALC@prc}%
  \ifthenelse{\boolean{ALC@noend}}{}{%
    \ALC@it\algorithmicendprocedure
  }%
}
\newenvironment{ALC@prc}{\begin{ALC@g}}{\end{ALC@g}}
\makeatother

\makeatletter
\g@addto@macro\normalsize{%
  \setlength\abovedisplayskip{4pt}
  \setlength\belowdisplayskip{4pt}
  \setlength\abovedisplayshortskip{4pt}
  \setlength\belowdisplayshortskip{4pt}
}
\makeatother

\newcommand{\E}{\mathbb{E}}

\newcommand{\1}[1]{\mathbbm{1}{\left\{#1\right\}}}
\renewcommand{\P}{\mathbb{P}}
\renewcommand{\O}{O}
\newcommand{\hatbm}[1]{\hat{\bm{#1}}}

\newcommand{\ceil}[1]{\left\lceil#1\right\rceil}
\newcommand{\floor}[1]{\left\lfloor#1\right\rfloor}
\newcommand{\round}[1]{\left\lceil#1\right\rfloor}
\newcommand{\abs}[1]{\left\lvert#1\right\rvert}

\DeclareMathOperator*{\ue}{\normalfont\text{UE}}
\DeclareMathOperator*{\ie}{\normalfont\text{IE}}

\renewcommand{\ge}{\geqslant}
\renewcommand{\geq}{\geqslant}
\renewcommand{\le}{\leqslant}

\DeclareMathOperator*{\argmin}{arg\,min}
\DeclareMathOperator*{\KL}{KL}
\DeclareMathOperator*{\kl}{kl}
\DeclareMathOperator*{\Oracle}{Oracle}
\DeclareMathOperator*{\ERT}{\mathbb{E}[\normalfont \text{Reg}(T)]}

\makeatletter

\newsavebox\myboxA
\newsavebox\myboxB
\newlength\mylenA

\newcommand*\xoverline[2][0.75]{%
    \sbox{\myboxA}{$\m@th#2$}%
    \setbox\myboxB\null
    \ht\myboxB=\ht\myboxA%
    \dp\myboxB=\dp\myboxA%
    \wd\myboxB=#1\wd\myboxA
    \sbox\myboxB{$\m@th\overline{\copy\myboxB}$}
    \setlength\mylenA{\the\wd\myboxA}
    \addtolength\mylenA{-\the\wd\myboxB}%
    \ifdim\wd\myboxB<\wd\myboxA%
       \rlap{\hskip 0.5\mylenA\usebox\myboxB}{\usebox\myboxA}%
    \else
        \hskip -0.5\mylenA\rlap{\usebox\myboxA}{\hskip 0.5\mylenA\usebox\myboxB}%
    \fi}
\makeatother

\newfloat{procedure}{htbp}{loa}
\floatname{procedure}{Procedure}
\makeatletter\newcommand\l@procedure{\@dottedtocline{1}{1.5em}{2.3em}}\makeatother

\icmltitlerunning{Multiple-Play Stochastic Bandits with Shareable Finite-Capacity Arms}

\begin{document}

\twocolumn[
	\icmltitle{Multiple-Play Stochastic Bandits with Shareable Finite-Capacity Arms}



	\icmlsetsymbol{equal}{*}

	\begin{icmlauthorlist}
		\icmlauthor{Xuchuang Wang}{cuhk}
		\icmlauthor{Hong Xie}{cqu}
		\icmlauthor{John C.S. Lui}{cuhk}
	\end{icmlauthorlist}

	\icmlaffiliation{cuhk}{Department of Computer Science \& Engineering,
		The Chinese University of Hong Kong}
	\icmlaffiliation{cqu}{College of Computer Science, Chongqing University, China}

	\icmlcorrespondingauthor{Hong Xie}{xiehong2018@foxmail.com}

	\icmlkeywords{Multi-armed Bandits, Machine Learning, ICML}

	\vskip 0.3in
]



\printAffiliationsAndNotice{}  

\begin{abstract}
	We generalize the multiple-play multi-armed bandits (MP-MAB) problem with a \emph{shareable arms} setting,
	in which several plays can share the same arm.
	Furthermore, each shareable arm has a finite reward
	capacity and a ``per-load'' reward distribution, both of which are unknown to the learner.
	The reward from a shareable arm is load-dependent, which is the ``per-load'' reward
	multiplying either the number of plays pulling the arm,
	or its reward capacity when the number of plays exceeds the capacity limit.
	When the ``per-load'' reward follows a Gaussian distribution,
	we prove a sample complexity lower bound of learning the capacity
	from load-dependent rewards
	and	also a regret lower bound of this new MP-MAB problem.
	We devise a capacity estimator whose sample complexity upper bound matches the lower bound
	in terms of reward means and capacities.
	We also propose an online learning algorithm to address the problem and
	prove its regret upper bound.
	This regret upper bound's first term is the same as regret lower bound's,
	and its second and third terms also evidently correspond to lower bound's.
	Extensive experiments validate our algorithm's performance and
	also its gain in 5G \& 4G base station selection.

\end{abstract}


\section{Introduction}




Multi-armed bandits (MAB)~\citep{lai1985asymptotically,lattimore2020bandit} is a classic sequential decision making problem.
In the canonical MAB problem,
a learner sequentially pulls one arm from $K \in \mathbb{N}_+$ arms per time slot
and the pulled arm generates a stochastic reward whose mean is unknown to the learner.
To maximize the accumulative reward, the learner needs to either optimistically choose the arm with high uncertainty
in reward (exploration) or myopically select the one with high empirical mean reward (exploitation).
Multiple-play multi-armed bandits (MP-MAB)~\citep{anantharam1987asymptotically} generalizes the canonical MAB in that
the learner can select $N \in \{2,\ldots, K-1\}$ different arms out of $K$ arms in each time slot.

To model many real world applications,
one often needs to extend the simple MP-MAB
{where each arm can be assigned at most one play in each time slot.}
In this work, we consider arms with a \emph{shareable} nature:
an arm can be \emph{shared by several plays} in each time slot.
For example, consider a cognitive radio network~\citep{cai2018online} consisting of $K$ channels (arms) and $N$ secondary users (plays).
These so-called secondary users collaborate with each other and follow the rules set by the operator (learner).
The secondary users can transmit data via channels that are not occupied by primary users.
Each channel is available with a certain probability which is unknown to the operator.
The operator needs to repeatedly allocate $N$ secondary users to these $K$ channels, observe the availability of these  selected channels, and maximize the total amount of information transmission.
Since some of these channels may have high quality (bandwidth) that
can support the traffic demand of more than one secondary user, therefore, the operator can assign several secondary users to share a high quality channel,
especially when the channel also has a high availability rate.
Another application of our generalized MP-MAB is mobile edge computing, where each edge server (arm) may have multiple computing units (e.g., CPU cores), and thus can be shared by multiple users (plays).
A third application is in online advertisement placement, where one profitable advertisement (arm) may appear (be shared) at several different positions (plays) on a website.
In above examples, the learner can assign several plays to share a good arm.
Otherwise, the learner would not be able to utilize these arms' reward capacities and fail to maximize the total reward.




In this paper, we introduce a new bandit model in formalizing the \emph{shareable arms} setting
such that ``several plays can share the same arm''.
In our model, each arm $k$ is associated with a {``per-load'' reward} random variable $X_k$ and a finite reward capacity $m_k\in\mathbb{N}_+$,
both of which are \emph{unknown} to the learner.
An arm's reward is \emph{load-dependent:}
when $a_k$ plays are assigned to share the arm $k$, the reward is $\min\{a_k,m_k\}X_k$.
That is, if the number of plays $a_k$ is less than capacity $m_k$, the reward is linearly scaled as $a_kX_k$;
otherwise, it would be $m_kX_k$.
Rewards of different arms are independent.
In each time slot, the learner assigns these $N$ plays to $K$ arms according to an allocation (\emph{action}) in which each arm can be shared by several plays, and observes rewards from each selected arms separately (\emph{semi-bandit} feedback).
Both the reward $X_k$ and the capacity $m_k$ are not directly observable from the scaled feedback.
We call this problem as \emph{multiple-play multi-armed bandits with shareable arms} (\texttt{MP-MAB-SA}).
\texttt{MP-MAB-SA} uses the metric
\emph{regret}, i.e., the accumulative loss when comparing with an oracle which assigns its plays according to the optimal allocation,
and we aim to minimize the regret.

To illustrate the paper's results, we bring forward some notations here
and their formal definitions are deferred to Section~\ref{sec:model}.
Assume arm's ``per-load'' reward means are in a descending order.
Then, the optimal \(N\)-play allocation (action) is
assigning \(m_1\) plays to arm \(1\),
and \(m_2\) plays to arm \(2\), and so on, until there is no play left, that is,
\(
(m_1, m_2,\dots, m_{L-1}, \bar{m}_L, 0, \dots, 0),
\)
where \(L(\le N)\) is the least favored arm in the optimal action,
the number of plays pulling arm \(L\) is \(\bar{m}_L \coloneqq N-\sum_{k=1}^{L-1} m_k\)
--- the remaining plays after exploiting top \(L-1\) arms,
and \(\bar{m}_L\le m_L\).


We first examine the difficulty of learning capacity \(m_k\) from load-dependent rewards.
This task is different from common estimation tasks because the reward
samples --- depending on the number of plays on the arm --- are heterogeneous,
i.e., from different distributions.
We show that given the ``per-load'' reward is Gaussian, i.e., \(X_k\sim\mathcal{N}(\mu_k,1/2)\),
the task's sample complexity lower bound is \(\Omega(m_k^2/\mu_k^2\log (1/\delta))\):
to accurately learn an arm's capacity \(m_k\) with confidence \(1-\delta\),
one needs at least this number of explorations;
no matter how these explorations are conducted. (Section~\ref{subsec:sample-complexity-lower-bound})

We then study \texttt{MP-MAB-SA}'s regret lower bound.
Under consistent policies
and Gaussian ``per-load'' rewards,
the regret lower bound is \(
\Omega( (\sum_{k=L+1}^K {\Delta_{L,k}}/{\kl(\mu_k, \mu_L)} +
\sum_{k=1}^{L-1} {m_k^2}/{\mu_k^2} + { m_L^2}/{(m_L-\bar{m}_L+1)^2 \mu_L^2 } ) \log T),
\)
where \(\kl\) represents KL-divergence between two Gaussian distributions with the same variance,
and $\Delta_{i,j}\coloneqq \mu_i - \mu_j$ is the ``per-load'' reward mean difference between arm \(i\) and \(j\).
This lower bound clearly decomposes the cost in addressing \texttt{MP-MAB-SA}:
the first term is for distinguishing suboptimal arms, the second term is
for estimating top \(L-1\) optimal arms' reward capacities,
and the third term is for validating that arm \(L\)'s capacity \(m_L\) is no less than \(\bar{m}_L\).
(Section~\ref{subsec:regret-lower-bound})

We devise a capacity estimator based on uniform confidence intervals (UCI).
When the ``per-load'' rewards are either \([0,1]\) supported or Gaussian,
the estimator's sample complexity for accurately estimating the capacity
with a probability of at least $1-\delta$ is $\O((m_k^2/\mu_k^2)\log(1/\delta))$,
which matches the sample complexity lower bound in terms of reward mean \(\mu_k\) and capacity \(m_k\).
(Section~\ref{sec:learning_sharing_capacity})

We design the Orchestrative Exploration algorithm (\texttt{OrchExplore})
to address the \texttt{MP-MAB-SA} problem.
Its two procedures are carefully designed to reduce the regret,
implement our capacity estimator, and also address
the exploration-exploitation trade-off.
One procedure utilizes a parsimonious exploration idea:
in each time slot, at most one play is assigned to explore
while other plays are exploiting.
This idea could be traced back to \citet{anantharam1987asymptotically}, and
was recently made in~\citet{combes2015learning} for learning-to-rank algorithms
and also utilized in~\citet{wang2020optimal} for distributed bandits.
(Section~\ref{sec:opt_algorithm})

We prove that our \texttt{OrchExplore} algorithm achieves the regret
$\O((\sum_{k=L+1}^K \Delta_{L,k}/\kl(\mu_k, \mu_L) + \sum_{k=1}^{L-1} m_k^2/\mu_k^2 + {m_L^2}/{(m_L - \bar{m}_L+ 1)^2 \mu_L^2} )\log T)$,
where
\(\kl\) represents the KL-divergence between
two Bernoulli distributions in the \([0,1]\) supported case
or two Gaussian distributions with the same variances in the Gaussian case.
Its first term neatly matches the regret lower bound's first term
and its second and third terms also corresponds to the lower bound's.
(Section~\ref{sec:orchexplore_analysis})

We also conduct numeral simulations to validate the superior performance of \texttt{OrchExplore}
compared with other MAB algorithms (Section~\ref{sec:simulation})
and apply our algorithms to a 5G \& 4G base station selection problem (Appendix~\ref{app:real_world_simulation}).

\section{Related Works}
Since the seminal work by~\citet{lai1985asymptotically}, multi-armed bandits has been well studied in literature, especially in statistics and reinforcement learning
(cf. \citep{bubeck2012regret,slivkins2019introduction,lattimore2020bandit}).
MAB was then generalized to MP-MAB~\citep{anantharam1987asymptotically,gai2012combinatorial,chen2013combinatorial,kveton2015combinatorial,komiyama_optimal_2015}.
\citet{anantharam1987asymptotically} first studied MP-MAB and provided its asymptotically optimal regret analysis; \citet{gai2012combinatorial} considered a UCB-style algorithm for network applications; \citet{chen2013combinatorial} showed that CUCB can achieve a better regret bound than the one showed by~\citet{gai2012combinatorial}; \citet{komiyama_optimal_2015} proved that Thompson sampling achieved the optimal regret.
Our paper further generalizes stochastic MP-MAB so that it allows several plays to share the same arm.

There are many extensions of MP-MAB. The combinatorial bandits is the most popular one~\citep{cesa2012combinatorial,chen2013combinatorial,chen2016combinatorial,kveton2014matroid,gai2012combinatorial}) where
combinatorial action space and objective functions with some mild assumptions
were considered.
Another direction is to specialize MP-MAB to some applications
such as online website advertising, e.g., the cascade bandits~\citep{combes2015learning,kveton2015combinatorial,wen2017online}, multiple-play bandits with position-based click model~\citep{lagree2016multiple,komiyama2017position}, etc.
Recently, a new line of works considered the decentralized MP-MAB (multi-\emph{player} MAB)~\citep{anandkumar2011distributed,rosenski2016multi,bistritz2018distributed,wang2020optimal,magesh2021decentralized}). In this setting, players either cannot communicate with each other or their communication is highly restrictive, which adds difficulty in designing algorithms.
A decentralized version of \texttt{MP-MAB-SA} was also studied by the authors~\cite{wang2022multi}.
\section{Model Formulation}\label{sec:model}



Consider $K \in \mathbb{N}_+$ arms
indexed by $[K] \coloneqq \{1,2,\dots, K\}$.
Each arm $k\in [K]$ is characterized by $(m_k, X_k)$,
where $m_k \in \{1, \dots, N\}$ and $X_k$ is a random variable
with support in \([0,1]\), or it follows a Gaussian distribution
with the same variance \(\sigma^2\le 1/2\) for all arms.
Here, the integer $m_k$ models the finite \textit{reward capacity} of arm $k$
(mapping to real world applications is presented in Appendix~\ref{appsub:motive_m}).
The $X_k$ models the \textit{``per-load'' stochastic reward} of arm $k$,
whose mean is denoted as
$
    \mu_k
    \coloneqq
    \mathbb{E} [ X_k ].
$
We assume that the reward mean $\mu_k$ are distinct
and without loss of generality, they are descending ordered as
$
    \mu_1 > \mu_2 >\dots >\mu_K.
$
This ordering is unknown to the learner.


Consider $T \in \mathbb{N}_+$ time slots.
At each time slot
$t \le T$,
the learner assigns $N \in \mathbb{N}_+$ plays to $K$ arms \((N<K)\).
Let $a_{k,t} \in \{0, 1, \ldots,N\}$ denote the number of
plays assigned to arm $k$ in time slot $t$.
All $N$ plays are assigned in each time slot,
i.e., $\sum^K_{k=1} a_{k,t} = N$.
Denote the action in time slot $t$ as
$
    \bm{a}_t \coloneqq (a_{1,t}, a_{2,t},\ldots,a_{K,t}).
$
The action space $\mathscr{A}$ is
\begin{equation}
    \label{eq:action_space}
    \mathscr{A}
    \coloneqq
    \Big\{
    (a_1,a_2,\ldots,a_K) \in \mathbb{N}^K :
    \sum\nolimits_{k \in [K] } a_k = N
    \Big\}\!.
\end{equation}
At the end of time slot $t$,
the learner receives a reward $R_k(a_{k,t})$
from assigning $a_{k,t}$ plays to arm $k$,
which is \emph{independent} across arms and time slots.
To capture the reward capacity's nature of applications like edge computing and
cognitive radio network (details are in Appendix~\ref{appsub:motive_R}), we consider the following \emph{load-dependent} reward $R_k(a_{k,t})$:
\begin{equation}\label{eq:arm_level_feedback}
    R_k(a_{k,t})\coloneqq \min\{a_{k,t},m_k\}\cdot X_k.
\end{equation}
Eq.(\ref{eq:arm_level_feedback}) captures the threshold property
of the reward capacity:
if $a_{k,t} < m_k$, the load-dependent reward random variable is $a_{k,t} X_k$,
and if $a_{k,t} \ge m_k$, it is $m_k X_k$.
As a counterpart to ``per-load'' reward mean \(\mu_k\), we name \(m_k\mu_k\) as the \emph{``full-load'' reward mean}.
The multiplier $\min\{a_{k,t},m_k\}$ represents how many capacities of arm $k$ are utilized
by $a_{k,t}$ plays, and has no restriction on how these capacities are distributed among plays.
For any action $\bm{a}_t \in \mathscr{A}$,
the expected total reward to the learner is
\begin{equation*}
    \begin{split}
        f(\bm{a}_t)
        {\coloneqq}  \mathbb{E}  \Big[\!\sum\nolimits_{k\in [K]}\!R_k(a_{k,t}) \!\Big]
        {=} \sum\nolimits_{k\in [K]}\!\min\{a_{k,t}, m_k\}\mu_k.
    \end{split}
\end{equation*}
The learner only observes rewards from arms with at least one play.
She neither knows the capacity $m_k$,
nor whether the number of assigned plays $a_{k,t}$
is greater than $m_k$ or not.

The optimal action for maximizing the expected reward $f(\bm{a}_t)$ is
to assign $m_1$ plays to arm $1$, $m_2$ plays to arm $2$,
and so on, until there is no play left.
Let $\bm{a}^*$ denote this optimal action and it can be expressed as
\begin{equation}
    \label{eq:optimal_action}
    \bm{a}^*
    \coloneqq
    \Big(
    \!m_1, \ldots, m_{L-1}, N - \sum\nolimits_{k=1}^{L-1}m_k, 0, \ldots , 0
    \!\Big),
\end{equation}
where $L$ denotes the smallest number of top arms covering $N$ plays and it can be expressed as
\begin{equation}\label{eq:critical_top_arms}
    L \coloneqq\min\Big\{n:\sum\nolimits_{k=1}^n m_k \ge N\Big\}.
\end{equation}
These \(L\) arms are called \emph{optimal arms}, while the rest are called \emph{suboptimal arms},
and arm \(L\) is called \emph{least favored optimal arm}.
We denote \(\bar{m}_L \coloneqq N - \sum\nolimits_{k=1}^{L-1}m_k\) as the number of plays pulling arm \(L\)
in the optimal action.
The optimal action $\bm{a}^*$ is unknown to the learner.
We define regret as the learner's total loss when comparing with $\bm{a}^\ast$,
\begin{equation*}
    \text{Reg}(T)
    \coloneqq
    \sum\nolimits_{t=1}^T
    \left(f(\bm{a}^*)-f(\bm{a}_t)\right),
\end{equation*}
Our objective is designing algorithms to minimize the expected regret $\ERT$.
\section{Fundamental Limits of \texttt{MP-MAB-SA} }\label{sec:lower_bound}

In this section, we consider the learning limits of the \texttt{MP-MAB-SA} problem when the ``per-load''
rewards are Gaussian.
We first focus on the capacity learning task
and rigorously prove its sample complexity lower bound.
Then, relying on this new sample complexity result, we prove a nontrivial lower bound on the regret of \texttt{MP-MAB-SA}.

Except that the sample complexity and regret lower bounds in this section are only for the Gaussian rewards,
all other theoretical results in the paper apply for
both the \([0,1]\) supported random reward
and
the Gaussian reward.



\subsection{Sample Complexity Lower Bound}\label{subsec:sample-complexity-lower-bound}

The challenges of learning capacity \(m_k\) lie in the load-dependent reward feedback (Eq.(\ref{eq:arm_level_feedback})) and heterogeneous explorations.
As the feedback depends on the random variable \(X_k\) multiplying the uncertain factor \(\min\{a_{k,t}, m_k\}\), one cannot easily discern whether the number of plays \(a_{k,t}\) is greater than the capacity \(m_k\) or not, let alone the capacity \(m_k\).
Furthermore, the shareable arm setting allows any number of plays to pull an arm
--- heterogeneous explorations, which further complicates the learning task.
We show that the task can be reduced to hypotheses testing,
which is a key step in deriving the lower bound.




\begin{theorem}[Sample Complexity Minimax Lower Bound]\label{thm:sample_lower_bound}
    Assume arm \(k\)'s ``per-load'' reward \(X_k\) follows the Gaussian distribution \(\mathcal{N}(\mu_k, \sigma_k^2)\),
    where \(\sigma_k^2\) is the variance, and that \(\mu_k^2/m_k^2\sigma_k^2\ge 2\).
    If the exploration times\footnote{One exploration can have any number of plays pulling the same arm.} \(n\) of arm \(k\) is less than \(({\sigma_k^2 m_k^2}/{\mu_k^2})\log\left( 1/{4\delta} \right),\)
    then
    the probability of falsely estimating the capacity is no less than \(\delta\), or formally, \[
        \P\left(\hat{m}_{k} \neq m_k\vert n \le ({\sigma_k^2 m_k^2}/{\mu_k^2})\log\left( 1/{4\delta} \right)\right)
        \ge \delta,
    \]
    where \(\hat{m}_k\) is any possible estimator that one can design.

    Similar, we also have a sample complexity lower bound for identifying whether an arm's unknown capacity \(m_k\)
    is no less than the integer \(d(\ge 2)\) or not.
    Assume that \((m_k-d+1)^2\mu_k^2/(m_k^2\sigma_k^2) \ge 2\log(N/(d-1))\).
    If the exploration times \(n\) of arm \(k\) is less than \[({\sigma_k^2 m_k^2}/{(m_k-d+1)^2\mu_k^2})\log\left( 1/{4\delta} \right),\]
    then the probability of falsely identifying whether the capacity \(m_k\) is greater than \(d\) or not is no less than \(\delta\).
\end{theorem}

\begin{proof}[Proof of Theorem~\ref{thm:sample_lower_bound}]
    We provide the proof of the first sample complexity result in three steps.
    The second statement's proof is similar to the first's (see Appendix~\ref{appsub:sample-complexity-lower-bound-second-part}).

        {\bf Step 1: reduce the task to hypothesis testing. }
    The original task is, given a number of observations,
    to find the capacity \(m_k\) among its \(N\) potential integer values \(\{1,2,\dots, N\}\).
    We reduce the original task to find the capacity \(m_k\) from
    a binary subset \(\{m_k^{(0)}, m_k^{(1)}\}\) of \(\{1,2,\dots, N\}\) which contains \(m_k\).
    This reduced task is simpler than the original task and
    its sample complexity no greater than the original one's.

    Denote \(n\) as the exploration times of arm \(k\) and
    \(h_n\coloneqq \{a_{k,1}, a_{k,2},\dots, a_{k,n}\}\) as the sequence of
    the number of plays pulling arm \(k\)
    in these $n$ explorations.
    Define two load-dependent reward random variables as Eq.(\ref{eq:arm_level_feedback}):
    \(Z^{(0)}_k (a) \coloneqq \min\{a,m_k^{(0)}\}\cdot X_k\) and
    \(Z^{(1)}_k (a) \coloneqq \min\{a,m_k^{(1)}\}\cdot X_k.\)
    If \(a>m_k^{(0)}\), then $Z^{(0)}_k (a)$ follows a probability distribution
    \(\mathcal{N}( m_k^{(0)} \mu_k, ( m_k^{(0)}\sigma_k )^2)\),
    while if \(a<m_k^{(0)}\),  $Z^{(0)}_k (a)$ follows \(\mathcal{N}(a \mu_k, a^2\sigma_k^2)\).
    The \(Z^{(1)}_k(a)\) is similar.
    Denote \(\mathbb{P}_0^a\) and \(\mathbb{P}_1^a\) as
    probability measures induced by $ Z^{(0)}_k (a)$ and $Z^{(1)}_k (a)$ respectively.
    Denote \(\P_i^{\otimes h_n}\) as the product measure of \(\P_i^{a_{k,1}}, \ldots, \P_i^{a_{k,n}}\),
    where \(i=0,1\).
    Formally, this reduced task becomes:
    given \(n\) samples from
    an arbitrary exploration sequence \(h_n=(a_{k,1},\dots, a_{k,n})\),
    to distinguish the hypotheses between
    \begin{align*}
         &
        H_0: (R_k(a_{k,1}), \dots, R_k(a_{k,n}))\sim \P_0^{\otimes h_n},
        \\
         &
        H_1: (R_k(a_{k,1}), \dots, R_k(a_{k,n}))\sim \P_1^{\otimes h_n}.
    \end{align*}

    {\bf Step 2: apply the Le Cam's method. }
    We apply a version of Le Cam's method~\citep[Theorem 2.2]{tsybakov_introduction_2008} to this hypothesis testing problem as follows:
    \[\begin{split}
            \inf_{\hat{m}_k}\max \left(\P_0^{\otimes h_n}(\hat{m}_k=m_k^{(1)}), \P_1^{\otimes h_n}(\hat{m}_k=m_k^{(0)})\right)\\
            \ge \frac 1 4 \exp\left(-\KL(\P_0^{\otimes h_n}, \P_1^{\otimes h_n})\right),
        \end{split}
    \]
    where \(\inf_{\hat{m}_k}\) is taken over all estimators \(\hat{m}_k\),
    and \(\KL\) is the standard KL-divergence.

    {\bf Step 3: calculate the KL divergence. }
    The measure \(\P_0^{\otimes h_n}\) is a product of \(n\) independent probability measures,
    each of which depends on one entry of sequence \(h_n\).
    We denote \(n_l\) as the number of times that arm \(k\) pulled by \(l\in\{1,\dots,N\}\) plays among the sequence \(h_n\), i.e.,
    \(n_l \coloneqq \sum_{t=1}^n \1{a_{k,t}=l}\).
    Assume \(m_k^{(0)} < m_k^{(1)}\) since one can rotate their order.
    Then, we can decompose the KL divergence as follows:
    \begin{equation}\label{eq:kl-decompose}
        \begin{split}
            &\KL(\P_0^{\otimes h_n}, \P_1^{\otimes h_n})
            = \sum\nolimits_{l=1}^{N} n_l\KL(\P_0^l,\P_1^l)\\
            =& \sum\nolimits_{l=1}^{m_k^{(0)}}\!\! n_l\!\KL(\P_0^l,\P_1^l) {+} \sum\nolimits_{l=m_k^{(0)}+1}^{m_k^{(1)}} \!\!n_l\!\KL(\P_0^l,\P_1^l)\\
            &+ \sum\nolimits_{l=m_k^{(1)}+1}^{N} n_l\KL(\P_0^l,\P_1^l).
        \end{split}
    \end{equation}
    When the number of plays \(l\) is between \(\{1,\dots, m_k^{(0)}\}\), both probability measures
    \(\P_0^l\) and \(\P_1^l\) are induced by the
    same random variable \(l\cdot X_k\).
    So their KL divergence is equal to \(0\), i.e., \(\KL(\P_0^l,\P_1^l)_{\{0<l\le m_k^{(0)}\}} = 0\).
    When \(l\) is between \(\{m_k^{(0)}+1,\dots,m_k^{(1)}\}\), \(\P_0^l\) and \(\P_1^l\)
    are induced by
    \(m_k^{(0)}\cdot X_k\) and \(l\cdot X_k\) respectively.
    When \(l\) is between \(\{m_k^{(1)}+1,\dots,N\}\), \(\P_0^l\) and \(\P_1^l\) are induced by
    \(m_k^{(0)}\cdot X_k\) and \(m_k^{(1)}\cdot X_k\) respectively.
    In Appendix~\ref{appsub:detail_kl_upper_bound}, we show for \(X_k\sim\mathcal{N}(\mu_k,\sigma_k^2)\) and the binary set as \((m_k^{(0)}, m_k^{(1)})=(m_k-1,m_k)\) or \((m_k, m_k+1)\),  these KL-divergence terms obey the following inequality \(
    \KL(\P_0^l,\P_1^l)_{\{\!m_k^{(0)}<l\le m_k^{(1)}\!\}}
    {\le} \KL(\P_0^l,\P_1^l)_{\{\!m_k^{(1)}<l\le N\!\}} {\le} \frac{\mu_k^2}{m_k^2 \sigma_k^2},
    \)
    where the last inequality needs the \(\mu_k^2/m_k^2\sigma_k^2 \ge 2\) condition.
    Substituting these three terms of Eq.(\ref{eq:kl-decompose})'s RHS,
    we have
    \begin{equation}\label{eq:not-m-explore}
        \KL(\P_0^{\otimes h_n}, \P_1^{\otimes h_n}) \le \sum\nolimits_{l=m_k^{(0)}+1}^{N}\!\! n_l\frac{\mu_k^2}{m_k^2 \sigma_k^2} \le
        \frac{n\mu_k^2}{m_k^2 \sigma_k^2}.
    \end{equation}

    Then, we substitute Eq.(\ref{eq:not-m-explore}) into Step 2's result and obtain
    \(
    \inf_{\hat{m}_k}\max (\P_0^{\otimes h_n}(\hat{m}_k=m_k^{(1)}), \P_1^{\otimes h_n}(\hat{m}_k=m_k^{(0)}))
    \ge \frac 1 4 \exp(-\frac{n\mu_k^2}{m_k^2 \sigma_k^2}).
    \)
    Letting the inequality's RHS greater than
    the failure probability \(\delta\) leads to \[
        n\le ({\sigma_k^2 m_k^2}/{\mu_k^2})\log\left( {1}/{4\delta} \right).
    \]
    It means that if the number of times of explorations \(n\) is no greater than \(({\sigma_k^2 m_k^2}/{\mu_k^2})\log\left( {1}/{4\delta} \right)\), then the probability of falsely estimating the capacity --- either \(\P_0^{\otimes h_n}(\hat{m}_k=m_k^{(1)})\) or \(\P_1^{\otimes h_n}(\hat{m}_k=m_k^{(0)})\) --- would be no less than \(\delta\).
\end{proof}

Theorem~\ref{thm:sample_lower_bound} states that to correctly estimate an arm's capacity with \(1-\delta\) confidence, one needs at least \(\Omega((m_k^2/\mu_k^2)\log(1/\delta))\) times of explorations.
In Section~\ref{sec:learning_sharing_capacity}, we devise an estimator whose sample complexity upper bound matches the lower bound
in terms of reward mean \(\mu_k\) and capacity \(m_k\),
which implies that this lower bound is tight.
Theorem~\ref{thm:sample_lower_bound}'s second result can
depict the difficult of validating whether arm \(L\)'s capacity \(m_L\)
is no less than \(\bar{m}_L\).


\begin{remark}\label{rmk:lower-bound-key-remark}
    When the binary set's elements are chosen as \(m_k^{(0)}=m_k, m_k^{(1)} > m_k\),
    the number of explorations \(n\) in Eq.(\ref{eq:not-m-explore})'s RHS can be strengthened to \(n'\coloneqq \sum_{l>m_k}n_l\) (see Eq.(\ref{eq:not-m-explore})'s middle term).
    It means that --- with well-selected hypotheses ---
    the upper bound of \(\KL(\P_0^{\otimes h_n}, \P_1^{\otimes h_n})\) may only depend
    on the number of explorations whose number of plays is greater than \(m_k\).
    So, Theorem~\ref{thm:sample_lower_bound}'s first result can be enhanced to \[
        \P\left(\hat{m}_{k} \neq m_k\vert n' \le ({\sigma_k^2 m_k^2}/{\mu_k^2})\log\left( 1/{4\delta} \right)\right)
        \ge \delta.
    \]
    Similar improvement can also be made in the second result via choosing the binary set as \(\{m_k, d\}\).
    Note that all of these \(n'\) ``irregular'' explorations contribute costs to regret.
    This is a critical observation for the regret lower bound's proof.
\end{remark}

\subsection{Regret Lower Bound}\label{subsec:regret-lower-bound}
Next, we provide an asymptotical regret lower bound for the \texttt{MP-MAB-SA} problem.
Its full proof is in Appendix~\ref{app:regret_lower_bound}.


\begin{theorem}[Regret Lower Bound]~\label{thm:regret_lower_bound}
    For any consistent
    algorithm (please refer to Definition~\ref{def:consistent})
    to address a \(K\)-armed \emph{\texttt{MP-MAB-SA}} problem
    whose ``per-load'' rewards follow Gaussian distributions with the same variance \(\sigma^2\le 1/2\),
    and
    whose least favored arm \(L\) is shared by more than one play in its optimal action \(\bm{a}^*\), i.e., \(\bar{m}_L = N-\sum_{k=1}^{L-1}m_k>1\),
    and assume that \(\mu_k^2/m_k^2\sigma^2\ge 2\) for all arm \(k(<L)\)
    and \((m_k-\bar{m}_L + 1)^2\mu_L^2/m_L^2\sigma^2 \ge 2\log(N/(\bar{m}_L - 1))\),
    then its regret is lower bounded as follows:
    \begin{equation*}
        \begin{split}
            &\liminf_{T\to \infty}\frac{\ERT}{\log T}
            \ge  \sum_{k=L+1}^K \frac{\Delta_{L,k}}{\kl(\mu_k, \mu_L)} \\
            &\qquad\qquad +
            \sum_{k=1}^{L-1}\frac{\Delta_{k,L}\sigma^2m_k^2}{\mu_k^2} + \frac{\Delta_{L,L+1}\sigma^2m_L^2}{(m_L - \bar{m}_L + 1)^2\mu_L^2},
        \end{split}
    \end{equation*}
    where $\kl$ is KL-divergence between two Gaussian distributions with the same variance.
\end{theorem}



The regret lower bound's first term and last two terms
are \textit{orthogonal}.
Because
the first term is due to distinguishing suboptimal arms,
while the second term is from learning top \(L-1\) optimal arms' capacities
and the third term corresponds to identifying that arm \(L\)'s capacity \(m_L\) is no less than \(\bar{m}_L\).
The last two terms are quantified by Theorem~\ref{thm:sample_lower_bound}'s sample complexity lower bound
(see Remark~\ref{rmk:lower-bound-key-remark} also).
Although the last two terms hold only for the Gaussian rewards,
the first term also holds for any \([0,1]\) supported stochastic rewards,
where \(\kl\) becomes the KL-divergence between two Bernoulli distributions.
\section{Learning Reward Capacity}\label{sec:learning_sharing_capacity}

In this section,
we derive reward capacities' uniform confidence intervals (UCI), develop a capacity estimator,
and analyse the estimator's sample complexity.
Proofs of this section are deferred to Appendix~\ref{app:capacity_proof}.

Our estimation is built on two kinds of explorations:
(1) \textbf{individual exploration (IE)}, i.e., when an arm is played by a number of plays \(a_{k,t}\) below its capacity \(m_k\),
and
(2) \textbf{united exploration (UE)}, i.e., when the number of plays exceeds its capacity.
When $a_{k,t}< m_k$, the observed reward divided by $a_{k,t}$
is a sample of ``per-load'' reward  $X_k$ and
can be used to estimate its mean $\mu_k$.
When $a_{k,t}\ge m_k$, the observation is from the ``full-load'' reward $m_k X_k$ and
can estimate its mean $m_k\mu_k$.
Note that one cannot distinguish both cases from the reward observations.
To separate them, one coarse approach
is exploring with extreme number of plays, i.e., assign \(1(<m_k)\) play for IEs
or \(N(\ge m_k)\) plays for UEs.
Later, our algorithm (at Section~\ref{sec:opt_algorithm}) employs the capacity's confidence bounds to better differentiate them.

Denote $\tau_{k,t}$ as the number of IEs for arm \(k\) up to time $t$,
$S_{k,t}^{\text{IE}}$ as the associated total ``per-load'' rewards,
and $\hat{\mu}_{k,t}$ as the ``per-load'' reward's sample mean:
\(\tau_{k,t} \coloneqq   \sum_{s=1}^t
\1{a_{k,s} {<} m_k},
S_{k,t}^{\text{IE}} \coloneqq \sum_{s=1}^t \frac{R_{k,s}}{a_{k,s}}\1{a_{k,s} {<} m_k}\),
and \(\hat{\mu}_{k,t} \coloneqq {S_{k,t}^{\text{IE}}}/{\tau_{k,t}}\).
Similarly, we define $ \iota_{k,t}$, $S_{k,t}^{\text{UE}} $, and ``full-load'' reward's sample mean $\hat{\nu}_{k,t}$ for UEs: \(\iota_{k,t} \coloneqq
\sum_{s=1}^t  \1{a_{k,s} {\ge} m_k}, S_{k,t}^{\text{UE}} \coloneqq  \sum_{s=1}^t  R_{k,s} \1{a_{k,s} {\ge} m_k},\) and \(\hat{\nu}_{k,t} \coloneqq  {S_{k,t}^{\text{UE}}}/{\iota_{k,t}}.\)

\begin{lemma}[Uniform Confidence Interval (UCI) for Reward Capacity $m_k$]\label{lma:ufc_m}
    Denote the function
    $
        \phi(x,\delta) \coloneqq \sqrt{\left(1+\frac{1}{x}\right)\frac{\log(2\sqrt{x+1}/\delta)}{2x}}.
    $
    For any arm \(k\), conditioned on the assumption\footnote{
        This assumption on \(\delta\) is also required in Lemma~\ref{lma:learning_m_criterion}
        and Theorem~\ref{cor:m_sample_complexity}, where this assumption is omitted.
    } that $\phi(\tau_{k,t},\delta) + \phi(\iota_{k,t},\delta) < \hat{\mu}_{k,t}$, the event
    \begin{equation*}
        \begin{split}
            \{\forall t \in \mathbb{N}^+,
            m_k \in
            [{\hat{\nu}_{k,t}}/({\hat{\mu}_{k,t} + \phi(\tau_{k,t},\delta) + \phi(\iota_{k,t},\delta)}),\\
            {\hat{\nu}_{k,t}}/({\hat{\mu}_{k,t} - \phi(\tau_{k,t},\delta) - \phi(\iota_{k,t},\delta)})] \}
        \end{split}
    \end{equation*}
    holds with a probability of at least $1-\delta$.
\end{lemma}

Lemma~\ref{lma:ufc_m} states a sequence of confidence intervals that is uniformly valid over an unbounded time horizon with fixed confidence \(1-\delta\).
Although one can also apply Hoeffding's inequality to construct such a uniform interval,
our approach provides a \emph{``shaper concentration''}
in some instances (see Appendix~\ref{sec:hfd_ci}).

Notice that reward capacity $m_k$ is an integer.
If the ceiling of lower confidence bound is equal to the floor of upper confidence bound,
i.e., only one integer inside the interval,
then this integer is the estimated capacity.
We denote them as the final confidence bounds of \(m_k\) as follows:
\begin{align}
    {m}_{k,t}^{l} & \coloneqq
    \max\!\left\{\!\ceil{{\hat{\nu}_{k,t}}/{(\hat{\mu}_{k,t} + \phi(\tau_{t},\delta) + \phi(\iota_{t},\delta))}}\!,
    1 \right\}\!,    \label{eq:m_lower_bound} \\
    {m}_{k,t}^{u} & \coloneqq
    \min\!\left\{\!\floor{{\hat{\nu}_{k,t}}/{(\hat{\mu}_{k,t} - \phi(\tau_{t},\delta) - \phi(\iota_{t},\delta))}}\!,
    N\!\right\}\!.\label{eq:m_upper_bound}
\end{align}

\begin{lemma}[Reward Capacity Estimator]\label{lma:learning_m_criterion}
    For any arm $k$ and time slot \(t\), if the capacity \(m_k\)'s upper and lower confidence bounds
    are equal, i.e., \(m_{k,t}^l=m_{k,t}^u\),
    then
    the probability of correctly estimating $m_k$ is at least $1-\delta$, i.e.,
    \[
        \mathbb{P}({\hat{m}_{k,t}} = m_k\vert m_{k,t}^l = m_{k,t}^u) \ge 1 - \delta,
    \]
    where the estimator \(\hat{m}_{k,t}\) is defined as \({m}_{k,t}^{l}\).
\end{lemma}

Lemma~\ref{lma:learning_m_criterion} identifies conditions that the capacity estimate is correct with a high confidence and defines our estimator.
From the criterion \(m_{k,t}^l = m_{k,t}^u\), we derive a sample complexity result for our capacity estimator.



\begin{theorem}[Estimator's Sample Complexity Upper Bound]\label{cor:m_sample_complexity}
    For any arm $k$, time slot \(t\), and \(0<\delta \le 2\exp(-49m_k^2 / \mu_k^2)\),
    if the number of IEs \(\tau_{k,t}\) and UEs \(\iota_{k,t}\) are both no less than
    \((49m_k^2/\mu_k^2)\log (2/\delta)\),
    then the estimator in Lemma~\ref{lma:learning_m_criterion} is correct with confidence \(1-\delta\),
    i.e.,
    \[
        \mathbb{P}
        \left( {\hat{m}_{k,t}} = m_k \vert \tau_{k,t}, \iota_{k,t} \geq (49m_k^2/\mu_k^2)\log (2/\delta) \right)
        \ge 1 - \delta.
    \]
    Similar, we also have a sample complexity upper bound for identifying
    whether an arm's capacity \(m_k(\ge d)\)\footnote{Note that \(m_k\ge d\) is unknown a priori.}
    is no less than
    an integer \(d(\ge 2)\) or not.
    For \(0<\delta \le 2\exp(-49m_k^2/(m_k-d+1)^2\mu_k^2)\),
    if the number of IEs \(\tau_{k,t}\) and UEs \(\iota_{k,t}\) are both no less than
    \[(49m_k^2/(m_k-d+1)^2\mu_k^2)\log (2/\delta),\]
    then from the criterion that capacity \(m_k\)'s lower confidence bounds \(m_{k,t}^l\) is no less than \(d\),
    i.e., \(m_{k,t}^l \ge d\),
    one can correctly identify that capacity \(m_k\) is no less than
    an integer \(d\) with confidence \(1-\delta\).
\end{theorem}
Theorem~\ref{cor:m_sample_complexity} shows that
our estimator requires at most $\O((m_k^2/\mu_k^2)\log(1/\delta))$ number of IEs and UEs for arm $k$ to have
a correct capacity estimate with confidence $1-\delta$.
Comparing this to the sample complexity lower bound for Gaussian rewards
in Theorem~\ref{thm:sample_lower_bound} shows that
our sample complexity upper bound in Theorem~\ref{cor:m_sample_complexity} is tight in terms of reward mean \(\mu_k\) and reward capacity \(m_k\)
and our estimator in Lemma~\ref{lma:learning_m_criterion} is near optimal for Gaussian rewards.
The second result in Theorem~\ref{cor:m_sample_complexity} is prepared for validating arm \(L\)'s capacity \(m_L\) is
no less than \(\bar{m}_L\) in regret analysis.
\section{The Orchestrative Exploration Algorithm}\label{sec:opt_algorithm}

The capacity estimator designed in Section~\ref{sec:learning_sharing_capacity}
needs two kinds of observations:
``per-load'' reward samples from IEs
and
``full-load'' reward samples from UEs.
To acquire these observations with lower regret cost,
we devise
\textbf{parsimonious individual exploration (PIE)} and \textbf{parsimonious united exploration (PUE)}.
PIE and PUE also address the exploration-exploitation
trade-off.
We first present the details of PIE and PUE in the next two subsections,
then use them as procedures in the \texttt{OrchExplore} algorithm.



\textbf{Notations.}
We use bold notations to represent \(K\)-dim vectors,
e.g.,
\(\bm{\mu} = (\mu_1, \mu_2, \dots,\mu_K)\) represents
all \(K\) arms' average ``per-load'' rewards.
We use ``\(~\hat{~}~\)'' above a symbol to represent an estimate.
For example, \(\hat{\mu}_{k,t}\) is the empirical mean estimate of arm \(k\)'s reward in time slot \(t\).
Especially,
instead of using the number of times
of IEs and UEs
\(\tau_{k,t}\) and \(\iota_{k,t}\) (\(m_k\) is unknown),
\texttt{OrchExplore} uses
the number of \emph{effective} times of IEs and UEs:
\(\hat{\tau}_{k,t} \coloneqq   \sum_{s=1}^t
\1{a_{k,s} < m_{k,s}^l}\)
and
\(\hat{\iota}_{k,t} \coloneqq
\sum_{s=1}^t  \1{a_{k,s} \ge m_{k,s}^u}\)
(\(m_{k,t}^l\) and \(m_{k,t}^u\) are known),
where \emph{effective} means
that these IEs and UEs
are conducted with awareness
by \texttt{OrchExplore}.
\(\hat{\tau}_{k,t}\)
and \(\hat{\iota}_{k,t}\) are underestimates of \(\tau_{k,t}\)
and \(\iota_{k,t}\).
In \texttt{OrchExplore},
the ``per-load'' reward mean estimate \(\hat{\mu}_t\)
and ``full-load'' reward mean estimate \(\hat{\nu}_t\)
are also based on these effective explorations' observations.
The function \(\Oracle\) is a mapping from
an \texttt{MP-MAB-SA} problem's ``per-load'' reward means \(\bm{\mu}\) and reward capacities \(\bm{m}\)
to its optimal action.
That is, first assign the best arm with the number of plays that is equal to its capacity,
then the second best arm, and so on, until there is no play left
(e.g., the optimal action \(\bm{a}^*\) in Eq.(\ref{eq:optimal_action})).

\subsection{Parsimonious Individual Exploration (PIE)}

To reduce IEs' costs,
PIE utilizes two core ideas:
(1) when exploring/exploiting empirical optimal arms, it assigns as many plays as possible;
(2) when exploring empirical suboptimal arms, it only assigns a single play.
The deliberate exploration in (2)
should also be rare
since pulling empirical suboptimal arms can be expensive.
Next, we show how both ideas are realized.

\textbf{Explore empirical optimal arms. }
We need to
identify empirical optimal arms and
decide the appropriate number of plays pulling these arms.
The largest number of plays pulling an arm should be equal to its capacity's lower confidence bound \(m_{k,t}^l\)
so as to effectively acquire the arm's ``per-load'' reward observations.
To achieve that, we input arms' reward capacities' lower bounds \(\bm{m}_t^l\) and
empirical reward means \(\hatbm{\mu}_t\)
to the \(\Oracle\) function.
Its output \(\bm{a}_t^{\ie}\) would assign the empirical best arm
with the number of plays that is equal to its capacity lower confidence bound,
and then the empirical second best arm,
and so on, until no play left.
Denote \(\mathcal{S}_t\) as the set of empirical optimal arms chosen in \(\bm{a}_t^{\ie}\),
i.e., \(\mathcal{S}_t\coloneqq \{k:{a}_{k,t}^{\ie}>0\}\) and
\(L_t\coloneqq \argmin_k\{\hat{\mu}_k:k\in\mathcal{S}_t\}\) as the empirical least favored
optimal arm in \(\mathcal{S}_t\).

\textbf{Explore empirical suboptimal arms. }
We use arm's KL-UCB index~\cite{cappe_kullback-leibler_2013}
to indicate empirical suboptimal arms that need more explorations
--- a subset of empirical suboptimal arms whose KL-UCB indexes \(u_{k,t}\) are no less than
the least favored arm \(L_t\)'s
empirical mean \(\hat{\mu}_{L_t,t}\),
denoted
as \(\mathcal{E}_t \coloneqq \{k\not\in \mathcal{S}_t: u_{k,t} \ge \hat{\mu}_{L_t, t}\}\).
The KL-UCB index \(u_{k,t}\) of arm \(k\) at time slot \(t\) is defined as
\(u_{k,t}\coloneqq \sup\{q\ge 0:\hat{\tau}_{k,t}\kl(\hat{\mu}_{k,t}, q)\le \log(t) + 4 \log \log (t)\}.\)
To make the deliberate explorations as rare events,
PIE implements the following rule:
with a probability of \(1/2\),
the algorithm uniformly select an arm from \(\mathcal{E}_t\) (if not empty) and assign one play,
which otherwise would have pulled the arm \(L_t\), so to explore this arm;
otherwise, this round of PIE will not explore empirical suboptimal arms.

After obtaining \(\bm{a}_t^{\ie}\) from \(\Oracle(\hatbm{\mu}_t, \bm{m}_t^l)\) and
--- with a \(1/2\) probability --- rearranging one play of \(\bm{a}_t^{\ie}\)
to explore an empirical suboptimal arm,
PIE pulls arms and observe their rewards.
With new reward observations, PIE updates the empirical mean \(\hatbm{\mu}_t\),
arms' KL-UCB indexes \(\bm{u}_t\),
the number of effective times of IE \(\hat{\bm{\tau}}_t\), and the time slot index \(t\).






\subsection{Parsimonious United Exploration (PUE)}

One also needs to be parsimonious in unitedly exploration
because UE requires that the number of plays pulling an arm is no less than the arm's reward capacity
and some of these plays may be redundant in acquiring rewards.
PUE's two core ideas are:
(1) prioritize the UE of arms with high empirical reward means and whose capacities have not been accurately learnt;
(2) not simply assign all \(N\) plays to an arm but only
the number of plays equal to the arm's capacities' upper confidence bound \(m_{k,t}^u\).

To realize the first idea, we denote \(\mathcal{Y}_t\) as a subset of arms deserving UE.
It should be a subset of empirical optimal arms in \(\mathcal{S}_t\)
because one does not need suboptimal arms' capacities
to achieve the optimal action.
Furthermore,
\(\mathcal{Y}_t\) should exclude the empirical least favored optimal arm \(L_t\)
because, instead of estimating the arm's exact reward capacity,
it is enough to have that
the number of plays pulling this arm is no greater than its capacity's
lower confidence bound \(m_{L_t, t}^l\).
So, no need to further improve its capacity estimate.
In addition, arms whose capacities have been accurately learnt,
i.e., \(m_{k,t}^l=m_{k,t}^u\),
should also be excluded.
To sum up, the arm set \(\mathcal{Y}_t\) is defined as \(\mathcal{Y}_t\coloneqq \{k\in\mathcal{S}_t\setminus \{L_t\}:m_{k,t}^l \neq m_{k,t}^u\}\).
To prioritize the exploration of arms in \(\mathcal{Y}_t\),
we increase these arms' empirical means by a large positive value\footnote{
    When the ``per-load'' reward
    is \([0,1]\) supported,
    then \(\max_k \hat{\mu}_{k,t} < 1\) holds
    and one can set \(M=1\).
    For Gaussian reward case,
    the \(M\) can be chosen as
    the reward mean's
    upper bound plus three times
    the standard deviation,
    e.g., when reward means are
    \([0,1]\) bounded
    and variance \(\sigma^2\le 1/2\) as assumed,
    one can set \(M=5\).
} \(M\)
and denote the \textit{prioritized} mean vector as \(\hatbm{\mu}_t'\).


To implement the second idea,
we input the prioritized mean vector \(\hatbm{\mu}_t'\) and the reward capacities'
upper confidence bounds \(\bm{m}_{t}^u\) into the \(\Oracle\) function.
Its output action \(\bm{a}_t^{\ue}\) guarantees at least one valid UE for an
empirical optimal arm in \(\mathcal{Y}_t\).
Note that if the number of plays allocated in this valid UE
--- equal to the arm's capacity upper bound \(m_{k,t}^u\) --- is not too large,
the action \(\bm{a}_t^{\ue}\) may be able to
unitedly explore more than one arm at the same time.

Lastly, PUE plays arms according to \(\bm{a}_t^{\ue}\) and observe these arms' rewards,
then updates the ``full-load'' reward mean estimate \(\hatbm{\nu}_t\), the number of effective times of UE \(\hat{\bm{\iota}}_t\),
and time index \(t\).



\subsection{The Detail of \texttt{OrchExplore} Algorithm}


\begin{algorithm}[tb]
    \caption{Orchestrative Exploration \texttt{OrchExplore}}
    \label{alg:oe_algorithm}
    \textbf{Initial:} \(t \gets 1\), \(L_t\gets N\),
    \(\mathcal{Y}_t\gets \emptyset\),
    \(\mathcal{S}_t \gets \{1,\dots,N\}\),\\
    \(\hat{\bm{\mu}}_t, \bm{u}_t \gets\bm{0}\),
    \(\hat{\bm{\tau}}_t,\hat{\bm{\iota}}_t,\bm{m}_t^l \gets \bm{1}\),
    \({\bm{m}}_t^u \gets N\cdot\bm{1}\).

    \begin{algorithmic}[1]
        \WHILE{$t \le T$}

        \IF{$t$ is odd or $\mathcal{Y}_t = \emptyset$} \label{alg:oe:pie_condition}

        \STATE{\(\triangleright\) \textit{Parsimonious Individual Exploration}}
        \STATE{$\bm{a}_t^{\ie}\gets \Oracle(\hat{\bm{\mu}}_t, \bm{m}^{l}_t).$} \label{alg:oe:ie_oracle}
        \STATE{\(\mathcal{S}_t \gets \{k: a_{k,t}^{\ie} > 0\}\).}
        \STATE{$L_t\gets \argmin_k\{\hat{\mu}_{k}:k\in\mathcal{S}_t\}.$}
        \STATE{$\mathcal{E}_t\gets \{k\not\in \mathcal{S}_t: u_{k,t} \ge \hat{\mu}_{L_t,t}\}.$}
        \IF{$\mathcal{E}_t \neq \emptyset$} \label{alg:oe:parsimonious}
        \STATE{w.p. $1/2$, pick $l\in\mathcal{E}_t$ uniformly and $a_{L_t,t}^{\ie} \gets a_{L_t, t}^{\ie} -  1, a_{l,t}^{\ie}\gets 1$.}\label{alg:oe:par_exp}
        \ENDIF
        \STATE{Play $\bm{a}_t^{\ie}$ and observe rewards.}
        \STATE{Update $\hat{\bm{\mu}}_t, \bm{u}_t, \hat{\bm{\tau}}_t,t.$}

        \ELSE

        \STATE{\(\triangleright\) \textit{Parsimonious United Exploration}}
        \STATE{$\hat{\bm{\mu}}'_t\gets \hat{\bm{\mu}}_t$.}
        \STATE{$\hat{\mu}'_{k,t} \gets \hat{\mu}_{k,t}+M$ for all $k\in\mathcal{Y}_t$.}
        \STATE{$\bm{a}_t^{\ue}\gets \Oracle(\hat{\bm{\mu}}'_t, \bm{m}^{u}_t).$} \label{alg:oe:ue_oracle}
        \STATE{Play $\bm{a}_t^{\ue}$ and observe rewards.}
        \STATE{Update $\hat{\bm{\nu}}_t, \hat{\bm{\iota}}_t,t$.}

        \ENDIF
        \STATE{Update $\bm{m}^{l}_t,\bm{m}^{u}_t$ by Eq.(\ref{eq:m_lower_bound})-(\ref{eq:m_upper_bound}).}
        \STATE{$\mathcal{Y}_t \gets \{k\in \mathcal{S}_t\setminus \{L_t\}: {m}_{k,t}^{l} \neq {m}_{k,t}^{u}\}.$}

        \ENDWHILE
    \end{algorithmic}
\end{algorithm}

\texttt{OrchExplore}
is presented at Algorithm~\ref{alg:oe_algorithm}.
At the beginning, \texttt{OrchExplore} runs PUE and PIE in turn --- PIE in odd time slots and PUE in even time slots.
After each round of PIE or PUE, the algorithm updates capacities' lower and upper confidence bounds via Eq.(\ref{eq:m_lower_bound})-(\ref{eq:m_upper_bound}) (let \(\delta\gets 2/T\))
and the PUE set \(\mathcal{Y}_t\) according to the latest capacity bounds.

When the PUE set \(\mathcal{Y}_t=\emptyset\), i.e., all empirical optimal arms' capacities are learnt (\(m_{k,t}^l=m_{k,t}^u\)),
\texttt{OrchExplore} only runs PIE (cf., line~\ref{alg:oe:pie_condition}).
When both the PUE set \(\mathcal{Y}_t\) and the PIE set \(\mathcal{E}_t\) are empty (line~\ref{alg:oe:parsimonious}),
PIE acts as \textit{exploitation}:
it allocates plays to empirical optimal arms according to these arms' reward capacities
(except the least favored arm which is only assigned the remaining plays).

\section{Regret Analysis of \texttt{OrchExplore}}
\label{sec:orchexplore_analysis}

In this section, we show the \texttt{OrchExplore} algorithm enjoys a tight logarithmic regret upper bound.

\begin{theorem}[Regret Upper Bound of \texttt{OrchExplore}]\label{thm:oe_regret_upper_bound}
    When the time horizon \(T\ge \max_{1\le k\le L} \exp(49m_k^2 / \mu_k^2)\) and \(0 <\varepsilon < \min_{k=1}^{K-1}\frac{\mu_k - \mu_{k+1}}{2}\),
    Algorithm~\ref{alg:oe_algorithm}'s expected regret is upper bounded as follows,
    \begin{equation}\label{eq:oe_regret_upper_bound}
        \begin{split}
            &\ERT
            \le  \sum_{k=L+1}^K \frac{\Delta_{L,k}(\log T + 4 \log(\log T))}{\kl(\mu_k+\varepsilon,\mu_L-\varepsilon)} \\
            &\quad+ \sum_{k=1}^{L-1}\frac{49w_km_k^2\log (T)}{\mu_k^2} + \frac{49 w_L m_L^2\log (T)}{(m_L - \bar{m}_L+ 1)^2 \mu_L^2} \\
            &\quad+ 13K^2 N^2(4+\varepsilon^{-2}),
        \end{split}
    \end{equation}
    where
    \(\kl\) represents the KL-divergence between
    two Bernoulli distributions in the \([0,1]\) supported reward case
    or
    two Gaussian distributions with same variances
    in the Gaussian reward case,
    \(w_k\coloneqq f(\bm{a}^*) - m_k\mu_k +\mu_1\) is the highest cost of one round of UE for arm \(k\) and one round of deliberate exploration in PIE,
    and \(\bar{m}_L = N-\sum_{k=1}^{L-1}m_k\) is the number of plays pulling arm \(L\) in the optimal action.
\end{theorem}
\begin{proof}[Proof sketch of Theorem~\ref{thm:oe_regret_upper_bound}]
    The detailed proof is in Appendix~\ref{app:orchexplore_analysis}.
    \textbf{Step 1: show that the pulls of suboptimal arms are mainly caused by the deliberate explorations in PIE (line~\ref{alg:oe:par_exp}).}
    That is, except PIE's deliberate explorations, the cost of pulling suboptimal arms are finite, which is bounded by the last term in the RHS of Eq.(\ref{eq:oe_regret_upper_bound}).

    \textbf{Step 2: upper bound the cost of the suboptimal arms' deliberate explorations in PIE (line~\ref{alg:oe:par_exp}).} This cost, due to the advantage of KL-UCB index, corresponds to Eq.(\ref{eq:oe_regret_upper_bound})'s first term
    and a part of its second and third terms.

    \textbf{Step 3:  upper bound the cost of united explorations for optimal arms in PUE.}
    After covering the cost of exploring suboptimal arms in {Step 1} and {Step 2},
    we only need to consider the cost of exploring optimal arms in PUE.
    The total cost of these UEs is measured by the capacity estimator's sample complexity upper bound
    in Theorem~\ref{cor:m_sample_complexity}.
    This corresponds to Eq.(\ref{eq:oe_regret_upper_bound})'s second and third terms.
\end{proof}

From Theorem~\ref{thm:oe_regret_upper_bound}, letting \(T\to \infty\) and \(\varepsilon\to 0 \), one immediately obtains the following corollary.
\begin{corollary}
    The \emph{\texttt{OrchExplore}} algorithm's regret is asymptotically upper bounded as follows:
    \begin{equation}\label{eq:oe_regret_upper_bound_asym}
        \begin{split}
            &\limsup_{T\to\infty}\frac{\ERT}{\log T} \le \sum_{k=L+1}^K \frac{\Delta_{L,k}}{\kl(\mu_k,\mu_L)} \\
            &\qquad\quad + \sum_{k=1}^{L-1} \frac{49w_km_k^2}{\mu_k^2} + \frac{49 w_L m_L^2}{(m_L - \bar{m}_L+ 1)^2 \mu_L^2}.
        \end{split}
    \end{equation}
\end{corollary}

Comparing the regret upper bound in Eq.(\ref{eq:oe_regret_upper_bound_asym}) to the regret lower bound in Theorem~\ref{thm:regret_lower_bound} shows that their first terms are the same (i.e., optimal)
and their second and third terms --- in the Gaussian reward case --- both match.

\vspace{-0.7\baselineskip}
\section{Evaluation}\label{sec:simulation}
\vspace{-0.35\baselineskip}
We conduct simulations to validate the performance of \texttt{OrchExplore} in Algorithm~\ref{alg:oe_algorithm} and compare it to other algorithms adapted from MAB.
Consider a \texttt{MP-MAB-SA} problem with $K=9$ arms and $N=7$ plays. The arms' ``per-load'' reward means and capacities are as follows.

\centerline{
    \tabcolsep=0.08cm
    \begin{tabular}{|c||ccccccccc|}
        \hline
        Arm index $k$       & 1   & 2   & 3   & 4   & 5   & 6   & 7   & 8   & 9   \\
        \hline
        Reward mean $\mu_k$ & 0.9 & 0.8 & 0.7 & 0.6 & 0.5 & 0.4 & 0.3 & 0.2 & 0.1 \\
        Capacity $m_k$      & 2   & 4   & 3   & 3   & 2   & 1   & 3   & 4   & 2   \\
        \hline
    \end{tabular}
}\vspace{-0.25\baselineskip}
The ``per-load'' rewards follows Bernoulli distributions.
The optimal action $\bm{a}^*$ is $(2,4,1,0,\ldots,0)$ and its expected reward $f(\bm{a}^*) = 5.7$.
Each simulation is averaged over 200 realizations.
We set $\delta=2/T$ as default.
The Gaussian distribution case is evaluated in Appendix~\ref{appsub:guassian-simulation}.
We also apply \texttt{OrchExplore} to a 5G \& 4G base station selection application in Appendix~\ref{app:real_world_simulation}.


\begin{figure}[tb]
    \centering
    \subfloat[\texttt{OrchExplore} \textit{vs.} \texttt{MP-SE-SA} \textit{vs.} \texttt{ETC-UCB}\label{subfig:fine_grained_update}]{\includegraphics[width=0.5\columnwidth]{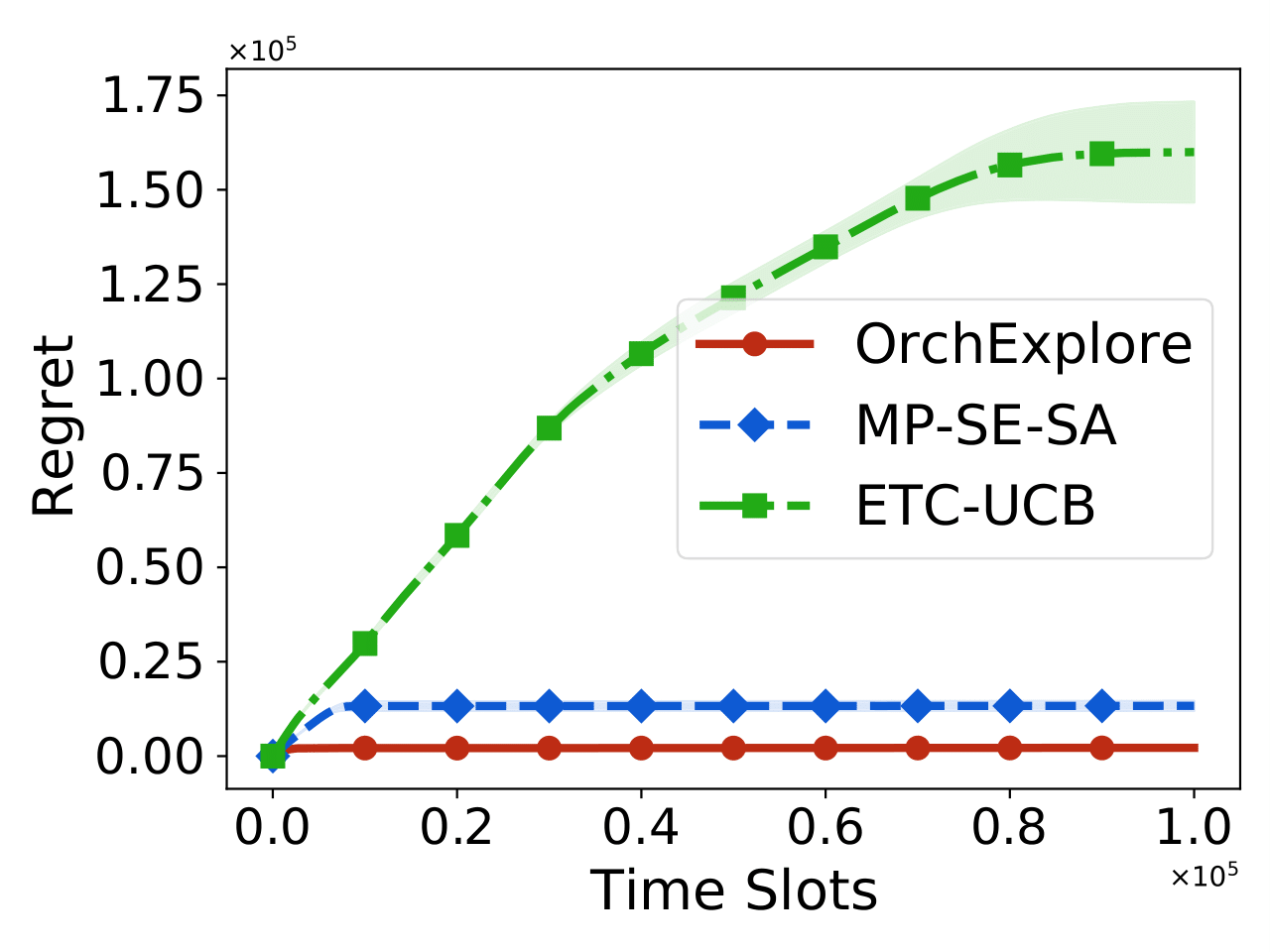}}
    \subfloat[Improvement of UCI\label{subfig:hfd_vs_uci}]{\includegraphics[width=0.5\columnwidth]{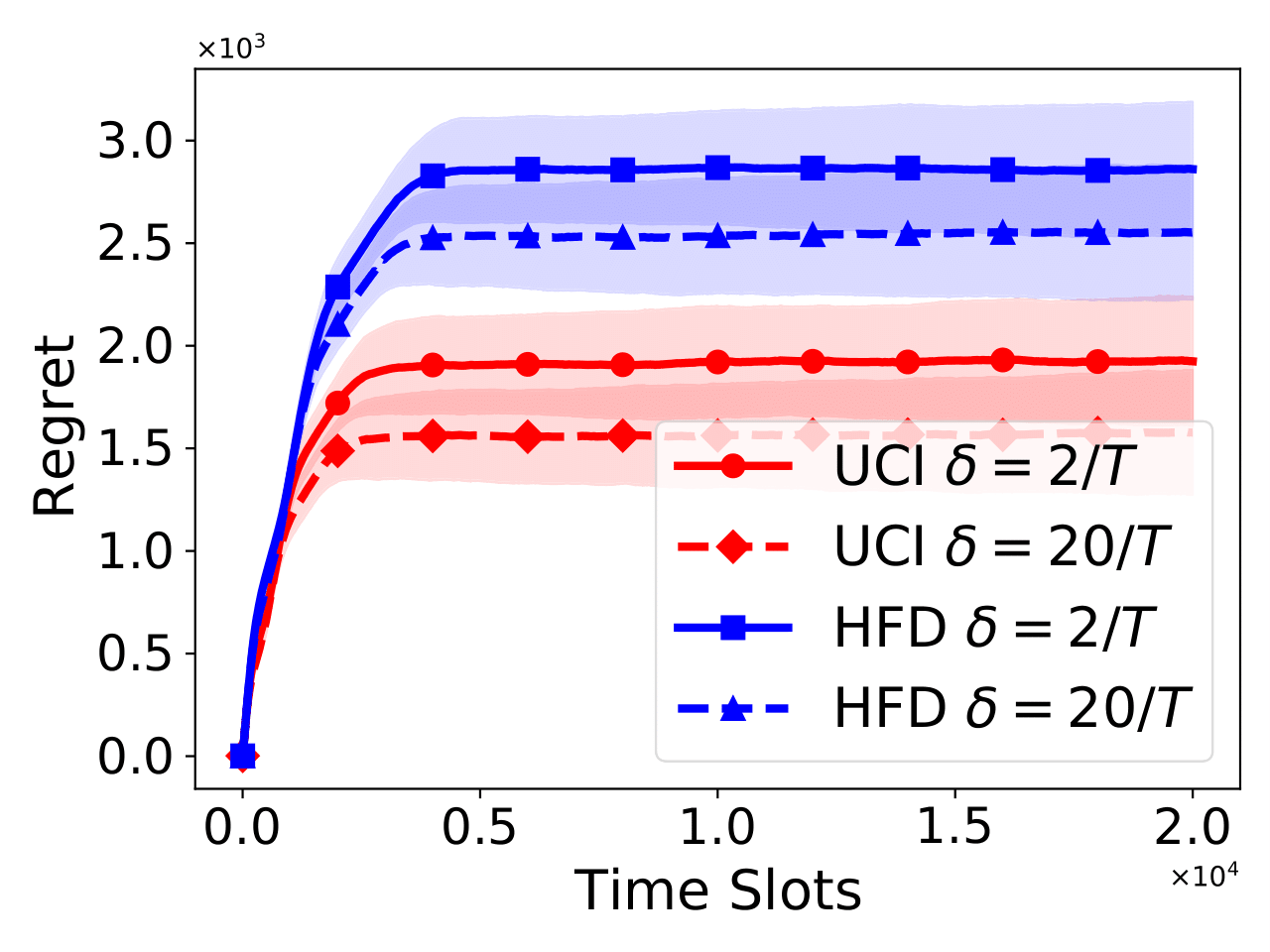}}\\
    \subfloat[Price of learning $m_k$\label{subfig:know_or_not}]{\includegraphics[width=0.5\columnwidth]{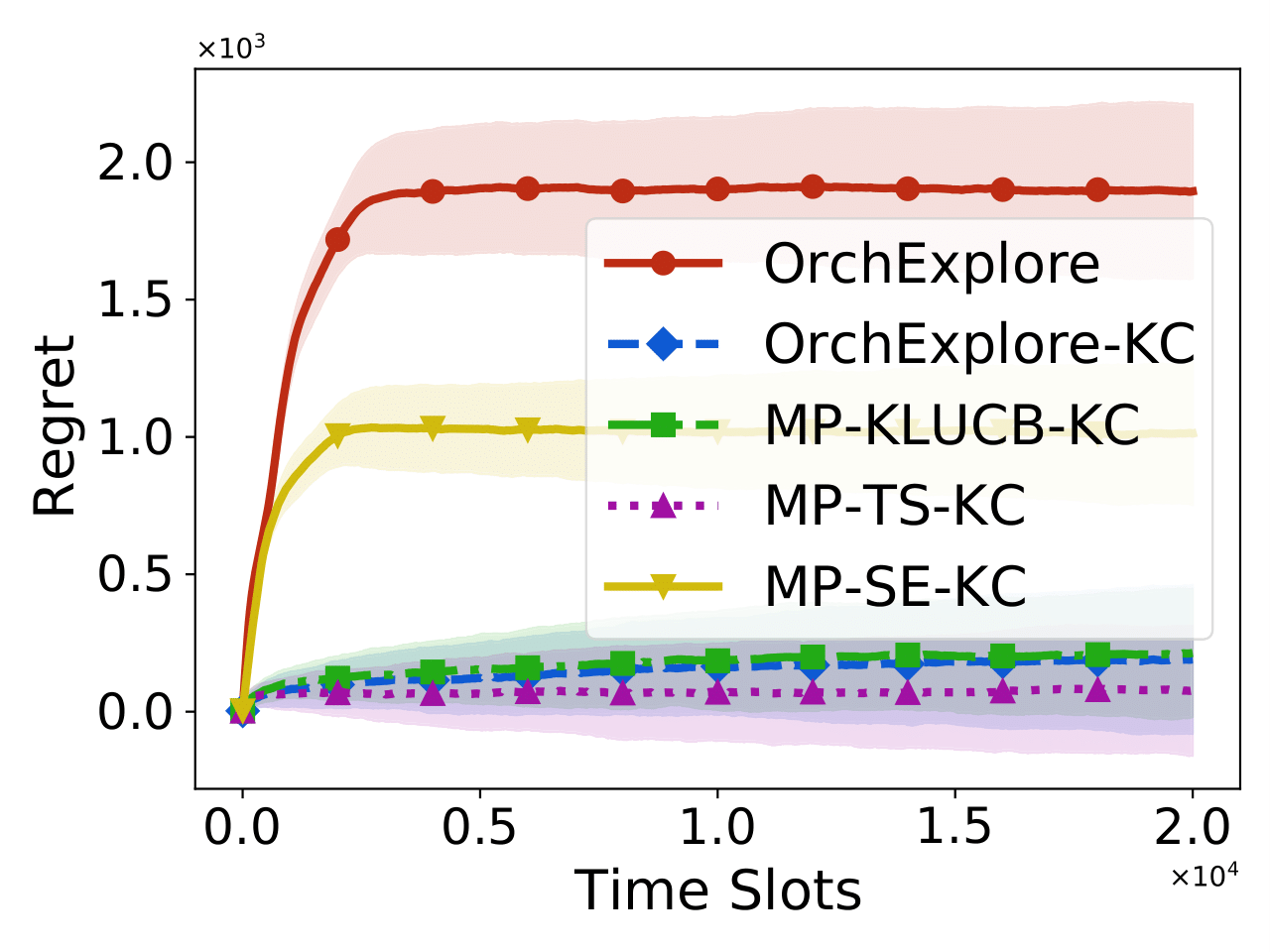}}
    \subfloat[Implicitly learn $m_k$\label{subfig:learn_or_not}]{\includegraphics[width=0.5\columnwidth]{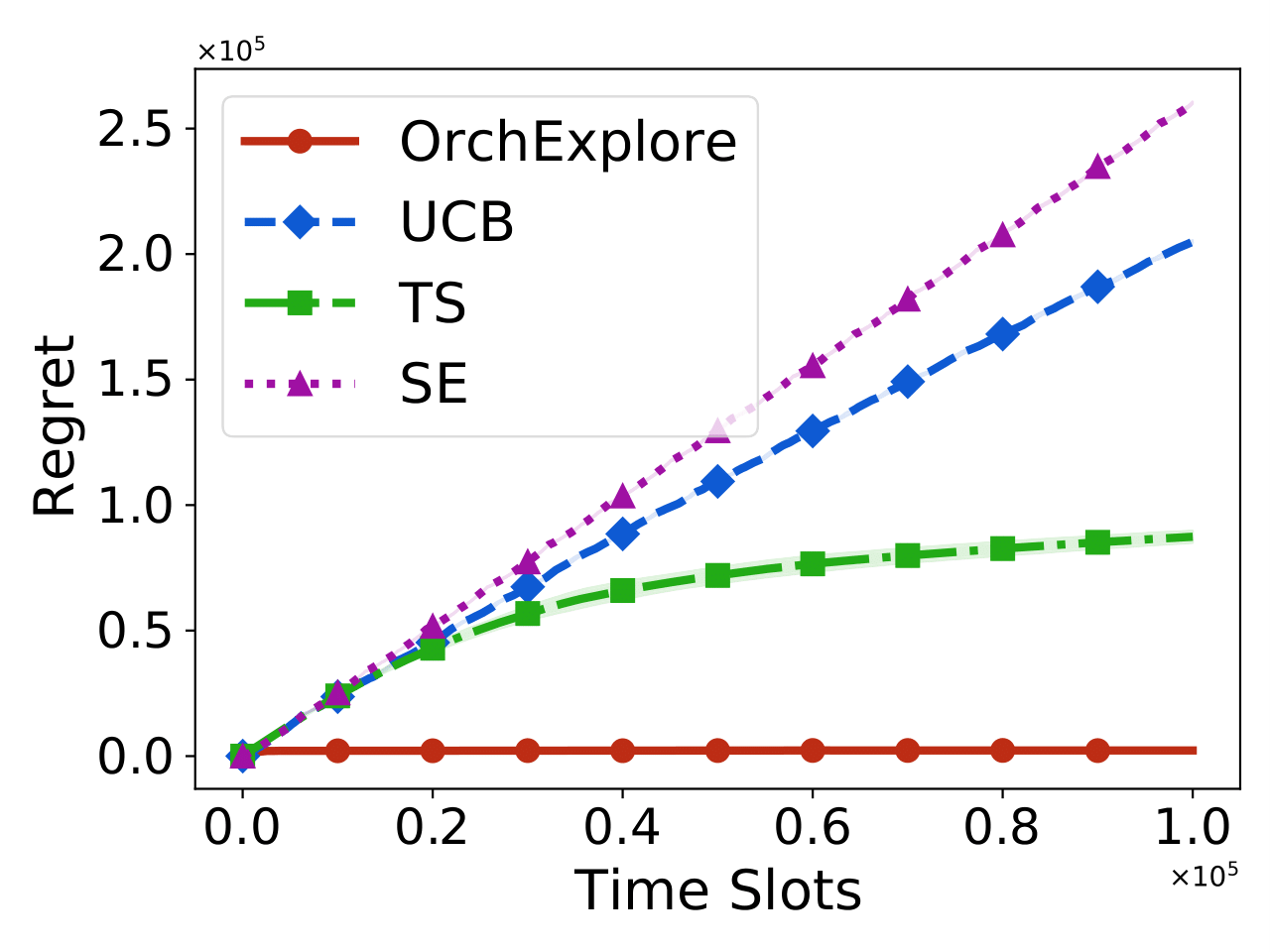}}
    \caption{Evaluation under Bernoulli Rewards}
\end{figure}

\vspace{-0.35\baselineskip}
\textbf{\texttt{OrchExplore} \textit{vs.} \texttt{MP-SE-SA} \textit{vs.} \texttt{ETC-UCB}. }
Besides \texttt{OrchExplore}, we also design other two algorithms for addressing
\texttt{MP-MAB-SA}:
the \texttt{ETC-UCB} two-phase algorithm where the ETC phase learns the reward capacity and the UCB phase
handles the reward means (Appendix~\ref{app:etc_ucb}),
and the elimination based algorithm \texttt{MP-SE-SA} which
learns the capacities and reward means in a fine-grained style (Appendix~\ref{sec:se_algorithm}).
Both algorithms enjoys lower computation complexity and
are more flexible in application, e.g., in batched learning,
while only \texttt{OrchExplore}'s regret is tight.
Figure~\ref{subfig:fine_grained_update}
shows the superiority of \texttt{OrchExplore} than
\texttt{ETC-UCB} and \texttt{MP-SE-SA}.
It validates the efficacy of parsimonious individual and united explorations.
\begin{remark}[Theoretical results comparison of \texttt{OrchExplore} to \texttt{MP-SE-SA} and \texttt{ETC-UCB}]
    \texttt{OrchExplore}
    (Theorem~\ref{thm:oe_regret_upper_bound}, \(O((\sum_{k=L+1}^K\Delta_{L,k}/\text{kl}(\mu_k,\mu_L))
    +\sum_{k=1}^Lm_k^2/\mu_k^2)\log T\))
    has a tighter regret upper bound
    than \texttt{ETC-UCB}
    (Theorem~\ref{thm:etc_ucb_bound},
    \(O((\sum_{k=L+1}^K \Delta_{1,k}m_k/\Delta_{L,k}^2 +
    \sum_{k=1}^K m_k^2/\mu_k^2)\log T)\)) and
    \texttt{MP-SE-SA}
    (Theorem~\ref{thm:main_regret},
    \(O((\sum_{k=L+1}^K \Delta_{1,k}m_k/\Delta_{L,k}^2 +
    \sum_{k=1}^N m_k^2/\mu_k^2)\log T)\))
    ---
    \texttt{OrchExplore}'s the first regret upper bound term
    matches
    the lower bound's first term while the other two's are not,
    and its second term is also smaller since \(L\le N<K\)
    in the summation range.
\end{remark}

\vspace{-0.35\baselineskip}
\textbf{Improvement of UCI over Hoeffding's inequality. }
Figure~\ref{subfig:hfd_vs_uci} illustrates that \texttt{OrchExplore} with uniform confidence interval (UCI) outperforms the others which use Hoeffding's inequality (HFD).
This confirms that the employed UCI is sharper than HFD.

\vspace{-0.35\baselineskip}
\textbf{The price of learning capacity. }
Figure~\ref{subfig:know_or_not}
compares \texttt{OrchExplore} with other four algorithms with \emph{known} capacity (KC):
\texttt{OrchExplore}-KC, KL-UCB~\cite{cappe_kullback-leibler_2013}, Thompson Sampling (TS)~\cite{komiyama_optimal_2015}, and successive elimination (SE)~\cite{perchet_multi-armed_2013},
where they select the empirical optimal action according to each arm's index and the known capacity.
Comparing the performance of \texttt{OrchExplore} to \texttt{OrchExplore}-KC's shows that the price of learning capacity is much larger than estimating reward means alone.

\vspace{-0.35\baselineskip}
\textbf{Comparison to implicitly learning capacity algorithms. }
One can regard the \texttt{MP-MAB-SA} as an MAB with the action space $\mathscr{A}$, i.e., each $N$-play allocation (action) as an independent arm.
With such transformation, there is no need to consider the shareable arms setting but to ``implicitly learn'' about arms' capacities.
We apply UCB, TS and SE to this MAB.
Figure~\ref{subfig:learn_or_not} shows that our \texttt{OrchExplore} outperforms those implicitly learning strategies.
The result is not surprising as that $\abs{\mathscr{A}}{\sim} K^N$ is very large,
and this confirms the necessity of modelling the shareable arms setting and devising \texttt{OrchExplore} to
tackle the problem.


\section{Conclusion}\label{sec:conclusion}
\vspace{-0.5\baselineskip}
We generalize the MP-MAB model to allow several plays sharing an arm.
This new model contains two groups of unknown parameters: arm's finite reward capacities
and ``per-load'' reward means,
based on which arms are associated with load-dependent stochastic rewards.
With load-dependent observations, the learning tasks of both types of parameters are \emph{coupled}:
without known one, it is difficult to learn the other.
Surprisingly,
we prove a regret lower bound (for Gaussian rewards)
which dichotomizes both learning tasks' regret costs,
and also propose an algorithm (\texttt{OrchExplore})
which achieves a tight regret upper bound
whose terms respectively match the cost due to distinguishing suboptimal reward means
and the cost due to learning reward capacities in the regret lower bound.

\section*{Acknowledgements}

We would like to thank anonymous reviewers from ICML 2022 and AISTATS 2022
for their comments that helped us improve this paper.
The work of Xuchuang Wang and John C.S. Lui was supported in part by the RGC SRFS2122-4202.
The work of Hong Xie was supported by Chongqing Talents:
Exceptional Young Talents Project (cstc2021ycjhbgzxm0195).

\bibliography{optimal_sharing_bandits.bib}

\begin{thebibliography}{28}
\providecommand{\natexlab}[1]{#1}
\providecommand{\url}[1]{\texttt{#1}}
\expandafter\ifx\csname urlstyle\endcsname\relax
  \providecommand{\doi}[1]{doi: #1}\else
  \providecommand{\doi}{doi: \begingroup \urlstyle{rm}\Url}\fi

\bibitem[Anandkumar et~al.(2011)Anandkumar, Michael, Tang, and
  Swami]{anandkumar2011distributed}
Anandkumar, A., Michael, N., Tang, A.~K., and Swami, A.
\newblock Distributed algorithms for learning and cognitive medium access with
  logarithmic regret.
\newblock \emph{IEEE Journal on Selected Areas in Communications}, 29\penalty0
  (4):\penalty0 731--745, 2011.

\bibitem[Anantharam et~al.(1987)Anantharam, Varaiya, and
  Walrand]{anantharam1987asymptotically}
Anantharam, V., Varaiya, P., and Walrand, J.
\newblock Asymptotically efficient allocation rules for the multiarmed bandit
  problem with multiple plays-part i: Iid rewards.
\newblock \emph{IEEE Transactions on Automatic Control}, 32\penalty0
  (11):\penalty0 968--976, 1987.

\bibitem[Bistritz \& Leshem(2018)Bistritz and Leshem]{bistritz2018distributed}
Bistritz, I. and Leshem, A.
\newblock Distributed multi-player bandits-a game of thrones approach.
\newblock \emph{Advances in Neural Information Processing Systems (NeurIPS)},
  2018.

\bibitem[Bourel et~al.(2020)Bourel, Maillard, and Talebi]{bourel2020tightening}
Bourel, H., Maillard, O.-A., and Talebi, M.~S.
\newblock Tightening exploration in upper confidence reinforcement learning.
\newblock In \emph{International Conference on Machine Learning}, 2020.

\bibitem[Bubeck et~al.(2012)Bubeck, Cesa-Bianchi, et~al.]{bubeck2012regret}
Bubeck, S., Cesa-Bianchi, N., et~al.
\newblock Regret analysis of stochastic and nonstochastic multi-armed bandit
  problems.
\newblock \emph{Foundations and Trends{\textregistered} in Machine Learning},
  5\penalty0 (1):\penalty0 1--122, 2012.

\bibitem[Cai et~al.(2018)Cai, Liu, Chen, and Lui]{cai2018online}
Cai, K., Liu, X., Chen, Y.-Z.~J., and Lui, J. C.~S.
\newblock An online learning approach to network application optimization with
  guarantee.
\newblock In \emph{IEEE INFOCOM 2018-IEEE Conference on Computer
  Communications}, pp.\  2006--2014. IEEE, 2018.

\bibitem[Capp{\'e} et~al.(2013)Capp{\'e}, Garivier, Maillard, Munos, Stoltz,
  et~al.]{cappe_kullback-leibler_2013}
Capp{\'e}, O., Garivier, A., Maillard, O.-A., Munos, R., Stoltz, G., et~al.
\newblock {K}ullback--{L}eibler upper confidence bounds for optimal sequential
  allocation.
\newblock \emph{Annals of Statistics}, 41\penalty0 (3):\penalty0 1516--1541,
  2013.

\bibitem[Cesa-Bianchi \& Lugosi(2012)Cesa-Bianchi and
  Lugosi]{cesa2012combinatorial}
Cesa-Bianchi, N. and Lugosi, G.
\newblock Combinatorial bandits.
\newblock \emph{Journal of Computer and System Sciences}, 78\penalty0
  (5):\penalty0 1404--1422, 2012.

\bibitem[Chen et~al.(2013)Chen, Wang, and Yuan]{chen2013combinatorial}
Chen, W., Wang, Y., and Yuan, Y.
\newblock Combinatorial multi-armed bandit: General framework and applications.
\newblock In \emph{International Conference on Machine Learning}, pp.\
  151--159. PMLR, 2013.

\bibitem[Chen et~al.(2016)Chen, Wang, Yuan, and Wang]{chen2016combinatorial}
Chen, W., Wang, Y., Yuan, Y., and Wang, Q.
\newblock Combinatorial multi-armed bandit and its extension to
  probabilistically triggered arms.
\newblock \emph{The Journal of Machine Learning Research}, 17\penalty0
  (1):\penalty0 1746--1778, 2016.

\bibitem[Combes et~al.(2015)Combes, Magureanu, Proutiere, and
  Laroche]{combes2015learning}
Combes, R., Magureanu, S., Proutiere, A., and Laroche, C.
\newblock Learning to rank: Regret lower bounds and efficient algorithms.
\newblock In \emph{Proceedings of the 2015 ACM SIGMETRICS International
  Conference on Measurement and Modeling of Computer Systems}, pp.\  231--244,
  2015.

\bibitem[{Gai} et~al.(2012){Gai}, {Krishnamachari}, and
  {Jain}]{gai2012combinatorial}
{Gai}, Y., {Krishnamachari}, B., and {Jain}, R.
\newblock Combinatorial network optimization with unknown variables:
  Multi-armed bandits with linear rewards and individual observations.
\newblock \emph{IEEE/ACM Transactions on Networking}, 20\penalty0 (5):\penalty0
  1466--1478, 2012.
\newblock \doi{10.1109/TNET.2011.2181864}.

\bibitem[Komiyama et~al.(2015)Komiyama, Honda, and
  Nakagawa]{komiyama_optimal_2015}
Komiyama, J., Honda, J., and Nakagawa, H.
\newblock Optimal regret analysis of {T}hompson sampling in stochastic
  multi-armed bandit problem with multiple plays.
\newblock In \emph{International Conference on Machine Learning}, pp.\
  1152--1161. PMLR, 2015.

\bibitem[Komiyama et~al.(2017)Komiyama, Honda, and
  Takeda]{komiyama2017position}
Komiyama, J., Honda, J., and Takeda, A.
\newblock Position-based multiple-play bandit problem with unknown position
  bias.
\newblock In \emph{Proceedings of the 31st International Conference on Neural
  Information Processing Systems}, pp.\  5005--5015, 2017.

\bibitem[Kveton et~al.(2014)Kveton, Wen, Ashkan, Eydgahi, and
  Eriksson]{kveton2014matroid}
Kveton, B., Wen, Z., Ashkan, A., Eydgahi, H., and Eriksson, B.
\newblock Matroid bandits: fast combinatorial optimization with learning.
\newblock In \emph{Proceedings of the Thirtieth Conference on Uncertainty in
  Artificial Intelligence}, pp.\  420--429, 2014.

\bibitem[Kveton et~al.(2015)Kveton, Wen, Ashkan, and
  Szepesv{\'a}ri]{kveton2015combinatorial}
Kveton, B., Wen, Z., Ashkan, A., and Szepesv{\'a}ri, C.
\newblock Combinatorial cascading bandits.
\newblock In \emph{Proceedings of the 28th International Conference on Neural
  Information Processing Systems-Volume 1}, pp.\  1450--1458, 2015.

\bibitem[Lagr{\'e}e et~al.(2016)Lagr{\'e}e, Vernade, and
  Capp{\'e}]{lagree2016multiple}
Lagr{\'e}e, P., Vernade, C., and Capp{\'e}, O.
\newblock Multiple-play bandits in the position-based model.
\newblock In \emph{Proceedings of the 30th International Conference on Neural
  Information Processing Systems}, pp.\  1605--1613, 2016.

\bibitem[Lai \& Robbins(1985)Lai and Robbins]{lai1985asymptotically}
Lai, T.~L. and Robbins, H.
\newblock Asymptotically efficient adaptive allocation rules.
\newblock \emph{Advances in applied mathematics}, 6\penalty0 (1):\penalty0
  4--22, 1985.

\bibitem[Lattimore \& Szepesv{\'a}ri(2020)Lattimore and
  Szepesv{\'a}ri]{lattimore2020bandit}
Lattimore, T. and Szepesv{\'a}ri, C.
\newblock \emph{Bandit algorithms}.
\newblock Cambridge University Press, 2020.

\bibitem[Magesh \& Veeravalli(2021)Magesh and
  Veeravalli]{magesh2021decentralized}
Magesh, A. and Veeravalli, V.~V.
\newblock Decentralized heterogeneous multi-player multi-armed bandits with
  non-zero rewards on collisions.
\newblock \emph{IEEE Transactions on Information Theory}, 2021.

\bibitem[Narayanan et~al.(2020)Narayanan, Ramadan, Carpenter, Liu, Liu, Qian,
  and Zhang]{narayanan2020first}
Narayanan, A., Ramadan, E., Carpenter, J., Liu, Q., Liu, Y., Qian, F., and
  Zhang, Z.-L.
\newblock A first look at commercial 5g performance on smartphones.
\newblock In \emph{Proceedings of The Web Conference 2020}, pp.\  894--905,
  2020.

\bibitem[Perchet et~al.(2013)Perchet, Rigollet,
  et~al.]{perchet_multi-armed_2013}
Perchet, V., Rigollet, P., et~al.
\newblock The multi-armed bandit problem with covariates.
\newblock \emph{Annals of statistics}, 41\penalty0 (2):\penalty0 693--721,
  2013.

\bibitem[Rosenski et~al.(2016)Rosenski, Shamir, and Szlak]{rosenski2016multi}
Rosenski, J., Shamir, O., and Szlak, L.
\newblock Multi-player bandits--a musical chairs approach.
\newblock In \emph{International Conference on Machine Learning}, pp.\
  155--163. PMLR, 2016.

\bibitem[Slivkins et~al.(2019)]{slivkins2019introduction}
Slivkins, A. et~al.
\newblock Introduction to multi-armed bandits.
\newblock \emph{Foundations and Trends{\textregistered} in Machine Learning},
  12\penalty0 (1-2):\penalty0 1--286, 2019.

\bibitem[Tsybakov(2008)]{tsybakov_introduction_2008}
Tsybakov, A.~B.
\newblock \emph{Introduction to Nonparametric Estimation}.
\newblock Springer Publishing Company, Incorporated, 1st edition, 2008.
\newblock ISBN 0387790519.

\bibitem[Wang et~al.(2020)Wang, Proutiere, Ariu, Jedra, and
  Russo]{wang2020optimal}
Wang, P.-A., Proutiere, A., Ariu, K., Jedra, Y., and Russo, A.
\newblock Optimal algorithms for multiplayer multi-armed bandits.
\newblock In \emph{International Conference on Artificial Intelligence and
  Statistics}, pp.\  4120--4129. PMLR, 2020.

\bibitem[Wang et~al.(2022)Wang, Xie, and Lui]{wang2022multi}
Wang, X., Xie, H., and Lui, J.~C.
\newblock Multi-player multi-armed bandits with finite shareable resources
  arms: Learning algorithms \& applications.
\newblock In \emph{Proceedings of IJCAI}, 2022.

\bibitem[Wen et~al.(2017)Wen, Kveton, Valko, and Vaswani]{wen2017online}
Wen, Z., Kveton, B., Valko, M., and Vaswani, S.
\newblock Online influence maximization under independent cascade model with
  semi-bandit feedback.
\newblock In \emph{Neural Information Processing Systems}, pp.\  1--24, 2017.

\end{thebibliography}
\bibliographystyle{icml2022}

\newpage
\appendix\onecolumn
\section{The Appendix Overview}

In the section, we provide a road map of the appendix: \begin{itemize}
	\item Appendix~\ref{app:model}: further motivates our reward model in Eq.(\ref{eq:arm_level_feedback}).
	\item Appendix~\ref{app:lower_bound}: provides the proof of lower bounds. 
	It includes: the KL-divergence's detail calculation in Appendix~\ref{appsub:detail_kl_upper_bound},
	a sketch of Theorem~\ref{thm:sample_lower_bound}'s second part in Appendix~\ref{appsub:sample-complexity-lower-bound-second-part},
	and the regret lower bound's full proof in Appendix~\ref{app:regret_lower_bound}.
	\item Appendix~\ref{app:capacity_proof}: provides learning reward capacity (Section~\ref{sec:learning_sharing_capacity})'s proofs including the uniform confidence interval (UCI)'s design (Appendix~\ref{appsub:uniform-confidence-interval}) and our estimator's sample complexity upper bound's proof (Appendix~\ref{appsub:proof-sample-complexity-upper-bound}). 
	\item Appendix~\ref{app:orchexplore_analysis}: proves the \texttt{OrchExplore} algorithm's regret upper bound (Theorem~\ref{thm:oe_regret_upper_bound}).
	\item Appendix~\ref{sec:se_algorithm}: devises an successive elimination based algorithm called \texttt{MP-SE-SA}.
	\item Appendix~\ref{sec:se_analysis}: provides the \texttt{MP-SE-SA} algorithm's regret upper bound analysis.
	\item Appendix~\ref{app:etc_ucb}: designs a two-phase algorithm called \texttt{ETC-UCB} whose ETC phase learns the capacities and UCB phase deals with the reward means, and provides its regret upper bound analysis.
	\item Appendix~\ref{app:addition-simulation}: provides additional empirical evaluations.
	It includes a real world application in Appendix~\ref{app:real_world_simulation}
	and a Gaussian rewards evaluation of Section~\ref{sec:simulation} in Appendix~\ref{appsub:guassian-simulation}.
	\item Appendix~\ref{sec:hfd_ci}: compare our uniform confidence interval (UCI) to Hoeffding's inequality based UCI.	
\end{itemize}

\section{Model Motivations}\label{app:model}

\subsection{Motivate the Reward Capacity {$m_k$}}\label{appsub:motive_m}

\textbf{Mobile edge computing: }
To illustrate, consider the mobile edge computing application,
where an offloading spot with $N$ tasks is covered by $K$ edge servers.
Each arm can model an edge server and each play can model a task.
$N$ plays represent assigning $N$ tasks to these servers.
The $m_k$ can model the number of computing units (e.g., cores of a CPU)
of the $k$-th edge server and $X_k$ can model the
reward (e.g., quantified by the completion time of a task) from one computing unit.

\textbf{Cognitive radio network: }
Another example is the channel selection in cognitive radio networks
where there are $K$ opportunistic channels for $N$ secondary users.
An arm can model a channel and a play can model a secondary user.
$N$ plays can model allocating $N$ secondary users to opportunistic channels
The $m_k$ can model the maximum number of connections that the
$k$-th opportunistic channel can support,
and $X_k$ can model the
utility of supporting a connection.
Note that  $X_k$ is a random variable capturing
the stochastic availability of the $k$-th opportunistic channel.

\subsection{Motivate the Reward Model in Eq.(\ref{eq:arm_level_feedback}): {$R_k(a_{k,t})\triangleq \min\{a_{k,t},m_k\}\cdot X_k$}}\label{appsub:motive_R}
\textbf{Mobile edge computing: }
For example, in edge computing systems,
the reward $R_k(a_{k,t})$ can model the total amount of time
to process $a_{k,t}$ tasks at edge server $k$.
Eq.(\ref{eq:arm_level_feedback}) captures that
each task gets one unit of computing resource
if the number of tasks $a_{k,t}$ is less than the number of computing
units $m_k$, otherwise those tasks will equally share the $m_k$ resources.

\textbf{Cognitive radio network: }
In cognitive radio networks,
the reward $R_k(a_{k,t})$ can model the total utility
of $a_{k,t}$ secondary users assigned to channel $k$.
Eq.(\ref{eq:arm_level_feedback}) captures that
each unity of the opportunistic spectrum is
allocated to one secondary user
if the number of secondary users $a_{k,t}$
is less than the number of spectrum connection $m_k$,
otherwise these secondary users will equally share the $m_k$ connections.

That \(m_k\) capacities' rewards are the same random variable \(X_k\)
models the availability of a channel: if the channel is occupied by a primary user \((X_k=0)\),
then no secondary user can access it; otherwise \((X_k=1)\),
secondary users can share the channel up to its capacity.

\section{Proofs of Lower Bounds}\label{app:lower_bound}

\subsection{Sample Complexity Lower Bound \\ Theorem~\ref{thm:sample_lower_bound}'s Step (3): Detail Derivation of the KL-divergence upper bound}
\label{appsub:detail_kl_upper_bound}

Recall that we assume \(m_k^{(0)} < m_k^{(1)}\) and \(X_k\sim\mathcal{N}(\mu_k,\sigma_k^2)\).
When \(m_k^{(0)}< l \le m_k^{(1)}\), we have
\[
	\begin{split}
		\KL(\P_0^l, \P_0^l)_{\{m_k^{(0)}< l \le m_k^{(1)}\}}
		& = \KL(m_k^{(0)}\mathcal{N}(\mu_k,\sigma_k^2), l\mathcal{N}(\mu_k,\sigma_k^2)) \\
		& = \frac 1 2\left( \log\left( \frac{l^2}{\left(m_k^{(0)}\right)^2}\right) +
		\frac{\left(m_k^{(0)}\right)^2}{l^2} - 1 \right) +
		\frac{\left(l-m_k^{(0)}\right)^2\mu_k^2}{2l^2\sigma_k^2} \\
		& \le \frac 1 2\left( \log\left( \frac{\left(m_k^{(1)}\right)^2}{\left(m_k^{(0)}\right)^2}\right) +
		\frac{\left(m_k^{(0)}\right)^2}{\left(m_k^{(1)}\right)^2} - 1 \right) +
		\frac{\left(m_k^{(1)}-m_k^{(0)}\right)^2\mu_k^2}{2\left(m_k^{(1)}\right)^2\sigma_k^2}\\
		& = \KL(m_k^{(0)}\mathcal{N}(\mu_k,\sigma_k^2), m_k^{(1)}\mathcal{N}(\mu_k,\sigma_k^2))\\
		& = \KL(\P_0^l, \P_0^l)_{\{m_k^{(1)}< l \le N\}},
	\end{split}
\]
where the inequality is due to that (1) \(F(x) = \log x + 1/x -1\) is increasing in \(x>1\) and \(x\gets (l/m_k^{(0)})^2\);
(2) the second term is also increasing in \(l\); and (3) \(m_k^{(0)}< l \le m_k^{(1)}\).
From this, we obtain \(
\KL(\P_0^l,\P_1^l)_{\{m_k^{(0)}<l\le m_k^{(1)}\}}
\le \KL(\P_0^l,\P_1^l)_{\{m_k^{(1)}<l\le N\}}.
\)

Then, given \(m_k^{(0)}=m_k-1, m_k^{(1)} = m_k\) we turn to bound \(\KL(m_k^{(0)}\mathcal{N}(\mu_k,\sigma_k^2), m_k^{(1)}\mathcal{N}(\mu_k,\sigma_k^2))\)
\[
	\begin{split}
		\KL(m_k^{(0)}\mathcal{N}(\mu_k,\sigma_k^2), m_k^{(1)}\mathcal{N}(\mu_k,\sigma_k^2))
		=& \frac 1 2\left( \log\left( \frac{\left(m_k^{(1)}\right)^2}{\left(m_k^{(0)}\right)^2}\right) +
		\frac{\left(m_k^{(0)}\right)^2}{\left(m_k^{(1)}\right)^2} - 1 \right) +
		\frac{\left(m_k^{(1)}-m_k^{(0)}\right)^2\mu_k^2}{2\left(m_k^{(1)}\right)^2\sigma_k^2}\\
		=& F\left( \left(  \frac{m_k}{m_k-1} \right)^2\right) + \frac{\mu_k^2}{2m_k^2\sigma_k^2}\\
		\le & F(4) + \frac{\mu_k^2}{2m_k^2\sigma_k^2}\\
		= & \underbrace{2\log 2 -0.75}_{\le 1} + \frac{\mu_k^2}{2m_k^2\sigma_k^2}\\
		\le & \frac{\mu_k^2}{m_k^2\sigma_k^2},
	\end{split}
\]
where the first inequality is due to that \(F(x)\) reaches its maximum in the largest \(x=\left( \frac{m_k}{m_k-1} \right)^2\), i.e., when \(m_k=2\),
and the last inequality is due to the condition that \(\mu_k^2/m_k^2\sigma_k^2\ge 2\).
In the case of binary set \(m_k^{(0)}=m_k, m_k^{(1)}=m_k+1\), a similar upper bound can also be derived.

To sum up, we obtain the KL-divergence terms' upper bounds as follows,
\[
	\KL(\P_0^l,\P_1^l)_{\{m_k^{(0)}<l\le m_k^{(1)}\}}
	\le \KL(\P_0^l,\P_1^l)_{\{m_k^{(1)}<l\le N\}} \le \frac{\mu_k^2}{m_k^2 \sigma_k^2}.
\]

\subsection{Sample Complexity Lower Bound (Theorem~\ref{thm:sample_lower_bound}): \\
	Identifying Whether a Capacity \(m_k\) is Greater Than \(d(\ge 2)\) or Not}\label{appsub:sample-complexity-lower-bound-second-part}

The second part's proof is similar to first's.
We highlight two differences in step 1 and step 3.

\textbf{Step 1: reduce the task to hypothesis testing. }
The original task is now to determine whether the capacity \(m_k\)
is in the set \(\{1,\dots, d-1\}\) or in the set \(\{d,\dots, N\}\).
This task can be reduced to find \(m_k\) from the binary set, e.g.,
\(\{d-1, m_k\}\)
in the case that \(m_k\ge d\) or \(\{m_k, d\}\) in the case that \(m_k<d\).

\textbf{Step 3: calculate the KL-divergence. }
Take the binary set \(\{d-1, m_k\}(d\le m_k)\) as an example.
In the case of binary set \(\{m_k, d\}(m_k < d)\), similar derivation also holds.
We can decompose the KL-divergence term and upper bound is as follows:
\begin{equation*}
	\begin{split}
		\KL(\P_0^{\otimes h_n}, \P_1^{\otimes h_n})
		=& \sum_{l=1}^{N} n_l\KL(\P_0^l,\P_1^l)
		= \sum_{l=1}^{d-1} n_l\KL(\P_0^l,\P_1^l) + \sum_{l=d}^{m_k} n_l\KL(\P_0^l,\P_1^l)
		+ \sum_{l=m_k+1}^{N} n_l\KL(\P_0^l,\P_1^l)\\
		\le&
		\frac{\sum_{l=d}^N n_l (m_k-d+1)^2\mu_k^2}{m_k^2 \sigma_k^2}
		\le \frac{n (m_k-d+1)^2\mu_k^2}{m_k^2 \sigma_k^2},
	\end{split}
\end{equation*}
where the first inequality is based on the inequality that \(
0 = \KL(\P_0^l,\P_1^l)_{\{0<l\le d-1\}}
\le \KL(\P_0^l,\P_1^l)_{\{d-1<l\le m_k\}}
\le \KL(\P_0^l,\P_1^l)_{\{m_k<l\le N\}} \le \frac{(m_k-d+1)^2\mu_k^2}{m_k^2 \sigma_k^2}
\) whose last inequality needs the condition that
\((m_k-d+1)^2\mu_k^2/(m_k^2\sigma_k^2) \ge 2\log(N/(d-1))\).
This is a counterpart condition to the \(\mu_k^2/m_k^2\sigma_k^2\ge 2\) condition in the first part's proof.

\subsection{Regret Lower Bound Proof}\label{app:regret_lower_bound}

We first state the definition of the consistent policies in Definition~\ref{def:consistent}.

\begin{definition}\label{def:consistent}
	A strategy \(\phi\) is \textit{consistent} if for all bandits environment, for all suboptimal action \(\bm{a}\), for all \(0 < \alpha \le 1\), it satisfies \(\E[N_{\phi,\bm{a}}(T)] = o(T^\alpha),\) where \(N_{\phi,\bm{a}}(T)\) is the number of times that the action \(\bm{a}\) is chosen in the strategy \(\phi\) up to time \(T\).
\end{definition}

\begin{proof}[Proof of Theorem~\ref{thm:regret_lower_bound}]

	This proof consists of two steps. In the first step, we bound the cost of exploring suboptimal arms.
	It is based on the classic result of MP-MAB~\citep{anantharam1987asymptotically}.
	In the second step, we utilize the sample complexity lower bound results of Theorem~\ref{thm:sample_lower_bound} and Remark~\ref{rmk:lower-bound-key-remark} to quantity the least cost of
	learning these top \(L-1\) optimal arms' reward capacities and the arm \(L\)'s capacity lower bound.

	\textit{We note that these two steps' regrets are orthogonal because the first step's regret is due to exploring suboptimal arms while the second step's regret is from learning optimal arms' reward capacities.}

	\textbf{Step 1: regret lower bound of exploring suboptimal arms.}

	We recall the uniformly good strategy definition from~\citet{anantharam1987asymptotically}.
	\begin{definition}[cf. {\cite{anantharam1987asymptotically}}]\label{def:uniformly-good}
		A strategy \(\phi\) is uniformly good on the MP-MAB problem if for all bandits environment, for all suboptimal arm \(k\), for all \(0 < \alpha \le 1\), it satisfies \(\E[N_{\phi,k}(T)] = o(T^\alpha),\) where \(N_{\phi,k}(T)\) is the number of times that the arm \(k\) is pulled in the strategy \(\phi\).
	\end{definition}

	Since Definition~\ref{def:consistent} guarantees that any suboptimal action would be
	selected only with \(o(T^\alpha)\) number of times, it implies
	that any suboptimal arm is also only pulled \(o(T^\alpha)\) times ---
	the uniformly good property in Definition~\ref{def:uniformly-good}. Then, we adapt the result of {MP-MAB} as follows:
	\begin{lemma}[{Adapted from~\citep[Theorem 3.1]{anantharam1987asymptotically}}]
		Let \(\phi\) be a uniformly good algorithm. For each suboptimal arm \(k\) and each \(\epsilon>0\),
		we have \[
			\liminf_{T\to \infty}\frac{\E[N_{\phi, k}(T)]}{\log T} \ge \frac{1}{\KL(v_k, v_L)},
		\]
		where arm \(L\) is the least favored optimal arm,
		$\KL$ represents KL-divergence,
		and $v_k$ is the reward distribution of arm $k$.
	\end{lemma}

	In our \texttt{MP-MAB-SA} model, when pulling a suboptimal arm \(k\), the smallest cost is \(\mu_L-\mu_k\eqqcolon \Delta_{L,k}\) (if the arm is not shared).
	So, under any uniformly good algorithm, the total cost of pulling suboptimal arms (\(k>L\)) in \texttt{MP-MAB-SA} is asymptotically lower bounded as follows \[
		\sum_{k=L+1}^K \frac{\Delta_{L,k}}{\kl(\mu_k, \mu_L)}\log T,
	\]
	where we use \(\kl\) for Gaussian distributions with the same variance to replace the general KL-divergence \(\KL\).



	\textbf{Step 2: regret lower bound of learning optimal arms' reward capacities.}

	For any consistent strategy, it chooses the optimal action \(\bm{a}^*\) ``most of the time''.
	That is, after finishing all \(T\) rounds of arm pulling,
	the optimal action \(\bm{a}^*\) is selected with the highest frequency.
	From this evidence, one can recognize the optimal action \(\bm{a}^*=(m_1, \dots, m_{L-1}, \bar{m}_L,0,\dots, 0)\)
	from any consistent strategy's action sequence.
	Then, from the optimal action \(\bm{a}^*\),
	one can \textit{``read out''} top \(L-1\) optimal arms' capacities and the least favored optimal arm \(L\)'s
	capacity lower bound \(\bar{m}_L\).
	This is equivalent to learn these capacities (or its lower bound).
	Therefore,
	any consistent strategy actually finish the learning task
	and spends at least
	the sample complexity lower bound's number of explorations on these optimal arms.

	Recall in Remark~\ref{rmk:lower-bound-key-remark} we show:
	the number of ``irregular`` explorations ---
	where the number of plays exploring an arm is greater than the arm's capacity
	--- spent to learn an arm's capacity
	(or validate whether it is no less than an integer or not)
	should be no less the task's sample complexity lower bound.
	Each of these ``irregular'' explorations contributes a cost to regret.

	For any top \(L-1\) optimal arm \(k\), one needs to spend \(\frac{\sigma_k^2m_k^2}{\mu_k^2}\log(1/4\delta)\) number of
	explorations to accurately learn its capacity \(m_k\) with a confidence of at most \(1-\delta\).
	Each of these exploration costs at least \(\mu_k - \mu_L \eqqcolon \Delta_{k,L}\).
	If the estimation fails, it would leads to a linear cost at least \(\Delta_{k,L}\cdot T\).
	To sum up, the cost of learning arm \(k (<L)\)'s capacity is at least \[
		\begin{split}
			\Delta_{k,L}\cdot \frac{\sigma_k^2m_k^2}{\mu_k^2}\log(1/4\delta) + \delta \cdot \Delta_{k,L}T \ge
			\Delta_{k,L}\cdot \frac{\sigma_k^2m_k^2}{\mu_k^2}\left( \log \frac{\mu_k^2 T}{4\sigma_k^2m_k^2} + 1\right),
		\end{split}
	\]
	where the LHS reaches its minimum by letting \(\delta = \frac{\sigma_k^2m_k^2}{\mu_k^2 T}\).

	Similarly, for the least favored optimal arm \(L\), the least cost of identifying that the capacity is no less than \(\bar{m}_L\) is at least \[
		\begin{split}
			&\quad\Delta_{L,L+1}\cdot \frac{\sigma_L^2m_L^2}{(m_L - \bar{m}_L + 1)^2\mu_L^2}\log(1/4\delta) + \delta \cdot \Delta_{L, L+1}T \\
			&\ge
			\Delta_{L,L+1}\cdot \frac{\sigma_L^2m_L^2}{(m_L - \bar{m}_L + 1)^2\mu_L^2}
			\left( \log \frac{(m_L - \bar{m}_L + 1)^2\mu_L^2 T}{4\sigma_L^2m_L^2} + 1\right),
		\end{split}
	\]
	where the LHS's minimum is reached when \(\delta = \frac{\sigma_L^2m_L^2}{(m_L - \bar{m}_L + 1)^2\mu_L^2 T}\).

	Summing up the above costs and let \(T\to\infty\), we show the total cost in this part is asymptotically lower bounded as follows:
	\[
		\left( \sum_{k=1}^{L-1}\frac{\Delta_{k,L}\sigma_k^2 m_k^2}{\mu_k^2} + \frac{\Delta_{L,L+1}\sigma_L^2m_L^2}{(m_L - \bar{m}_L + 1)^2\mu_L^2} \right) \log T.
	\]
\end{proof}

\section{Learning Reward Capacity's Proofs}\label{app:capacity_proof}

\subsection{Uniform Confidence Interval for Reward Capacity: Proof of Lemma~\ref{lma:ufc_m}}\label{appsub:uniform-confidence-interval}
    We apply the following Lemma~\ref{lma:uci_basic} to measure $\hat{\mu}_{k,t}$ and $\hat{\nu}_{k,t}$'s uncertainty.
    \begin{lemma}[{\citep[Lemma 5]{bourel2020tightening}}]\label{lma:uci_basic}
        Let $Y_1, \ldots, Y_t$ be a sequence of $t$ i.i.d. real-valued random variables with mean $\mu$, such that $Y_t- \mu$ is $\sigma$-sub-Gaussian. Let $\mu_t=\frac{1}{t}\sum_{s=1}^t Y_s$ be the empirical mean estimate. Then, for all $\sigma\in(0,1)$, it holds
        \begin{equation*}
            \mathbb{P}\left(\exists t \in \mathbb{N}, \abs{\mu_t - \mu} \ge \sigma \sqrt{\left(1+\frac{1}{t}\right)\frac{2\log(\sqrt{t+1}/\delta)}{t}}\right) \le \delta.
        \end{equation*}
    \end{lemma}

    Note that $X_k\in [0,1]$ is $1/2$-sub-Gaussian. Let $\sigma\gets 1/2$, $t\gets \tau_{k,t}$ and $\delta\gets \delta/2$ in Lemma~\ref{lma:uci_basic} we have
    \begin{equation*}
        \mathbb{P}(\exists \tau_{k,t}\in \mathbb{N}_+, \abs{\hat{\mu}_{k,t} - \mu_k} \ge \phi(\tau_{k,t},\delta) ) \le \delta/2,
    \end{equation*}
    where \begin{equation*}
        \phi(\tau_{k,t}, \delta) =\sqrt{\left(1+\frac{1}{\tau_{k,t}}\right)\frac{\log (2\sqrt{\tau_{k,t}+1}/\delta)}{2\tau_{k,t}}}
    \end{equation*}
    as we defined in the lemma.
    Then, the complementary event's probability is lower bounded as follows
    \begin{equation}\label{eq:mu_uci}
        \mathbb{P}(\forall \tau_{k,t}\in \mathbb{N}_+, \abs{\hat{\mu}_{k,t} - \mu_k} \le \phi(\tau_{k,t},\delta) ) \ge 1- \delta/2.
    \end{equation}
    Similarly, with a $1/m_k$ scaling for $\hat{\nu}_{k,t}$, we would have
    \begin{equation}\label{eq:nu_uci}
        \mathbb{P}(\forall \iota_{k,t}\in \mathbb{N}_+, \abs{\hat{\nu}_{k,t} - m_k\mu_k} \le m_k\phi(\iota_{k,t},\delta) ) \ge 1- \delta/2.
    \end{equation}

    The confidence intervals of Eq.(\ref{eq:mu_uci}) and Eq.(\ref{eq:nu_uci}) are as follows
    \begin{equation*}
        \begin{split}
            &{\mu}_{k} \in \left[\hat{\mu}_{k,t} - \phi(\tau_{k,t},\delta), \hat{\mu}_{k,t} + \phi(\tau_{k,t},\delta)\right],\\
            &m_k\mu_k \in \left[\hat{\nu}_{k,t} - m_k\phi(\iota_{k,t},\delta), \hat{\nu}_{k,t} + m_k\phi(\iota_{k,t},\delta)\right].
        \end{split}
    \end{equation*}
    Rearranging the second interval, we have \[
        m_k \in \left[ \frac{\hat{\nu}_{k,t}}{\mu_k + \phi(\tau_{k,t},\delta)}, \frac{\hat{\nu}_{k,t}}{\mu_k - \phi(\tau_{k,t},\delta)} \right].
    \]
    Then, via substituting the interval's two endpoints' \(\mu_k\) with \(\hat{\mu}_{k,t} + \phi(\tau_{k,t},\delta)\) and \(\hat{\mu}_{k,t} - \phi(\tau_{k,t},\delta)\) respectively,
    the above interval reduce to
    \begin{equation*}
        m_k\in
        \left[\frac{\hat{\nu}_{k,t}}{\hat{\mu}_{k,t}+\phi(\tau_{k,t},\delta) + \phi(\iota_{k,t},\delta)},
            \frac{\hat{\nu}_{k,t}}{\hat{\mu}_{k,t}-\phi(\tau_{k,t},\delta)-\phi(\iota_{k,t},\delta)}\right],
    \end{equation*}

    Finally, applying the union bound to Eq.(\ref{eq:mu_uci}) and Eq.(\ref{eq:nu_uci}), we have
    \begin{equation*}
        \begin{split}
            \mathbb{P}\left(\forall \tau_{k,t}, \iota_{k,t} \in \mathbb{N}_+, m_k\in
            \left[\frac{\hat{\nu}_{k,t}}{\hat{\mu}_{k,t}+\phi(\tau_{k,t},\delta) + \phi(\iota_{k,t},\delta)},
                \frac{\hat{\nu}_{k,t}}{\hat{\mu}_{k,t}-\phi(\tau_{k,t},\delta)-\phi(\iota_{k,t},\delta)}\right]\right)
            \ge 1-\delta.
        \end{split}
    \end{equation*}


\subsection{Sample Complexity Upper Bound: Proof of Theorem~\ref{cor:m_sample_complexity}}\label{appsub:proof-sample-complexity-upper-bound}
    \textbf{The estimator's sample complexity upper bound proof. }
    From Lemma~\ref{lma:ufc_m} and Lemma~\ref{lma:learning_m_criterion} and that $m_k\in\mathbb{N}_+$, we learn $m_k$ before the interval width is less than $1$. That is,
    \begin{equation*} \begin{split}
            \frac{\hat{\nu}_{k,t}}{\hat{\mu}_{k,t}-\phi(\tau_{k,t},\delta)-\phi(\iota_{k,t},\delta)}
            -
            \frac{\hat{\nu}_{k,t}}{\hat{\mu}_{k,t}+\phi(\tau_{k,t},\delta) + \phi(\iota_{k,t},\delta)}
            \le 1.
        \end{split}
    \end{equation*}
    It reduces to
    \[ \begin{split}
            (\phi(\tau_{k,t},\delta) + \phi(\iota_{k,t},\delta))^2 +2\hat{\nu}_{k,t}(\phi(\tau_{k,t},\delta)+\phi(\iota_{k,t},\delta))
            -\hat{\mu}_{k,t}^2 \le 0.
        \end{split}
    \]
    Replace $\hat{\nu}_{k,t}$ and $\hat{\mu}_{k,t}$ with their confidence upper and lower bounds respectively, we further have
    \[
        \begin{split}
            (\phi(\tau_{k,t},\delta) + \phi(\iota_{k,t},\delta))^2 -(\mu_k - \phi(\tau_{k,t},\delta))^2
            +2(m_k(\mu_k + \phi(\iota_{k,t},\delta)))(\phi(\tau_{k,t},\delta)+\phi(\iota_{k,t},\delta))  \le 0.
        \end{split}
    \]
    Rearrange the terms, it becomes
    \[
        \begin{split}
            ((2m_k+2)\phi(\tau_{k,t},\delta) + (2m_k+1)\phi(\iota_{k,t},\delta) -\mu_k )
            \times (\mu_k + \phi(\iota_{k,t}))\le 0.
        \end{split}
    \]
    As the term \((\mu_k + \phi(\iota_{k,t}))\) in LHS is positive, we finally have
    \[
        (2m_k + 2)\phi(\tau_{k,t},\delta) + (2m_k + 1)\phi(\iota_{k,t},\delta) \le \mu_k.
    \]

    One solution is to require both $\phi(\tau_{k,t},\delta)$ and $\phi(\iota_{k,t},\delta)$ no greater than  $\frac{\mu_k}{7m_k}$.
    Solving these, we have
    \begin{equation*}
        \tau_{k,t},\iota_{k,t} \ge \frac{49m_k^2\log (2/\delta)}{\mu_k^2},
    \end{equation*}
    where \(0<\delta \le 2\exp(-49m_k^2 / \mu_k^2)\).

    \textbf{The proof of the sample complexity upper bound for identifying whether an arm's capacity \(m_k(\ge d)\)
        is greater than integer \(d(\ge 2)\) or not.}
    With the assumption \(m_k \ge d\),
    we only needs to show the lower confidence interval \(m_{k,t}^l\) is greater than \(d-1\),
    i.e., \(m_{k,t}^l > d-1\). That is,
    \begin{equation*} \begin{split}
            \frac{\hat{\nu}_{k,t}}{\hat{\mu}_{k,t}+\phi(\tau_{k,t},\delta) + \phi(\iota_{k,t},\delta)}
            > d-1.
        \end{split}
    \end{equation*}
    Replace $\hat{\nu}_{k,t}$ and $\hat{\mu}_{k,t}$ with their confidence lower and upper bounds respectively and rearrange terms as the procedure in the first part of proof, we have
    \[
        2(d-1)\phi(\tau_{k,t},\delta) + (d-1+m_k)\phi(\iota_{k,t},\delta) \le (m_k - d+1)\mu_k.
    \]
    One solution is to require both $\phi(\tau_{k,t},\delta)$ and $\phi(\iota_{k,t},\delta)$ no greater than  $\frac{(m_k-d+1)\mu_k}{7m_k}$.
    Solving these, we have
    \begin{equation*}
        \tau_{k,t},\iota_{k,t} \ge \frac{49m_k^2\log (2/\delta)}{(m_k-d+1)^2\mu_k^2},
    \end{equation*}
    where \(0<\delta \le 2\exp(-49m_k^2 /(m_k-d+1)^2 \mu_k^2)\).



\section{Proof of the \texttt{OrchExplore} Algorithm's Regret Upper Bound (Theorem~\ref{thm:oe_regret_upper_bound})}
\label{app:orchexplore_analysis}


We first state two useful lemmas as building blocks in this section's proof.
\begin{lemma}[{\cite{wang2020optimal}'s Lemma 3}]
    \label{lma:empirical_mean}
    Let $k\in [K]$, and \(c>0\). Let \(H\) be a random set of rounds such that for all \(t,\, \{t\in H\}\in \mathcal{F}_{t-1}\).
    Assume that there exists \((C_t)_{t\ge 0}\), a sequence of independent binary random variables such that for any \(t\ge 1\),
    \(C_t\) is \(\mathcal{F}_t\)-measurable and \(\P[C_t=1]\ge c\). Further assume for any \(t\in H\), \(k\) is selected (\(a_{k,t}>0\)) if \(C_t=1\). Then,
    \[
        \sum_{t\ge 1}\P\left[\{ t\in H, \abs{\hat{\mu}_{k,t} - \mu_k} \ge \varepsilon \}\right] \le 2c^{-1}(2c^{-1} + \varepsilon^{-2}).
    \]
\end{lemma}

\begin{lemma}\label{lma:kl_ucb}
    In the \emph{\texttt{OrchExplore}} algorithm, for any arm \(k\in[K]\), we have
    \[
        \sum_{t\ge 0} \P[u_{k,t} < \mu_k] \le 30.
    \]
\end{lemma}
\begin{proof}[Proof of Lemma~\ref{lma:kl_ucb}]
    In \texttt{OrchExplore}, we update the KL-UCB index \(u_{k,t}\) at least once every two time slots.
    So, we have \[
        \sum_{t\ge 0} \P[u_{k,t} < \mu_k]  \le 2\sum_{t'\ge 0} \P[u_{k,t'} < \mu_k],
    \]
    where \(t'\) represents the time slots when the KL-UCB index \(u_{k,t}\) is updated.
    Utilizing~\cite{combes2015learning}'s Lemma 6, we have \(\sum_{t'\ge 0} \P[u_{k,t'} < \mu_k]
    \le 15\). Hence, we show \(\sum_{t\ge 0} \P[u_{k,t} < \mu_k]  \le 30.
    \)
\end{proof}

\textbf{Step 1: show that the pulls of suboptimal arms are mainly caused by the deliberate explorations in PIE (i.e., line~\ref{alg:oe:par_exp}).}


Given the capacity confidence lower bound \(\bm{m}_t^l\),
recall the action \(\bm{a}_t^{\ie}\) is defined as \(\bm{a}_t^{\ie}\coloneqq \Oracle(\hatbm{\mu}_t, \bm{m}^l_t)\).
Note that in this step's proof, we use \(\bm{a}_t^{\ie}\) to denote
the original output of \(\Oracle\) without the play rearrangement caused by deliberate explorations.
We define another action \(\bm{a}^{\ie,*}_t \coloneqq \Oracle(\bm{\mu}, \bm{m}_t^l)\) which
takes the true ``per-load'' reward mean \(\bm{\mu}\) as its input.
Especially, we denote \(\mathcal{S}_t^*\coloneqq \{k:a_{k,t}^{\ie,*}>0\}\) as the set of arms pulled in \(\bm{a}^{\ie,*}_t\).
Since the input reward means are correct and the estimated capacity lower bound \(\bm{m}_t^l\) is no greater than true capacity \(\bm{m}\),
the set \(\mathcal{S}_t^*\) is a subset of top \(N\) arms \(\{1,2,\dots,N\}\).
Let $0<\varepsilon< \min_{k=1}^{K-1}\frac{\mu_k - \mu_{k+1}}{2}$. We define several time slot sets as follows,
\[
    \begin{split}
        \mathcal{A} \coloneqq & \{t\ge 1: \bm{a}_t^{\ie} \neq \bm{a}^{\ie,*}_t\}, \\
        \mathcal{B} \coloneqq & \{t\ge 1: \exists k\in [K]\text{ s.t. } {a}^{\ie}_{k,t} > 0, \abs{\hat{\mu}_{k,t} -\mu_k}\ge \varepsilon\},\\
        \mathcal{C} \coloneqq & \{t\ge 1: \exists k\in[K]\text{ s.t. } a^{\ie,*}_{k,t} > 0, u_{k,t}< \mu_k\},\\
        \mathcal{D} \coloneqq & \{t\ge 1: t \in \mathcal{A}\setminus (\mathcal{B}\cup\mathcal{C}), \exists k\in [K] \text{ s.t. }  a_{k,t}^{\ie,*}>0, {a}^{\ie}_{k,t}=0, \abs{\hat{\mu}_{k,t} - \mu_k}\ge \varepsilon\}.
    \end{split}
\]

\begin{lemma}\label{lma:kc_initial_bound_1}
    \(
    \mathcal{A}\cup\mathcal{B} \subseteq \mathcal{B}\cup\mathcal{C}\cup\mathcal{D}
    \) and thus \(
    \E[\abs{\mathcal{A}\cup\mathcal{B}}] \le  \E[\abs{\mathcal{B}}]+\E[\abs{\mathcal{C}}]+\E[\abs{\mathcal{D}}].
    \)
\end{lemma}
\begin{proof}[{Proof of Lemma~\ref{lma:kc_initial_bound_1}}]
    This proof is similar to~\citep[Lemma 5]{wang2020optimal}. Denote \(t\in \mathcal{A}\setminus (\mathcal{B}\cup \mathcal{C})\). To prove the lemma, we need to show that \(t\in\mathcal{D}\).
    Since \(t\notin \mathcal{B}\), for all \(k\in [K]\) such that \({a}^{\ie}_{k,t}>0\), we have \begin{equation}\label{eq:kc_initial_bound_1:eq1}
        \abs{\hat{\mu}_{k,t} - \mu_k} < \varepsilon.
    \end{equation}
    Then, for \(t\in \mathcal{A}\setminus\mathcal{B}\), the arm set \(\mathcal{S}^*_t\coloneqq\{k:a_{k,t}^{\ie,*} > 0\}\) is different from the empirical optimal arm set \(\mathcal{S}_{t}\coloneqq\{k:{a}^{\ie}_{k,t} > 0\}\).
    Because \(t\notin \mathcal{B}\) implies that the order of empirical reward means of arms in \(\mathcal{S}_t\)
    is the same as the order of these arms' true reward means',
    and thus \(\mathcal{S}^*_t = \mathcal{S}_{t}\) is equivalent to \(\bm{a}^{\ie,*}_t=\bm{a}_t^{\ie}\),
    which contradicts \(t\in\mathcal{A}\).
    So, there exists an arm \(j\in\mathcal{S}^*_t\setminus \mathcal{S}_t\,(a^{\ie,*}_{j,t}>0, {a}_{j,t}^{\ie}=0)\) such that
    \begin{equation}
        \label{eq:kc_initial_bound_1:eq2}
        \hat{\mu}_{j,t} < \hat{\mu}_{k,t}\text{ for some arm }k\in\mathcal{S}_t\setminus \mathcal{S}^*_t\, ({a}^{\ie}_{k,t}>0, a^{\ie,*}_{k,t}=0).
    \end{equation}
    Combining (\ref{eq:kc_initial_bound_1:eq1}) and (\ref{eq:kc_initial_bound_1:eq2}) leads to \(
    \hat{\mu}_{j,t} < \hat{\mu}_{k,t} \le \mu_k + \varepsilon \le \mu_j - \varepsilon.
    \)
    The last inequality is due to that \(j < k\) (notice that reward means are in a descending order, and \(j\in\mathcal{S}^*_t, k\not \in \mathcal{S}^*_t\))
    and \(\varepsilon < \varepsilon_0\).
    It implies \(\abs{\hat{\mu}_{j,t} - \mu_j}\ge \varepsilon\) and thus, \(t\in\mathcal{D}\).
    Therefore, \(\mathcal{A}\cup\mathcal{B}\subseteq \mathcal{B}\cup\mathcal{C}\cup\mathcal{D}\).
\end{proof}


\begin{lemma}\label{lma:kc_initial_bound_2}
    \(
    \E[\abs{\mathcal{B}}] +\E[\abs{\mathcal{C}}] + \E[\abs{\mathcal{D}}] \le 12K^2 N (4+\varepsilon^{-2})
    \)
\end{lemma}

\begin{proof}[Proof of Lemma~\ref{lma:kc_initial_bound_2}]

    \emph{To show $\E[\abs{\mathcal{B}}]\le 8K(4+\varepsilon^{-2})$. } Let \(
    \mathcal{B}_k\coloneqq \{t\ge 1: {a}^{\ie}_{k,t} > 0, \abs{\hat{\mu}_{k,t} - \mu_k}\ge \varepsilon\},
    \)
    we have $\mathcal{B} = \cup_{1\le k \le K}\mathcal{B}_k$. Then, we define \[
        \begin{split}
            \mathcal{B}_k^{\ie} &\coloneqq \{t \text{ is in PIE}: {a}^{\ie}_{k,t} > 0, \abs{\hat{\mu}_{k,t} - \mu_k}\ge \varepsilon\},\\
            \mathcal{B}_k^{\ue} &\coloneqq \mathcal{B}_k\setminus\mathcal{B}_k^{\ie}.
        \end{split}
    \]

    We upper bound the cardinality of \(\mathcal{B}_k^{\ie}\), i.e., \(\abs{\mathcal{B}_k^{\ie}}\).
    In Lemma~\ref{lma:empirical_mean}, we set
    \(H = \{t \text{ is in PIE}: {a}^{\ie}_{k,t} > 0\}, C_t=\1{\text{arm }k\text{ is pulled in time slot }t}\)
    and thus $\P(C_t = 1) \ge \frac{1}{2}$
    (because arm \(k\) may not be pulled due to the deliberate exploration with a probability of \(1/2\)).
    Then, we have $\E[\abs{\mathcal{B}_{k}^{\ie}}]\le 4(4+\varepsilon^{-2})$.

    Observe that for any \(t\in \mathcal{B}_k^{\ue}\), there is a injective time slot \(t'\in\mathcal{B}_k^{\ie}\).
    Because \(a_{k,t}^{\ie}\) is only updated in PIE, and
    each PUE round always has a PIE round in its preceding time slot.
    Hence, we have \(\abs{\mathcal{B}_k^{\ue}} \le \abs{\mathcal{B}_k^{\ie}}.\)

    So, \(\E[\abs{\mathcal{B}}]\le \sum_{k=1}^K \E[\abs{\mathcal{B}_{k}}] = \sum_{k=1}^K (\E[\abs{\mathcal{B}_{k}^{\ie}}] + \E[\abs{\mathcal{B}_{k}^{\ie}}]) \le 2\sum_{k=1}^K \E[\abs{\mathcal{B}_{k}^{\ie}}] \le 8K(4+\varepsilon^{-2}).\)

    \emph{To show $\E[\abs{\mathcal{C}}]\le 30N$. } Denote $\mathcal{C}_k \coloneqq \{t\ge 1: u_{k,t} > \mu_k\}$.
    Notice that the set \(\mathcal{S}^*_t = \{k:a_{k,t}^{\ie,*} > 0\}\) is a subset of top \(N\) arms \(\{1,2,\dots, N\}\).
    We have $\mathcal{C} \subseteq \cup_{k=1}^N \mathcal{C}_k$ and thus
    \[
        \E[\abs{\mathcal{C}}]\le \sum_{k=1}^N\E[\abs{\mathcal{C}_k}] \le 30N,
    \]
    where the second inequality holds by Lemma~\ref{lma:kl_ucb}.

    \emph{To show $\E[\abs{\mathcal{D}}]\le 8K^2 N(4+\varepsilon^{-2})$. }
    Denote $\mathcal{D}_k \coloneqq \{t\ge 1: t \in \mathcal{A}\setminus (\mathcal{B}\cup\mathcal{C}), a_{k,t}^{\ie,*} > 0, {a}^{\ie}_{k,t}=0, \abs{\hat{\mu}_{k,t} -\mu_k}\ge \varepsilon\}$.
    We have $\mathcal{D} =\cup_{k=1}^N\mathcal{D}_k$.
    Then, we define \[
        \begin{split}
            \mathcal{D}_k^{\ie} &\coloneqq \{t\text{ is in PIE}: t \in \mathcal{A}\setminus (\mathcal{B}\cup\mathcal{C}), a_{k,t}^{\ie,*} > 0, {a}^{\ie}_{k,t}=0, \abs{\hat{\mu}_{k,t} -\mu_k}\ge \varepsilon\},\\
            \mathcal{D}_k^{\ue} &\coloneqq \mathcal{D}_k\setminus \mathcal{D}_k^{\ie}.
        \end{split}
    \]

    We first bound the cardinality of \(\mathcal{D}_k^{\ie}\).
    As $t\notin\mathcal{C}$ we have $u_{k,t} > \mu_k\ge \mu_{k^*}$
    where \(k^*\coloneqq \max\{k:k\in\mathcal{S}^*_t\}\) is the largest index in \(\mathcal{S}^*_t\).
    As $t\notin \mathcal{B}$
    the empirical reward means of arms in \(\mathcal{S}_t\) have the same order as these arms' true reward means
    and $\mu_{L_t} + \varepsilon > \hat{\mu}_{L_t,t}$. As $t\in \mathcal{A}$ (thus \(\mathcal{S}^*_t\neq \mathcal{S}_t\)), we know the empirical least favored arm's index $L_t$ in \(\mathcal{S}_t\) is greater than $k^*$ in \(\mathcal{S}_t^*\), that is,
    $\mu_{k^*} \ge \mu_{L_t} + \varepsilon$.
    Together they lead to $u_{k,t} \ge \hat{\mu}_{L_t,t}$, i.e., arm $k\in\mathcal{E}_t$. As the exploration arm is selected uniformly from $\mathcal{E}_t$, we know $\P(a_{k,t} > 0) \ge \frac{1}{2K}$.
    That is, when \(t\in \mathcal{D}_{k}\),
    there is a probability of at least \(1/2K\) to explore the arm \(k\in\mathcal{E}_t\) in PIE.
    In Lemma~\ref{lma:empirical_mean}, let $H=\{t\text{ is in PIE, }t\in\mathcal{A}\setminus (\mathcal{B}\cup\mathcal{C}), a_k^* > 0, \hat{a}^{*}_{k,t} = 0\}, C_t = \1{\text{arm }k\text{ is pulled in time slot }t}$
    and \(\P(C_t=1) = \P(a_{k,t}>0)\ge \frac{1}{2K}\), we have $\E[\abs{\mathcal{D}_{k}^{\ie}}] \le 4K(4K + \varepsilon^{-2})\le 4K^2(4+\varepsilon^{-2})$.

    Observe that for any \(t\in \mathcal{D}_k^{\ue}\), there is a injective time slot \(t'\in\mathcal{D}_k^{\ie}\).
    Because \(a_{k,t}^{\ie}\) is only updated in PIE, and PUE and PIE are executed in turn.
    Hence, \(\abs{\mathcal{D}_k^{\ue}} \le \abs{\mathcal{D}_k^{\ie}}.\)

    We obtain that $\E[\abs{\mathcal{D}}]
        \le \E\left[\sum_{k=1}^N \abs{\mathcal{D}_{k}}\right]
        = \E\left[\sum_{k=1}^N (\abs{\mathcal{D}_{k}^{\ie}} + \abs{\mathcal{D}_{k}^{\ue}} )\right]
        \le  2\E\left[\sum_{k=1}^N \abs{\mathcal{D}_{k}^{\ie}}\right]
        \le 8K^2N(4+\varepsilon^{-2})$.

    Summing up the above three upper bounds concludes the proof.
\end{proof}

\textbf{Step 2: upper bound the cost of the suboptimal arms' deliberate explorations in PIE (line~\ref{alg:oe:par_exp}).}

\begin{lemma}\label{lma:kc_explore_bound}
    Denote \(\mathcal{G}_k \coloneqq \{t\le T: t\notin \mathcal{A}\cup\mathcal{B}, \bm{a}_t^{\ie} = \bm{a}^{\ie,*}_t, \mathcal{Y}_t = \emptyset, a_{k,t} > 0\}\) for a arm $k\not \in \mathcal{S}^*_t$. We have
    \[
        \E[\abs{\mathcal{G}_k}] \le \frac{\log T + 4\log \log T}{\kl(\mu_k+\varepsilon, \mu_L-\varepsilon)} + 2(2+\varepsilon^{-2}).
    \]
\end{lemma}

Note that \(\mathcal{Y}_t=\emptyset\) and \(\bm{a}_t^{\ie} = \bm{a}^{\ie,*}_t\)
imply that \(\bm{a}_t^{\ie}\) is equal to the optimal action \(\bm{a}^*\)
because \(\mathcal{Y}=\emptyset\) means that the empirical optimal arms' capacities are learnt.
So, both \(\mathcal{S}_t^*\) and \(\mathcal{S}_t\) are equal to the optimal arm set \(\{1,2,\dots,L\}\)
and the arm \(k \not\in \mathcal{S}_t^*\) is suboptimal.
Therefore, the \(\{a_{k,t}>0\}\) can only happen in PIE's deliberate explorations,
and the event \(\mathcal{G}_k\) corresponds to these deliberate explorations.

\begin{proof}
    Denote
    $t_0 \coloneqq \frac{\log T + 4 \log\log T}{\kl(\mu_k+\varepsilon, \mu_L -\varepsilon)}$ and
    \[
        \begin{split}
            \mathcal{G}_{k,1}\coloneqq & \{t\in\mathcal{G}_k: \abs{\hat{\mu}_{k,t} - \mu_k}\ge \varepsilon\},\\
            \mathcal{G}_{k,2}\coloneqq & \left\{t \in \mathcal{G}_k: \sum_{\kappa=1}^t \1{\kappa\in\mathcal{G}_k} \le t_0  \right\}.
        \end{split}
    \]

    \emph{To show $\mathcal{G}_k \subseteq \mathcal{G}_{k,1}\cup \mathcal{G}_{k,2}$. } Let $t \in \mathcal{G}_k\setminus (\mathcal{G}_{k,1}\cup\mathcal{G}_{k,2})$.

    As $k\in\mathcal{G}_k$ we have $u_{k,t} \ge \hat{\mu}_{L, t}$. As $t\notin \mathcal{A}\cup\mathcal{B}$ we have $\hat{\mu}_{L,t} \ge \mu_L - \varepsilon$. As arm $k$ is suboptimal, we have $\mu_L - \varepsilon \ge \mu_k + \varepsilon$. As $k\notin \mathcal{G}_{k,1}$, we have $\mu_k+\varepsilon \ge \hat{\mu}_{k,t}$. Together, these lead to $\hat{\mu}_{k,t} < \mu_L-\varepsilon < u_{k,t}$.

    From $k\notin \mathcal{G}_{k,2}$, we have $t_0\le \sum_{\kappa=1}^t \1{\kappa  \in \mathcal{G}_k}\le \hat{\tau}_{k,t}$,
    where \(\hat{\tau}_{k,t}\) is the total number of times of IEs for arm \(k\).
    \[
        \begin{split}
            t_0 \kl(\hat{\mu}_{k,t}, \mu_L -\varepsilon) \le \hat{\tau}_{k,t} \kl(\hat{\mu}_{k,t}, \mu_L - \varepsilon)
            \le \hat{\tau}_{k,t}\kl(\hat{\mu}_{k,t}, u_{k,t})
            \le  \log T + 4 \log\log T,
        \end{split}
    \]
    where the second inequality holds for $y\mapsto\kl(x,y)$ is increasing for $0<x<y<1$, and the last inequality holds for
    the KL-UCB index $u_{k,t}$'s definition.

    Substituting $t_0$ with its definition expression, we obtain $\kl(\hat{\mu}_{k,t}, \mu_L - \varepsilon) \le \kl(\mu_k + \varepsilon, \mu_L - \varepsilon)$. Note that $x\mapsto \kl(x,y)$ is decreasing for $0<x<y<1$, which further leads to $\hat{\mu}_{k,t} \ge \mu_k + \varepsilon$.
    This \emph{contradicts} the assumption that $t\notin\mathcal{G}_{k,1}$. So, $\mathcal{G}_k \subseteq \mathcal{G}_{k,1}\cup\mathcal{G}_{k,2}$.

    \emph{To bound $\E[\abs{\mathcal{G}_{k,1}}]$ and $\E[\abs{\mathcal{G}_{k,2}}]$. }
    In Lemma~\ref{lma:empirical_mean}, let $H=\{t\in \mathcal{G}_k\}, C_t = 1$, we have $\E[\abs{\mathcal{G}_{k,1}}]\le 2(2+\varepsilon^{-2})$. For $\mathcal{G}_{k,2}$, we have \(\E[\abs{\mathcal{G}_{k,2}}]\le t_0.\)
    Substituting \(\E[\abs{\mathcal{G}_{k,1}}]\) and \(\E[\abs{\mathcal{G}_{k,2}}]\) by their upper bound in the inequality \(\E[\abs{\mathcal{G}_k}] \le \E[\abs{\mathcal{G}_{k,1}}] + \E[\abs{\mathcal{G}_{k,2}}]\), we prove that:
    \[
        \E[\abs{\mathcal{G}_k}] \le \frac{\log T + 4\log \log T}{\kl(\mu_k+\varepsilon, \mu_L-\varepsilon)} + 2(2+\varepsilon^{-2}).
    \]
\end{proof}

There are also some deliberate explorations outside \(\mathcal{G}_{k,t}\)
when \(\mathcal{Y}_t\neq \emptyset\) and \(t\notin \mathcal{A}\cup\mathcal{B}\).
Each of these explorations (in PIE) has a consequent PUE round since \(\mathcal{Y}_t\) is not empty.
We count their costs in the next step, together with PUE's.

\textbf{Step 3:  upper bound the cost of united explorations for optimal arms in PUE. }

When \(t\not\in {\mathcal{A}}\cup{\mathcal{B}}\), arms are unitedly explored in the order
that is the same as their true reward means'.
This is due to the the definition of event \({\mathcal{A}}\) and \({\mathcal{B}}\).
For example, only after arm \(1\) (the best arm)'s capacity is learnt then can PUE start to explore arm
\(2\).
With the correct exploration order,
when top \(L-1\) optimal arms' capacities are learnt
and the least favor optimal arm \(L\)'s capacity lower confidence bound are
verified to be no less than \(\bar{m}_L\coloneqq N-\sum_{k=1}^{L-1}m_k\),
the PUE set \(\mathcal{Y}_t\) will become empty and no suboptimal arm will be unitedly explored.

Although, when \(t\in \mathcal{A}\cup\mathcal{B}\), some suboptimal arms may be unitedly explored,
the number of times for \(t\in\mathcal{A}\cup\mathcal{B}\) is finite (Lemma~\ref{lma:kc_initial_bound_1} and Lemma~\ref{lma:kc_initial_bound_2}).
These costs are covered in step 1.
So, in step 3, we only need to upper bound the cost of UEs for optimal arms.

To measure how many number of times of UEs are enough to learn these top \(L-1\) optimal arms' reward capacities,
we choose the confidence \(1-\delta\) of Theorem~\ref{cor:m_sample_complexity}
as \(1-2/T\) and obtain the following lemma:
\begin{lemma}\label{lma:ue_sample_size}
    For any arm \(k\) and \(T\ge \exp(49m_k^2 / \mu_k^2)\), the inequality \(\P(\hat{m}_{k,t} = m_k)\ge 1-(2/T)\) holds if \[
        \hat{\tau}_{k,t},\hat{\iota}_{k,t} \ge  \frac{49m_k^2}{\mu_k^2}\log T.
    \]
\end{lemma}
Also notice that for any arm \(1\le k\le L\) and any time \(t \le T\),
the number of times of UEs on the arm \(\hat{\iota}_{k,t}\) is always smaller than the number of IEs on this arm \(\hat{\tau}_{k,t}\).
Because PUE always choose arms from \(\mathcal{Y}_t\) to explore and arms in \(\mathcal{Y}_t\) must have been
explored once by PIE in the prior time slot.
So, we only need to make sure the number of UEs \(\hat{\iota}_{k,t}\) exceeds the requirements in Lemma~\ref{lma:ue_sample_size}.

Lemma~\ref{lma:ue_sample_size} implies when \(T \ge \max_{k\le L-1} \exp(49m_k^2 / \mu_k^2)\),
the \(\frac{49m_k^2}{\mu_k^2}\log T\) times of UEs of arm \(k\) can assure that
the \texttt{OrchExplore} algorithm
learns the correct \(m_k\) with high confidence.
So,
the total cost of PUEs in learning these top \(L-1\) optimal arms' capacities is
upper bounded by \[\sum_{k=1}^{L-1}\frac{49w_km_k^2\log (T)}{\mu_k^2} + \frac{2(L-1)}{T}\times NT,\]
where \(w_k\coloneqq f(\bm{a}^*) - m_k\mu_k + \mu_1\)  is the highest cost of one round of PUE for arm \(k\) plus
\(\mu_1\) --- the highest cost of one possible deliberate exploration in a PIE round just preceding this PUE round (see the end of step 2).

With a similar procedure and Theorem~\ref{cor:m_sample_complexity}'s second part,
we can also show that, when \(T \ge \exp(49m_{L}^2/(m_L - \bar{m}_L +1)^2\mu_L^2)\),
the cost of validating that arm \(L\)'s capacity lower confidence bound \(m_{k,t}^l\) is no less
than \(\bar{m}_L\) is upper bounded by \[
    \frac{49w_Lm_L^2\log (T)}{\left(m_L - \bar{m}_L +1\right)^2\mu_L^2} + \frac{2}{T}\times NT.
\]

\textbf{Sum up previous three step's upper bounds. }

Finally, the regret of the \texttt{OrchExplore} algorithm is upper bounded as follows:
\[
    \begin{split}
        \ERT \le&  N\E\left[\abs{{\mathcal{A}}\cup {\mathcal{B}}}\right] + 2NL + \sum_{k=1}^L\frac{49w_km_k^2\log (T)}{\mu_k^2} + \frac{49w_Lm_L^2\log (T)}{\left(m_L - \bar{m}_L +1\right)^2\mu_L^2}
        + \sum_{k>L} (\mu_L - \mu_k)\E[\abs{G_k}]  \\
        \le & 13K^2 N^2 (4+\varepsilon^{-2}) + \sum_{k=1}^L\frac{49w_km_k^2\log (T)}{\mu_k^2} + \frac{49w_Lm_L^2\log (T)}{\left(m_L - \bar{m}_L +1\right)^2\mu_L^2}
        +\!\! \!\! \sum_{k=L+1}^K\!\! \!\!  \frac{(\mu_L -\mu_k)(\log T + 4 \log(\log T))}{\kl(\mu_k+\varepsilon,\mu_L-\varepsilon)}. \\
    \end{split}
\]

This finite time regret upper bound immediately leads to the following asymptotical form:\[
    \limsup_{T\to\infty}\frac{\ERT}{\log T} \le \sum_{k=L+1}^K \frac{\Delta_{L,k}}{\KL(\nu_k,\nu_L)} + \sum_{k=1}^{L-1} \frac{49w_km_k^2}{\mu_k^2} + \frac{49w_Lm_L^2}{\left(m_L - \bar{m}_L +1\right)^2\mu_L^2}.
\]

\section{The \texttt{MP-SE-SA} Algorithm}\label{sec:se_algorithm}

In this section, we first present the high level idea of our \texttt{MP-SE-SA} algorithm.
Then, we explain the successive elimination (SE) framework
and provide detailed description of \texttt{MP-SE-SA}.

\subsection{Design Overview}

Besides the exploration-exploitation trade-off,
the main challenge of the \texttt{MP-MAB-SA} problem is its
two coupled learning tasks:
(1) learning each arm's per load reward mean, (2) learning each arm's reward capacity.

One typical approach is to deal with \textit{these coupled learning tasks as a whole}, e.g., assign plays according to the UCB indexes of the capacities and reward means.
However, we note that opportunistic estimating the capacity $m_k$ (via UCB) cannot easily balance exploitation and exploration because ${m}_{k,t}$ is \emph{not} estimated as the mean of a distribution while the reward mean $\hat{\mu}_k$ does.
An alternative is to
\textit{separate the two coupled learning tasks
    as independent ones},
for example,
one first individually and unitedly explores
all $K$ arms to estimate their capacities,
and then adapts UCB
to the \texttt{MP-MAB-SA} with known capacity setting to update per load reward mean estimates.
We name this two-phase strategy as \texttt{ETC-UCB}.
This is a simple, yet inefficient, algorithm.
Because when the number of arms $K$ is much greater than the number of plays $N$,
there would be a great cost
in learning the $K-N$ suboptimal arm's reward capacities which turns out to be unnecessary.
We present the algorithm's detail and regret upper bound analysis in Appendix~\ref{app:etc_ucb}.

%

A better approach should \emph{partially} separate (decouple) \texttt{MP-MAB-SA}'s two learning tasks,
but also utilize their relations to improve the efficiency,
which needs an approach
that is flexible enough for fine-grained level operations.
We extend successive elimination (SE)~\citep{perchet_multi-armed_2013} to achieve that.
Our algorithm design has two challenges.
First, applying SE to handle the exploration-exploitation trade-off with multiple plays
is more complicated than single play MAB.
In particular, it also needs to balance two types of explorations: individual exploration (IE) and united exploration (UE).
Second,
the number of arms $L$ that should be reserved from elimination is unknown in advance.
Specifically, it can only be determined by the reward means' rank and their capacities, both of which are unknown a priori.


\subsection{The Successive Elimination Framework}

Recall that
the optimal arm set is $[L]\coloneqq\{1,2,\ldots, L\}$
and the rest arms are {suboptimal},
where $L$ is defined in Eq.(\ref{eq:critical_top_arms})
as the number of arms pulled in the optimal action.
The main idea of our algorithm is as follows.
We initialize a candidate set $\mathcal{J}_t = [K]$.
In each exploration round,
we uniformly explore
each arm in $\mathcal{J}_t$
and then use their rewards to update estimates of reward means and capacities.
In the process, we eliminate suboptimal arms from $\mathcal{J}_t$
according to \textit{two criteria} (see below) until $\mathcal{J}_t=[L]$.
As all arms $k\in\mathcal{J}_t$ have the same rounds of IE $\tau_{k,t}$ and UE $\iota_{k,t}$,
we omit their subscript $k$ as $\tau_t$ and $\iota_t$.
Denote reward mean estimate $\hat{\mu}_{k,t}$'s descending order map as $\sigma_t(\cdot)$.

\textbf{The elimination criterion.}
The first criterion is to accurately eliminate suboptimal arms with an opportune number of explorations (i.e., avoid over explorations).
This relies on reward mean estimates $\hat{\mu}_{k,t}$ and
the following elimination condition.
For any arm $k$ in the candidate set $\mathcal{J}_t$, if its reward mean estimate $\hat{\mu}_{k,t}$ is much worse than the ${L}^{\text{th}}$ largest\footnote{The $L$ is estimated in the second criterion's Eq.(\ref{eq:update_L_with_lower_bound}).}, i.e.,
\begin{equation*}
    \hat{\mu}_{k,t} \le \hat{\mu}_{\sigma_t(L),t} - U(\tau_t, T),
\end{equation*}
we eliminate the arm $k$ from $\mathcal{J}_t$. The function $U(\tau_t, T)$ is a high confidence upper bound on the deviation of $\hat{\mu}_{k,t} - \hat{\mu}_{\sigma_t(L),t}$ from $\mu_k - \mu_{\sigma_t(L)}$, and it is expressed as
$
    U(\tau_t, T) \coloneqq 2\sqrt{ 2 \tau^{-1}_t \xoverline{\log} \left(T/\tau_t \right)},
$
where $\xoverline{\log}(x) = \max\{\log x, 1\}$.


\textbf{The over elimination avoidance criterion.}
The second criterion is to assure that the total capacity of remaining arms in the candidate set $\mathcal{J}_t$
can cover $N$ plays, i.e., avoid any over elimination.
This depends on capacity estimates ${m}_{k,t}$ and their uniform confidence interval (UCI).
Denote $\tilde{L}_t$ as the expected size of $\mathcal{J}_t$ at time $t$.
It assures that with observations up to time $t$, the total capacities of top $\tilde{L}_t$ arms in $\mathcal{J}_t$ is no less than $N$. So, we can achieve this criterion as long as the
size $\abs{\mathcal{J}_t}$ is no less than $\tilde{L}_t$.


A key element of our algorithm design is to efficiently reduce the expected size $\tilde{L}_t$.
At the beginning, we set $\tilde{L}_t=N$ since $N$ arms cover at least $N$ plays.
As the algorithm proceeds, we update capacities' lower and upper bounds via Eq.(\ref{eq:m_lower_bound}-\ref{eq:m_upper_bound})
for all arm $k$ in $\mathcal{J}_t$. We then use ${m}_{k,t}^{l}$ to update $\tilde{L}_t$,
\begin{equation}\label{eq:update_L_with_lower_bound}
    \tilde{L}_t
    =
    \min\left\{n: \sum\nolimits_{k=1}^n m_{\sigma(k),t}^{{l}} \ge N\right\}.
\end{equation}
Figure~\ref{fig:size_update} depicts the expected size $\tilde{L}_t$'s update and compares it with \texttt{ETC-UCB} (in Appendix~\ref{appsub:etc_ucb_alg}).
The improvements of the \texttt{MP-SE-SA} algorithm are two folds:
(1) it only performs united explorations on top $N$ arms after eliminating $K-N$ obviously inferior arms (see the blue shadow),
(2) it gradually reduces the expected arm size $\tilde{L}_t$ in exploration rounds, which further avoids learning exact capacities for the rest $N-L$ suboptimal arms (see the orange shadow).

\begin{figure}[htb]
    \centering
    \includegraphics[width=0.4\textwidth]{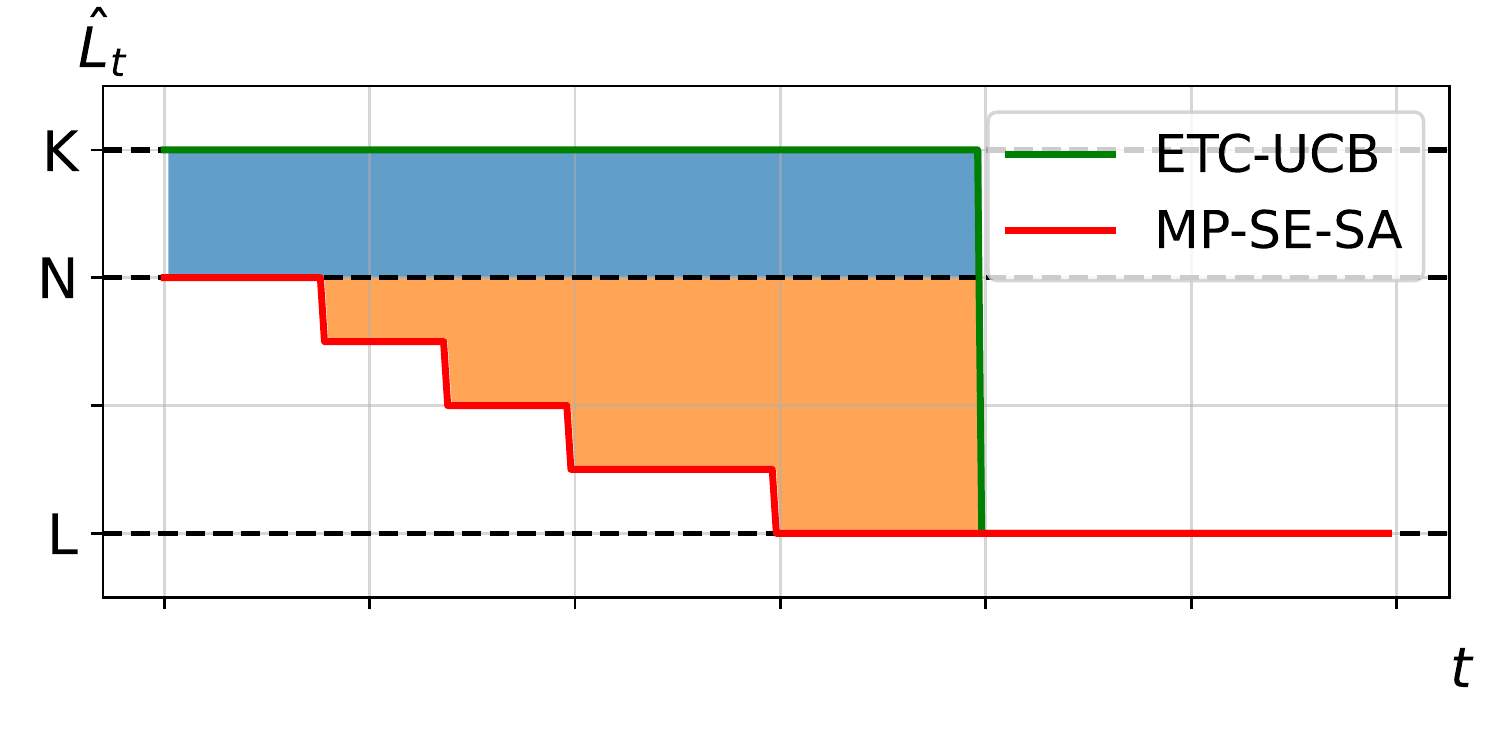}
    \caption{The update of candidate set's expected size $\tilde{L}_t$}
    \label{fig:size_update}
\end{figure}

\subsection{The \texttt{\texttt{MP-SE-SA}} Algorithm}
\setlength{\textfloatsep}{0pt}
\begin{algorithm}[htb]
    \caption{Multiple-play successive elimination with shareable arms (\texttt{MP-SE-SA})}
    \label{alg:main}
    {\bfseries Input:} $K$, $N$, $T$ and parameters $\gamma \in [1, \infty), \xi \in (0, \infty).$\\
    {\bfseries Initial:}
    \(t,\tau_t,\iota_t \gets 1\), \(\mathcal{J}_t \gets [K]\), \(\delta\gets 2\xi/T\), \(\hat{\bm{\mu}}_t\), \(\hat{\bm{\nu}}_t \gets\bm{0}\), \(\bm{m}_t^l \gets \{1,\ldots,1\}\), \(\bm{m}_t^u \gets \{N,\ldots,N\}\), \(\tilde{L}_t\gets N.\)
    \begin{algorithmic}[1]
        \WHILE{$t\le T$}
        \STATE {Update the descending ordering $\sigma_t(\cdot)$ such that $\hat{\mu}_{\sigma_t(k),t}$ is the $k$th largest in $\{\hat{\mu}_{k,t}, k \in \mathcal{J}_t\}$.}
        \STATE{Update the expected set size $\tilde{L}_t$ by Eq.(\ref{eq:update_L_with_lower_bound}). \label{alg:line:update_L}}
        \IF{$\tilde{L}_t< \abs{\mathcal{J}_t}$}\label{alg:line:L<S}
        \STATE{\textsc{Elimination}($\mathcal{J}_t, \hat{\bm{\mu}}_t, \tilde{L}_t, \sigma_t(\cdot),\gamma,T$).}\label{alg:line:eliminate}

        \STATE{\scshape Individual Exploration($\mathcal{J}_t, \hat{\bm{\mu}}_t, \tau_t, t$).}\label{alg:line:ie}
        \ELSIF{$\tilde{L}_t= \abs{\mathcal{J}_t}$}\label{alg:line:L=S}
        \STATE{\scshape Exploitation($\mathcal{J}_t,\bm{m}_t^l,\hat{\bm{\mu}}_t,\tilde{L}_t,\sigma_t(\cdot),\tau_t, t$).}\label{alg:line:explore}

        \STATE{$\mathcal{J}_t'\gets\{k\in\mathcal{J}_t:{m}_{k,t}^l\neq{m}_{k,t}^u\}.$}\label{alg:line:S_prime}
        \IF{\(\mathcal{J}_t' \neq \emptyset\)}
        \STATE{\scshape United Exploration($\mathcal{J}_t', \hat{\bm{\nu}}_t, \iota_t, t$).}
        \label{alg:line:ue}
        \STATE{Update $\hat{{m}}_{k,t}^l$ and $\hat{{m}}_{k,t}^u$ by Eq.(\ref{eq:m_lower_bound})(\ref{eq:m_upper_bound}).} \label{alg:line:update_m}
        \ENDIF
        \ENDIF
        \ENDWHILE
    \end{algorithmic}
\end{algorithm}

We present \texttt{MP-SE-SA} in Algorithm~\ref{alg:main}.
The magnitude of the current candidate arm set size $\abs{\mathcal{J}_t}$ comparing to the expected size $\tilde{L}_t$ directs the \texttt{\texttt{MP-SE-SA}} algorithm.
That $\tilde{L}_t<\abs{\mathcal{J}_t}$ (Line~\ref{alg:line:L<S})
implies the candidate arm set $\mathcal{J}_t$ containing suboptimal arms.
Then, the algorithm repeatedly employs IEs to the arms in $\mathcal{J}_t$ (Line~\ref{alg:line:ie}) so as to distinguish suboptimal ones and eliminate them (Line~\ref{alg:line:eliminate}).
After eliminating the $\abs{\mathcal{J}_t}-\tilde{L}_t$ suboptimal arms, $\tilde{L}_t = \abs{\mathcal{J}_t}$ (Line~\ref{alg:line:L=S}) and the algorithm turns to exploit the current arm set (Line~\ref{alg:line:explore}).
In the scenario, for arms whose capacity have not been exactly learnt, i.e., in set $\mathcal{S}_{t}'$ at Line~\ref{alg:line:S_prime}, the algorithm employs UEs to acquire samples for estimating the full load mean $m_k\mu_k$ (Line~\ref{alg:line:ue}) and update the $\hat{{m}}_{k,t}^l$ and $\hat{{m}}_{k,t}^u$ estimates (Line~\ref{alg:line:update_m}).
Then $\tilde{L}_t$ may decrease accordingly (Line~\ref{alg:line:update_L}) and the algorithm may go back to the $\tilde{L}_t<\abs{\mathcal{J}_t}$ case.
Finally, when $\tilde{L}_t=\abs{\mathcal{J}_t}$ and the capacities of arms in $\mathcal{J}_t$ are learnt ($\mathcal{J}_t'{=}\emptyset$),
the algorithm finds the optimal arm set, i.e., $\mathcal{J}_t{=}[L]$ and, from then on, settles down on the optimal action.


To enhance the algorithm's efficiency, we add two parameters: $\gamma \ge  1$ for scaling elimination's deviation gap as $\gamma U(\tau_t, T)$ and $\xi > 0$ for tuning UCI's confidence level $1-\delta$ as $1 - 2\xi/T$.
The smaller the $\gamma$, the more aggressive in eliminating arms, while the smaller the $\xi$, the more conservative in estimating capacities.
$\gamma$ and $\xi$ can be tuned for better performance in a specific environment but simply setting both as $1$ is also valid.
In simulation (Section~\ref{sec:simulation} and Appendix~\ref{app:addition-simulation}), we set both equal to \(1\) as default.

    {\texttt{MP-SE-SA}'s four procedures are presented in Algorithm~\ref{alg:procedures}.
        The \textsc{Elimination} procedure at Line~\ref{proc:elimination} corresponds to \textbf{the elimination criterion} in the previous subsection.
        The \textsc{Individual Exploration} procedure (Line~\ref{proc:ie}) collects samples for estimating candidate arms' per load reward mean \(\mu_k\).
        It evenly divides the current candidate arm set \(\mathcal{J}_t\) to \(\ceil{\abs{\mathcal{J}_t}/N}\) subsets so that each of them contains no more than \(N\) arms (Line~\ref{proc:ie:divide}).
        In each time slot, the procedure assigns plays to individually explore arms of one subset (Line~\ref{proc:ie:ie}).
        The \textsc{United Exploration} procedure (Line~\ref{proc:ue}) collects samples for estimating the full load reward mean \(m_k\mu_k\) of candidate arms whose capacities have not been learnt, i.e., in \(\mathcal{J}_t'\).
        It assigns all \(N\) plays to pull each arm in \(\mathcal{J}_t'\) in turn (Line~\ref{proc:ue:ue}).
        The \textsc{Exploitation} procedure (Line~\ref{proc:exploitation}) assigns plays to maximize expected reward according to the estimated per load reward \(\hat{\mu}_{k,t}\) and capacities' lower confidence bounds \({m}_{k,t}^l\).}
\begin{algorithm}[htb]
    \caption{Procedures of \texttt{MP-SE-SA}}
    \label{alg:procedures}
    \begin{algorithmic}[1]
        \PROCEDURE{Elimination}{$\mathcal{J}_t, \hat{\bm{\mu}}_t, \tilde{L}_t, \sigma_t(\cdot),\gamma,T$} \label{proc:elimination}
        \FORALL{$k\in \mathcal{J}_t$ }
        \IF {$\hat{\mu}_{k,t} \le \hat{\mu}_{\sigma(\tilde{L}_t),t} - \gamma U(\tau_t, T)$}
        \STATE $\mathcal{J}_t\gets \mathcal{J}_t\setminus\{k\}.$ \label{line:elimination}
        \ENDIF
        \ENDFOR
        \ENDPROCEDURE

        \PROCEDURE{Individual Exploration}{$\mathcal{J}_t,\!\hat{\bm{\mu}}_t,\! \tau_t,\! t$} \label{proc:ie}
        \STATE {Divide $\mathcal{J}_t$ to subsets $\{\mathcal{S}_{1,t},\mathcal{S}_{2,t}, \ldots, \mathcal{S}_{{\ceil{\abs{\mathcal{J}_t}/N}},t}\}$ such that $\abs{\mathcal{S}_{i,t}} \le N$ and $\cup_i\mathcal{S}_{i,t}=\mathcal{J}_t$.} \label{proc:ie:divide}
        \FORALL{$\mathcal{S}_{i,t}$}
        \STATE {Individually assign $N$ plays to arms in $\mathcal{S}_{i,t}$ and observe their rewards $R_{k,t}$ for all $k\in\mathcal{S}_{i,t}$.} \label{proc:ie:ie}
        \STATE {$\hat{\mu}_{k,t}{\gets} \left(\hat{\mu}_{k,t}(\tau_t-1){+}R_{k,t}\right)/\tau_t$ for all $k$ in $\mathcal{S}_{i,t}.$}
        \ENDFOR
        \STATE{
            $\tau_t \gets \tau_t + 1$, $t\gets t+\ceil{\abs{\mathcal{J}_t}/N}.$
        }
        \ENDPROCEDURE

        \PROCEDURE{United Exploration}{$\mathcal{J}_t', \hat{\bm{\nu}}_t, \iota_t, t$} \label{proc:ue}
        \FORALL{$k\in\mathcal{J}_t'$}
        \STATE {Assign $N$ plays to arm $k$, observe reward $R_{k,t}$.\label{proc:ue:ue}}
        \STATE $\hat{\nu}_{k,t}\gets \left(\hat{\nu}_{k,t}(\iota_t-1) + R_{k,t}\right)/\iota_t.$
        \ENDFOR
        \STATE{$\iota_t \gets \iota_t + 1, t\gets t+\abs{\mathcal{J}_t'}.$}
        \ENDPROCEDURE

        \PROCEDURE{Exploitation}{$\mathcal{J}_t,\bm{m}_t^l,\hat{\bm{\mu}}_t,\!\tilde{L}_t,\!\sigma_t(\!\cdot\!),\tau_t, t$}\label{proc:exploitation}
        \STATE {Assign \(m_{\sigma_t(k), t}^l\) plays to arm \(\sigma_t(k)\) for \(k<\tilde{L}_t\) and \(N-\sum_{k=1}^{\tilde{L}_t-1}m_{\sigma_t(k),t}^l\) plays to arm \(\sigma_t(\tilde{L}_t)\).}
        \STATE Observe rewards $R_{\sigma_t(k), t}$ for all $k \le \tilde{L}_t.$
        \STATE {$\hat{\mu}_{\sigma_t(k),t} \gets\hat{\mu}_{\sigma_t(k),t}(\tau_t{-}1) {+} {R_{\sigma_t(k),t}}/{m_{\sigma_t(k), t}^l})/\tau_t\,\text{for}\,k {<} \tilde{L}_t$.}
        \STATE {$\hat{\mu}_{\sigma_t(\tilde{L}_t),t} \gets (\hat{\mu}_{\sigma_t(\tilde{L}_t),t}(\tau_t-1)  + {R_{\sigma_t(\tilde{L}_t),t}}/{(N-\sum_{k=1}^{\tilde{L}_t-1}m_{\sigma_t(k),t}^l)})/\tau_t$.}
        \STATE{
            $\tau_t \gets \tau_t + 1, t\gets t+1.$
        }
        \ENDPROCEDURE\label{alg:end_proc}
    \end{algorithmic}
\end{algorithm}

\section{Regret Analysis of \texttt{MP-SE-SA}}
\label{sec:se_analysis}

\subsection{Regret Result Overview}



We rigorously prove
that \texttt{MP-SE-SA} (Algorithm~\ref{alg:main}) has a logarithmic regret.
We first define several quantities in the regret bound.
We define $g_{i,j}\coloneqq {(\mu_1 - \mu_j)}/{(\mu_i - \mu_j)} = {\Delta_{1,j}}/{\Delta_{i,j}}$ for measuring \texttt{MP-MAB-SA}'s difficulty from the elimination algorithms' aspect.
Assuming that the suboptimal arm $j$ survives from eliminations,
the $g_{L,j}$ for $j>N$ represents a ratio between the cost of
mis-eliminating the best arm $1$ while keeping arm $j$
over the cost of mis-eliminating arm $L$ while keeping arm $j$.
The largest per time slot expected reward is
$
	\sum^{L-1}_{k=1} m_k \mu_k + (N-\sum_{k=1}^{L-1} m_k)\mu_L,
$
and the smallest per time reward is $\min_{k\in[K]} m_k \mu_k$,
which happens when all $N$ plays are assigned to an arm with the smallest full load reward mean.
So, the largest per time regret denoted by $h$ is
\(
h
\coloneqq
\sum^{L-1}\nolimits_{k=1} m_k \mu_k {+} (N-\sum\nolimits_{k=1}^{L-1}m_k)\mu_L - \min_{k\in[K]} m_k \mu_k.
\)
For convenience, we denote $w_k$ as the cost upper bound of one round of IE and one round of UE for arm $k$,
$  w_k\coloneqq f(\bm{a}^*) - m_k\mu_k + \mu_1.$


\begin{theorem}[Regret Upper Bound of \texttt{MP-SE-SA}]\label{thm:main_regret}
	When the horizon $T\ge \xi\max_{k\in [N]}\exp({1/(64m_k^2\mu_k^2)})$, Algorithm~\ref{alg:main}'s expected regret is upper bounded as follows,
	\begin{equation}\label{eq:main_regret}
		\begin{split}
			\ERT
			\le  \sum_{k=L+1}^K \frac{342\gamma^2 g_{L,k} m_k }{\Delta_{L,k}}\xoverline{\log}\left(\frac{T\Delta_{L,k}^2}{18\gamma^2}\right) + \frac{4(L-1)h}{\Delta_{L-1,L}^2}
			+ \sum_{k=1}^N \frac{49w_k m_k^2}{\mu_k^2}\xoverline{\log}\left( \frac{T}{\xi} \right)
			+ 2\xi Kh,
		\end{split}
	\end{equation}
	where $\gamma\ge 1,\xi \ge 0$ are two tunable parameters of Algorithm~\ref{alg:main}, $L$ is the number of arms in the optimal action, $N$ is the number of plays, and $K$ is the number of arms.
\end{theorem}

\begin{proof}[Proof Sketch of Theorem~\ref{thm:main_regret}]
	The detailed proof is in Appendix~\ref{appsub:auxillary_regret_upper_bound}-\ref{appsub:mp_se_usa_upper_bound}.
	One key idea in the proof is to \emph{virtually decouple} the suboptimal arm elimination and expected candidate size update, since their dependency invalids the separating technique for analyzing SE algorithm (Appendix~\ref{appsub:auxillary_regret_upper_bound}):
	the elimination only happens when $\tilde{L}_t < \abs{\mathcal{S}_t}$,
	and if $\tilde{L}_t$ is large, elimination may not be possible to proceed.
	When elimination cannot proceed, i.e., $\tilde{L}_t = \abs{\mathcal{S}_t} > L$,
	we consider a \emph{virtual} rearrangement of IE and UE rounds, that is, virtually move a number of IEs and UEs (from the future) to the beginning to accumulate observations in advance and thus reduce $\tilde{L}_t$ so that the elimination can proceed.
	Such rearrangement does not change the total regret.
	We apply Corollary~\ref{cor:m_sample_complexity}'s sample complexity result to bound the number of rearranged rounds, which leads to the last two terms in Eq.(\ref{eq:main_regret}). The first two terms corresponds to successively eliminating arms
	in rounds that are not rearranged.
\end{proof}

Theorem~\ref{thm:main_regret} states that the regret upper bound of Algorithm~\ref{alg:main}
has a dependency of $\log T$.
The upper bound in Eq.(\ref{eq:main_regret}) is \emph{problem dependent}
as the factor $g_{L,k}$, the capacity $m_k$, reward mean $\mu_k$, and reward gaps $\Delta_{L,k}$ all depend on the specific bandit environment.
Since these dependent parameters are in the very complex formula of the regret bound,
techniques for deriving problem independent bounds from problem dependent ones (e.g.,~\citep[Corollary 2.1]{perchet_multi-armed_2013}) are not applicable.
Deriving a {problem independent} bound for \texttt{MP-SE-SA} can be highly nontrivial.

Theorem~\ref{thm:main_regret}'s bound has the following asymptotical form.
\begin{corollary} Algorithm~\ref{alg:main}'s regret upper bound is
	\label{cor:asymptotical_bound}
	\begin{equation*}
		\ERT
		\le\O\left(\sum_{k=L+1}^K \frac{g_{L,k}m_k}{\Delta_{L,k}}\log T\right)
		+ \O\left(\sum_{k=1}^N \frac{w_k^2m_k^2}{\mu_k^2}\log T\right).
	\end{equation*}
\end{corollary}

The first term is due to the successive elimination framework.
The second term corresponds to the worst case's cost of learning top \(N\) arms' reward capacities.
We then compare both terms to the regret lower bound's two terms in Theorem~\ref{thm:regret_lower_bound}, which points potential gaps in the upper bound.
In the comparison of their first terms, the upper bound has an additional \(m_k\) factor and is tight up to a positive coefficient.
Their second terms are different in summation ranges, where the lower bound only requires to learn \(L\) optimal arms' capacity, while the upper bound needs to learn top \(N\) arms'.
This gap implies the possibility to avoid learning \(N-L\) suboptimal arms' capacity in a finer-grained algorithm, which 
is achieved by our \texttt{OrchExplore} algorithm in Section~\ref{sec:opt_algorithm}.

\subsection{Auxillary Regret Upper Bounds}\label{appsub:auxillary_regret_upper_bound}

As building blocks for analyzing \texttt{MP-SE-SA}, we first study SE in two simpler cases:
MP-MAB and \texttt{MP-MAB-SA} with known capacity (KC).
We name the former algorithm as MP-SE, the latter as \texttt{MP-SE-SA}-KC.

\subsubsection{MP-SE's Regret Upper Bound}
As MP-MAB assumes that all arm's reward capacities $m_k$ are $1$,
MP-SE is obtained by applying $\tilde{L}_t = N$
in \texttt{MP-SE-SA} (Algorithm~\ref{alg:mp_se}).

\begin{algorithm}[htb]
	\caption{Multiple-Play Successive Elimination (MP-SE)}
	\label{alg:mp_se}
	{\bfseries Input:} Arm set $[K]$, plays $N$, time horizon $T$, and parameters $\gamma\in [1,\infty)$.\\
	{\bfseries Initial:} $t,\tau_t \gets 1,\, \mathcal{S}_t\gets [K],\, \hat{\bm{\mu}}_t\gets\bm{0}\in \mathbb{R}^K.$
	\begin{algorithmic}[1]
		\WHILE{$t\le T$}
		\STATE{Sort $\{\hat{\mu}_{k,t}, k \in \mathcal{S}_t\}$ via a mapping $\sigma$, such that $\hat{\mu}_{\sigma_t(k),t}$ is the $k$th largest among them.}
		\IF[Use $N$ to replace $\tilde{L}_t$.]{$N < \abs{\mathcal{S}_t}$}
		\STATE{\textsc{Elimination}($\mathcal{S}_t, \hat{\bm{\mu}}_t, \sigma_t(\cdot),\gamma,T$).}
		\STATE{\textsc{Individual Exploration}($\mathcal{S}_t, \hat{\bm{\mu}}_t, \tau_t, t$).}
		\ELSIF{$N = \abs{\mathcal{S}_t}$}
		\STATE{\scshape Exploitation($\mathcal{S}_t,{\bm{m}}_t^l,\hat{\bm{\mu}}_t,\tilde{L}_t,\sigma_t(\cdot),\tau_t, t$).}
		\ENDIF
		\ENDWHILE
	\end{algorithmic}
\end{algorithm}


\begin{theorem}\label{thm:mp_se_upper}
	With the setting $m_k=1$ in Algorithm~\ref{alg:mp_se}, MP-SE(-SA)'s regret is upper bounded as follows,
	\begin{equation}\label{eq:mp_se_upper}
		\ERT
		\le \sum_{k=N+1}^K \frac{342\gamma^2 g_{N,k}}{\Delta_{N,k}}\xoverline{\log}\left(\frac{T\Delta_{N,k}^2}{18\gamma^2}\right) + \frac{2(N-1)h}{\Delta_{N-1,N}^2}.
	\end{equation}
	where $\gamma \ge 1$ is the algorithm's input constant parameter.
\end{theorem}

The detailed algorithm of MP-SE is in Algorithm~\ref{alg:mp_se}.
\begin{proof}[Proof of Theorem~\ref{thm:mp_se_upper}]
	We divide the proof into three steps.

	\textbf{Step 1: construct \(\{\zeta_{N,k}\}\)s as critical times of eliminating suboptimal arms. }
	With our definition of $g_{i,j}$ and specifying $i=N$ and $k \ge N+1$, we have
	\begin{equation*}
		g_{N,k} = \frac{\mu_1 - \mu_k}{\mu_N - \mu_{k}}.
	\end{equation*}
	As $\mu_N > \mu_{N+1} > \ldots > \mu_K$, we have $g_{N, N+1} > g_{N, N+2} > \ldots > g_{N, K} > 1$.

	For each suboptimal arm $k\ge N+1$, we choose a \emph{fixed} IE sample size separators $\zeta_{N,k}\in \mathbb{N}_+$ such that
	\begin{equation*}
		\Delta_{N,k} \ge \frac{3}{2}\gamma U(\zeta_{N,k}, T),
	\end{equation*}
	and denote $\zeta_{N,k}^*\in \mathbb{R}_+$ such that $\Delta_{N,k} = \frac{3}{2}\gamma U(\zeta_{N,k}^*, T)$.
	Comparing the following inequality's LHS and RHS:
	\[\frac{3}{2}\gamma U\left( \frac{18\gamma^2}{\Delta_{N,k}^2} \xoverline{\log}\left( \frac{T\Delta_{N,k}^2}{18\gamma^2}\right), T \right) = \Delta_{N,k}\sqrt{\left( \xoverline{\log}\frac{T\Delta_{N,k}^2}{18\gamma^2} - \xoverline{\log}\xoverline{\log} \frac{T\Delta_{N,k}^2}{18\gamma^2}\right) \left/  \xoverline{\log}\frac{T\Delta_{N,k}^2}{18\gamma^2} \right.} \le \Delta_{N,k} = \frac{3}{2}\gamma U(\zeta_{N,k}^*, T),\]
	where \(U(\tau, T) =2\sqrt{({2}/{\tau})\xoverline{\log}\left( {T}/{\tau} \right)}\) is decreasing with respect to \(\tau\) and \(\xoverline{\log}(x)=\max\{\log x, 1\}\),
	we have $\zeta_{N,k}^* \le 18\gamma^2 / \Delta_{N,k}^2 \times \xoverline{\log}(T\Delta_{N,k}^2/18\gamma^2)$. Then, choosing $\zeta_{N,k} = \ceil{\zeta_{N,k}^*} < \zeta_{N,k}^* + 1$ yields
	\begin{equation}\label{eq:tau_condition}
		\zeta_{N,k} \le \frac{19\gamma^2}{\Delta_{N,k}^2}\xoverline{\log}\left(\frac{T\Delta_{N,k}^2}{18\gamma^2}\right).
	\end{equation}

	As $\Delta_{N,N+1} \le \Delta_{N,N+2} \le \ldots \le \Delta_{N,K}$ and the function $U(\tau_t, T)$ is decreasing to $\tau_t$, w.o.l.g.
	we have $\zeta_{N,N+1}\ge \zeta_{N,N+2}\ge \ldots \ge \zeta_{N,K}$. For convenience, denote $\Delta_{N,K+1} = 0, \zeta_{N,K+1} = 1$.

	\textbf{Step 2. decompose the elimination process to good events and bad events. }
	To analyze \textsc{Elimination}, we separate elimination's sample space into two mutually exclusive and exhausted events: good events and bad events.

	\textbf{Good Events:} each suboptimal arms $k \in \{N+1, N+2, \ldots, K\}$ are eliminated in or before $\zeta_{N,k}$.

	The good events mean that the elimination of all suboptimal arms proceeds properly.
	The cost of good events contributes to regret is at most $\sum_{k=N+1}^K \zeta_{N,k} g_{N,k} \Delta_{N,k}$,
	where $g_{N,k}\Delta_{N,k}$ is the cost of individually exploring the suboptimal arm $k$ once.

	\textbf{Bad Events:} either some suboptimal arm $k\in\{N+1, N+2,\ldots, K\}$ are \emph{not} eliminated in or before $\zeta_{N,k}$,
	or some top arms $k\in\{1,2,\ldots, N\}$ are falsely eliminated.

	\textbf{Step 3. bound the cost of bad events. }

	\textbf{Step 3a. bound the cost of underestimating the some of top \(N-1\) arms' reward means. }
	To tackle the bad events, we first rule out the possibility that some of the top $N-1$ arms are {excessively} underestimated, that is, there exists some top arms $k < N$, whose reward empirical mean estimate $\hat{\mu}_{k,t}(s)$ is less than the $N^{th}$ arm's estimate $\hat{\mu}_{N,t}(s)$ where $s$ represents the number of observations supporting the empirical mean estimator.
	The probability of such event is in fact very small, and it can be expressed as,
	\begin{equation*}
		\begin{split}
			\mathbb{P}(\{\exists k\in\{1, 2, \ldots, N-1\}: \hat{\mu}_{k,t}(s) < \hat{\mu}_{N,t}(s)\})
			\le & (N-1)\mathbb{P}(\hat{\mu}_{N-1,t}(s) < \hat{\mu}_{N,t}(s))\\
			= & (N-1)\mathbb{P}((\hat{\mu}_{N,t}(s) {-} \mu_{N}) - (\hat{\mu}_{N-1,t}(s) {-} \mu_{N-1})> \Delta_{N-1,N})\\
			\le & (N-1) e^{-\frac{s \Delta_{N-1,N}^2}{2}},
		\end{split}
	\end{equation*}
	where the last inequality is from the Hoeffding's inequality.
	Thus, the potential cost to regret is at most
	\begin{equation*}
		\begin{split}
			\sum_{k=1}^{N-1} \Delta_{k, N} \cdot \sum_{s=1}^T (N-1)e^{-\frac{s \Delta_{N-1,N}^2}{2}}
			\le (N-1)\sum_{k=1}^{N-1}\Delta_{k,N} \cdot \int_{s=0}^{\infty} e^{-\frac{s \Delta_{N-1,N}^2}{2}}ds
			\le \frac{2(N-1)\sum_{k=1}^{N-1}\Delta_{k,N}}{\Delta_{N-1,N}^2}.
		\end{split}
	\end{equation*}
	The advantage of ruling out the possibility of excessively underestimating the top $N-1$ arms
	is to make sure that the calibrated arm for elimination (i.e. the $\sigma(N)$ one) can only be arm $k\ge N$,
	so as to make the elimination conservative.

	\textbf{Step 3b. decompose the bad events. }
	Now, we are ready to tackle the bad events. We separate the bad events into sub-periods by $\{\zeta_{N,k}\}_{k\ge N+1}$, i.e. when the candidate set $\mathcal{S}_t$'s IE sample size $\tau_t$ is in $(1, \zeta_{N,k}], (\zeta_{N,k}, \zeta_{N,K-1}],\ldots, (\zeta_{N,N+2}, \zeta_{N,N+1}]$. Specifically, we define two sequences of events for $k\in\{N+1, N+2, \ldots, K\}$:
	\begin{equation*}
		\begin{split}
			\mathcal{A}_k\coloneqq& \left\{\text{all top arms in } \{1,2,\ldots, N\}
			\text{ have not been eliminated before }\zeta_{N,k}\right\}\\
			=& \left\{\text{Arm } N \text{ has not been eliminated before }\zeta_{N,k}\right\},\\
			\mathcal{B}_k\coloneqq& \left\{\text{every arm } i\text{ in }\{k, k+1,\ldots, K\}
			\text{ has been eliminated before }\zeta_{N,k}\right\},
		\end{split}
	\end{equation*}
	where event $\mathcal{A}_k$'s equivalence holds for arm $N$ would be falsely eliminated at first among all $N$ top arms.

	Next, we construct bad events based on \(\mathcal{A}_k\) and \(\mathcal{B}_k\), and bound their probabilities respectively.
	As $\zeta_{N,N+1}\ge \zeta_{N,N+2}\ge \ldots \ge \zeta_{N,K}$, we have
	\begin{equation*}
		\begin{split}
			&\mathcal{A}_{K}\supset \mathcal{A}_{K-1} \supset \dots \supset \mathcal{A}_{N+2}\supset \mathcal{A}_{N+1},\\
			&\mathcal{B}_{K}\supset \mathcal{B}_{K-1} \supset \dots \supset \mathcal{B}_{N+2}\supset \mathcal{B}_{N+1}.\\
		\end{split}
	\end{equation*}
	Let
	$\mathcal{C}_k = \mathcal{A}_k \cap\mathcal{B}_k$
	and denote the whole bad events as
	$\mathcal{C}_{K+1}$. Then we can divide $\mathcal{C}_{K+1}$
	as
	$(\mathcal{C}_{K+1}\setminus\mathcal{C}_{K})\cup (\mathcal{C}_K\setminus\mathcal{C}_{K-1})\cup \ldots \cup (\mathcal{C}_{N+2}\setminus\mathcal{C}_{N+1})\cup\mathcal{C}_{N+1}$.
	Notice that the cost contributing to regret after $\zeta_{N,k}$ on $\mathcal{C}_{k}$ is at most $Tg_{N,k-1}\Delta_{N,k-1}$.
	Thus, the total cost contribute to regret from the bad event is
	\begin{equation*}
		T\sum_{k=N+1}^K g_{N,k}\Delta_{N,k}\mathbb{P}(\mathcal{C}_{k+1}\setminus\mathcal{C}_k).
	\end{equation*}

	Applying the relations between events $\mathcal{A}_k, \mathcal{B}_k, \mathcal{C}_k$, we have
	\begin{equation*}
		\begin{split}
			\mathcal{C}_{k+1}\setminus\mathcal{C}_k \Leftrightarrow
			\left(\left(\mathcal{A}_{k+1}\setminus\mathcal{A}_k\right)\cap \mathcal{B}_{k+1}\right)
			\cup
			\left(\left(\mathcal{B}_{k+1}\setminus \mathcal{B}_k\right) \cap \mathcal{A}_{k+1}\right),
		\end{split}
	\end{equation*}
	which leads to
	\begin{equation}\label{eq:split_C}
		\begin{split}
			&\quad \sum_{k=N+1}^K g_{N,k}\Delta_{N,k}\mathbb{P}(\mathcal{C}_{k+1}\setminus\mathcal{C}_k)
			\\
			& \le   \sum_{k=N+1}^K g_{N,k}\Delta_{N,k}\mathbb{P}\left(\left(\mathcal{A}_{k+1}\setminus\mathcal{A}_k\right)\cap \mathcal{B}_{k+1}\right)
			+ \sum_{k=N+1}^K g_{N,k}\Delta_{N,k}\mathbb{P}\left(\left(\mathcal{B}_{k+1}\setminus \mathcal{B}_k\right) \cap \mathcal{A}_{k+1}\right).
		\end{split}
	\end{equation}
	Note that $\left(\mathcal{B}_{k+1}\setminus \mathcal{B}_k\right) \cap \mathcal{A}_{k+1}$ and $\left(\mathcal{A}_{k+1}\setminus\mathcal{A}_k\right)\cap \mathcal{B}_{k+1}$ are the bad events. We will bound their probabilities respectively.

	\textbf{Step 3c. bound the second term of Eq.(\ref{eq:split_C})'s RHS. }
	Notice that the event $\left(\mathcal{B}_{k+1}\setminus \mathcal{B}_k\right) \cap \mathcal{A}_{k+1}$ implies that arm $k$ is not eliminated in or before $\zeta_{N,k}$ while all top arms are in the candidate arm set $\mathcal{S}_t$. Thus, we have
	\begin{equation*}
		\begin{split}
			&\quad\mathbb{P}\left(\left(\mathcal{B}_{k+1}\setminus \mathcal{B}_k\right) \cap \mathcal{A}_{k+1}\right)\\
			&\le   \mathbb{P}(\hat{\mu}_{k,t}(\zeta_{N,k}) > \hat{\mu}_{N,t}(\zeta_{N,k}) - \gamma U(\zeta_{N,k}, T)) \\
			&\le \mathbb{P}\left((\hat{\mu}_{k,t}(\zeta_{N,k})-\mu_k) - (\hat{\mu}_{N,t}(\zeta_{N,k}) - \mu_{N}) > \frac{1}{2}\gamma U(\zeta_{N,k}, T)\right) \\
			&\le \frac{\zeta_{N,k}}{T},
		\end{split}
	\end{equation*}
	where the second equation is from $\Delta_{N,k} \ge \frac{3}{2}\gamma U((\zeta_{N,k}), T)$ and the third is from Hoeffding's inequality and $U((\zeta_{N,k}), T)$'s formula.
	Then, the second term of Eq.(\ref{eq:split_C})'s RHS is upper bounded as follows
	\begin{equation*}
		\begin{split}
			\sum_{k=N+1}^K g_{N,k}\Delta_{N,k}\mathbb{P}\left(\left(\mathcal{B}_{k+1}\setminus \mathcal{B}_k\right) \cap \mathcal{A}_{k+1}\right)
			\le  \frac{1}{T}\sum_{k=N+1}^K \zeta_{N,k} g_{N,k}\Delta_{N,k}.
		\end{split}
	\end{equation*}

	\textbf{Step 3d. bound the first term of Eq.(\ref{eq:split_C})'s RHS. }
	Event $\left(\mathcal{A}_{k+1}\setminus\mathcal{A}_k\right)\cap \mathcal{B}_{k+1}$ implies that some top arms in $\{1,2,\ldots, N\}$ are falsely eliminated between $\zeta_{N,k+1}+1$ and $\zeta_{N,k}$ while suboptimal arms $\{k+1, k+2, \ldots, K\}$ are all properly eliminated.
	\begin{equation*}
		\begin{split}
			& \mathbb{P}\left(\left(\mathcal{A}_{k+1}\setminus\mathcal{A}_k\right)\cap \mathcal{B}_{k+1}\right) \\
			\le & \mathbb{P}(\exists (j,s), j\in\{N+1, N+2, \ldots, k\}, \zeta_{N,k+1}+1 \le s \le \zeta_{N,k}: \hat{\mu}_{N,t}(s) < \hat{\mu}_{j,t}(s) - \gamma U(s, T))\\
			\le & \sum_{j=N+1}^k \mathbb{P}(\exists \zeta_{N,k+1}+1\le s \le \zeta_{N,k}: \hat{\mu}_{i,t}(s) < \hat{\mu}_{j,t}(s) - \gamma U(s, T))\\
			\le &\sum_{j=N+1}^k \mathbb{P}(\exists \zeta_{N,k+1}+1\le s \le \zeta_{N,k}: (\hat{\mu}_{j,t}(s) - \mu_j) - (\hat{\mu}_{i,t}(s) -\mu_i) \ge \gamma U(s, T))\\
			= & \sum_{j=N+1}^k\left(\Phi(\zeta_{N,k}) - \Phi(\zeta_{N,k+1})\right),
		\end{split}
	\end{equation*}
	where we denote $\Phi(\zeta) \coloneqq \mathbb{P}(\exists s \le \zeta: (\hat{\mu}_{j,t}(s) - \mu_j) - (\hat{\mu}_{i,t}(s) -\mu_i) \ge \gamma U(s, T))$ for any $1\le i \le N < j \le k$. Next, we apply the following Lemma~\ref{lma:martingale} to bound the function $\Phi(\zeta)$.

	\begin{lemma}[{\citep[Lemma A.1]{perchet_multi-armed_2013}}]\label{lma:martingale}
		Let $Z_t$ be a martingale difference sequence with $a\le Z_t \le b$, then for every $S>0$ and every integer $T\ge 1$, \begin{equation*}
			\mathbb{P}\left(\exists t\le T: \frac{1}{t}\sum_{i=1}^t Z_i \ge \sqrt{\frac{2(b-a)^2}{t}\log \left(\frac{4}{\delta}\frac{T}{t}\right)}\right)\le \delta.
		\end{equation*}
	\end{lemma}

	Apply the formula replacement $t \gets s, T\gets \zeta, (b-a) \gets 2, \delta \gets 4\zeta_{N,k}/T$ in Lemma~\ref{lma:martingale}, we have $\Phi(\zeta) \le 4\zeta/T$ and thus
	\begin{equation*}
		\mathbb{P}\left(\left(\mathcal{A}_{k+1}\setminus\mathcal{A}_k\right)\cap \mathcal{B}_{k+1}\right)
		\le \frac{4}{T}\sum_{j=N+1}^k (\zeta_{N,k} - \zeta_{N,k+1}).
	\end{equation*}
	Then, the first term of Eq.(\ref{eq:split_C})'s RHS is bounded as follows
	\begin{equation*}
		\begin{split}
			&\quad\sum_{k=N+1}^K g_{N,k}\Delta_{N,k}\mathbb{P}\left(\left(\mathcal{A}_{k+1}\setminus\mathcal{A}_k\right)\cap \mathcal{B}_{k+1}\right)\\
			&\le \frac{4}{T}\sum_{k=N+1}^{K}\sum_{j=N+1}^k g_{N,k}\Delta_{N,k}(\zeta_{N,k} - \zeta_{N,k+1}) \\
			& = \frac{4}{T} \Bigg( \sum_{j=N+1}^{K}\sum_{k=j}^K \zeta_{N,k+1}(g_{N,k+1}\Delta_{N,k+1}-g_{N,k}\Delta_{N,k}) +     \sum_{j=N+1}^{K} \zeta_{N,j} g_{N,j}\Delta_{N,j}\Bigg)\\
			&=  \frac{4}{T} \left( \sum_{j=N+1}^{K}\sum_{k=j}^K \zeta_{N,k+1}(\Delta_{N,k+1}-\Delta_{N,k}) +
			\sum_{j=N+1}^{K} \zeta_{N,j} g_{N,j}\Delta_{N,j}\right).
		\end{split}
	\end{equation*}

	Summing up all above individual contributions to the expected regret, we have
	\begin{equation}\label{eq:elimination_upper_bound}
		\begin{split}
			\ERT \le 4 \sum_{j=N+1}^{K}\sum_{k=j}^K \zeta_{N,k+1}(\Delta_{N,k+1}-\Delta_{N,k}) + 6\sum_{k=N+1}^K \zeta_{N,k} g_{N,k}\Delta_{N,k} + \frac{2(N-1)\sum_{k=1}^{N-1}\Delta_{k,N}}{\Delta_{N-1,N}^2}.
		\end{split}
	\end{equation}

	Then, we substitute Eq.(\ref{eq:tau_condition}) into the Eq.(\ref{eq:elimination_upper_bound})'s first term inner summation $\sum_{k=j}^K \zeta_{N,k+1}(\Delta_{N,k+1}-\Delta_{N,k})$ and get
	\begin{equation*}
		\begin{split}
			\sum_{k=j}^K \zeta_{N,k+1}(\Delta_{N,k+1}-\Delta_{N,k})
			\le&  \sum_{k=j}^K \frac{19\gamma^2}{\Delta_{N,k+1}^2}\xoverline{\log}\left(\frac{T\Delta_{N,k+1}^2}{18\gamma^2}\right)(\Delta_{N,k+1}-\Delta_{N,k})\\
			= & 19\gamma^2 \sum_{k=j}^K \xoverline{\log}\left(\frac{T\Delta_{N,k+1}^2}{18\gamma^2}\right)\frac{\Delta_{N,k+1}-\Delta_{N,k}}{\Delta_{N,k+1}^2}\\
			\le & 19\gamma^2 \int_{\Delta_{N,j}}^{\Delta_{N,k}} \xoverline{\log}\left(\frac{Tx^2}{18\gamma^2}\right)\frac{1}{x^2}dx\\
			\le & \frac{19\gamma^2}{\Delta_{N,j}}\left(\xoverline{\log}\left(\frac{T\Delta_{N,j}^2}{18\gamma^2}\right)+2\right),
		\end{split}
	\end{equation*}
	and then substitute Eq.(\ref{eq:tau_condition}) into the Eq.(\ref{eq:elimination_upper_bound})'s second term as follows
	\begin{equation*}
		6\sum_{k=N+1}^K \zeta_{N,k} g_{N,k}\Delta_{N,k} \le
		\sum_{k=N+1}^K \frac{114\gamma^2 g_{k}}{\Delta_{N,k}}\xoverline{\log}\left(\frac{T\Delta_{N,k}^2}{18\gamma^2}\right).
	\end{equation*}

	Then, $\ERT$ is upper bounded as
	\begin{equation*}
		\begin{split}
			\ERT
			\le & \sum_{j=N+1}^{K} \frac{76\gamma^2}{\Delta_{N,j}}\left(\xoverline{\log}\left(\frac{T\Delta_{N,j}^2}{18\gamma^2}\right)+2\right) + \frac{2(N-1)\sum_{k=1}^{N-1}\Delta_{k,N}}{\Delta_{N-1,N}^2} + \sum_{k=N+1}^K \frac{114\gamma^2 g_{N,k}}{\Delta_{N,k}}\xoverline{\log}\left(\frac{T\Delta_{N,k}^2}{18\gamma^2}\right) \\
			\le& \sum_{k=N+1}^K \frac{342\gamma^2 g_{N,k}}{\Delta_{N,k}}\xoverline{\log}\left(\frac{T\Delta_{N,k}^2}{18\gamma^2}\right) + \frac{2(N-1)\sum_{k=1}^{N-1}\Delta_{k,N}}{\Delta_{N-1,N}^2}\\
			\le& \sum_{k=N+1}^K \frac{342\gamma^2 g_{N,k}}{\Delta_{N,k}}\xoverline{\log}\left(\frac{T\Delta_{N,k}^2}{18\gamma^2}\right) + \frac{2(N-1)h}{\Delta_{N-1,N}^2}.
		\end{split}
	\end{equation*}
\end{proof}

\subsubsection{\texttt{MP-SE-SA}-KC's Regret Upper Bound}
With known capacity (KC), one still needs to estimate $\tilde{L}_t$
as per capacity reward means are unknown.
\texttt{MP-SE-SA}-KC is obtained by replacing ${\bm{m}}_{k,t}^l,{\bm{m}}_{k,t}^u$
with exact ${\bm{m}}$ for updating $\tilde{L}_t$
(see Line~\ref{alg:line:kc_update_L} in Algorithm~\ref{alg:mp_se_sa_kc}).


\begin{theorem}\label{thm:mp_se_sa_known_upper}
	With known capacity $m_k \ge 1$ in Algorithm~\ref{alg:mp_se_sa_kc},
	the \texttt{MP-SE-SA}-KC has the regret upper bound,
	\begin{equation}\label{eq:mp_se_kas_upper}
		\begin{split}
			\ERT
			\le \sum_{k=L+1}^K \frac{342\gamma^2 m_k g_{L,k}}{\Delta_{L,k}}\xoverline{\log}\left(\frac{T\Delta_{L,k}^2}{18\gamma^2}\right)
			+  \frac{4(L-1)h}{\Delta_{L-1,L}^2},
		\end{split}
	\end{equation}
	where $L$ is the smallest number of top arms that can cover all $N$ plays in Eq.(\ref{eq:critical_top_arms}), $h$ is the highest instantaneous regret per time slot.
\end{theorem}




Notice that in Theorem~\ref{thm:mp_se_upper} for MP-SE and Theorem~\ref{thm:mp_se_sa_known_upper} for \texttt{MP-SE-SA}-KC, each arm's reward capacity $m_k$ is known. Thus, exploration rounds of MP-SE and \texttt{MP-SE-SA}-KC only involve individual exploration (IE).
Therefore, united exploration (UE) is only required in the \texttt{MP-SE-SA} without knowing the value of the reward capacity (Theorem~\ref{thm:main_regret}).
The detailed algorithm of \texttt{MP-SE-SA}-KC is in Algorithm~\ref{alg:mp_se_sa_kc}.

\begin{algorithm}[htb]
	\caption{\texttt{MP-SE-SA}-KC}
	\label{alg:mp_se_sa_kc}
	{\bfseries Input:} Arm set $[K]$, plays $N$, time horizon $T$, sharing capacity $\bm{m}$ and parameters $\gamma\in [1,\infty)$.\\
	{\bfseries Initial:} $t,\tau_t\gets 1,\, \mathcal{S}_t\gets [K],\, \hat{\bm{\mu}}_t\gets\bm{0}\in \mathbb{R}^K,\,  \tilde{L}_t\gets N.$
	\begin{algorithmic}[1]
		\WHILE{$t\le T$}
		\STATE{Sort $\{\hat{\mu}_{k,t}, k \in \mathcal{S}_t\}$ via a mapping $\sigma$, such that $\hat{\mu}_{\sigma_t(k),t}$ is the $k$th largest among them.}
		\STATE{$\tilde{L}_t\gets \argmin_n\left\{n:\sum_{k=1}^n {m}_{\sigma_t(k),t}\ge N\right\}.$}\label{alg:line:kc_update_L}

		\IF{$\tilde{L}_t < \abs{\mathcal{S}_t}$}
		\STATE{\textsc{Elimination}($\mathcal{S}_t, \hat{\bm{\mu}}_t, \sigma_t(\cdot),\gamma,T$).}
		\STATE{\scshape Individual Exploration($\mathcal{S}_t, \hat{\bm{\mu}}_t, \tau_t, t$).}
		\ELSIF{$\tilde{L}_t = \abs{\mathcal{S}_t}$}
		\STATE{\scshape Exploitation($\mathcal{S}_t,{\bm{m}}_t^l,\hat{\bm{\mu}}_t,\tilde{L}_t,\sigma_t(\cdot),\tau_t, t$).}
		\ENDIF
		\ENDWHILE
	\end{algorithmic}
\end{algorithm}

\begin{proof}[Proof of Theorem~\ref{thm:mp_se_sa_known_upper}]
	The elimination part of \texttt{MP-SE-SA}-KC is different from MP-SE in two aspects,
	\begin{enumerate}
		\item \texttt{MP-SE-SA}-KC only keeps top $L$ arms, so all $N$ symbols in MP-SE should be replaced with $L$.
		\item The cost contributing to regret after $\zeta_{L,k+1}$ on the event $\mathcal{C}_{k+1}$ is now $m_k$ time the cost of MP-SE, i.e., $m_{k}\cdot g_{L,k}\Delta_{L,k}T$.
	\end{enumerate}
	Thus, the cost of elimination is
	\begin{equation*}
		\sum_{k=L+1}^K \frac{342\gamma^2 m_k g_{L,k}}{\Delta_{L,k}}\xoverline{\log}\left(\frac{T\Delta_{L,k}^2}{18\gamma^2}\right) + \frac{2(L-1)\sum_{k=1}^{L-1}m_k\Delta_{k,L}}{\Delta_{L-1,L}^2}.
	\end{equation*}
	where $m_k$ is the additional factor in the first term, which corresponds to the second different aspect.

	Notice that when top $L$ arms' total reward capacities $\sum_{k\le L} m_k$ is \emph{strict} greater than $N$ and $\Delta_{L-1,L} > 0$, the number of plays assigned to arm $L$ in the optimal action is less then $m_L$ (i.e., not fully utilize the $L^{\text{th}}$ arm's capacity). Thus, we need to differentiate the $L^{\text{th}}$ arm. Or otherwise, the failure of not fully utilizing the other top arms would introduce additional costs.
	For any fixed sample size $s$, the failure probability is
	\begin{equation*}
		\begin{split}
			\mathbb{P}(\{\exists k\in\{1, 2, \ldots, L-1\}: \hat{\mu}_{k,t}(s) < \hat{\mu}_{L,t}(s)\})
			\le & (L-1)\mathbb{P}(\hat{\mu}_{L-1,t}(s) < \hat{\mu}_{L,t}(s)\}) \\
			= & (L-1)\mathbb{P}((\hat{\mu}_{L,t} - \mu_{L}) - (\hat{\mu}_{L-1,t}-\mu_{L-1})> \Delta_{L-1,L})\\
			\le & (L-1)e^{-\frac{\tau \Delta_{L-1,L}^2}{2}}.
		\end{split}
	\end{equation*}
	Then, the total cost of such event is at most
	\begin{equation*}
		\begin{split}
			\sum_{\tau = 1}^T (L-1)e^{-\frac{\tau \Delta_{L,L-1}^2}{2}}\cdot m_{L-1}\Delta_{L-1, L}
			\le & (L-1)m_{L-1}\Delta_{L-1, L}\int_{\tau = 1}^{\infty} e^{-\frac{\tau \Delta_{L,L-1}^2}{2}}d\tau \\
			\le &(L-1)m_{L-1}\Delta_{L-1, L} \cdot \frac{2}{\Delta_{L-1, L}^2}\\
			= & \frac{2(L-1)m_{L-1}}{\Delta_{L-1,L}}.
		\end{split}
	\end{equation*}

	Thus the regret of \texttt{MP-SE-SA}-KC is upper bounded as
	\begin{equation*}
		\begin{split}
			\ERT
			\le & \sum_{k=L+1}^K \frac{342\gamma^2 m_k g_{L,k}}{\Delta_{L,k}}\xoverline{\log}\left(\frac{T\Delta_{L,k}^2}{18\gamma^2}\right) + \frac{2(L-1)\left(\sum_{k=1}^{L-1}m_k\Delta_{k,L} + m_{L-1}\Delta_{L-1,L}\right)}{\Delta_{L-1,L}^2}.
			\\
			\le & \sum_{k=L+1}^K \frac{342\gamma^2 m_k g_{L,k}}{\Delta_{L,k}}\xoverline{\log}\left(\frac{T\Delta_{L,k}^2}{18\gamma^2}\right)
			+  \frac{4(L-1)h}{\Delta_{L-1,L}^2}.
		\end{split}
	\end{equation*}

\end{proof}

\subsection{\texttt{MP-SE-SA} Regret Upper Bound}\label{appsub:mp_se_usa_upper_bound}
As Algorithm~\ref{alg:main} shows, in \texttt{MP-SE-SA}, the elimination of a suboptimal arm not only relies on the elimination criterion but also the over elimination avoidance criterion, i.e., $\tilde{L}_t \le \abs{\mathcal{S}_t}$.
Thus, one critical caveat in analyzing the algorithm is that even when we are able to discern a suboptimal arm via the elimination condition, we may not be able to execute the elimination.
Because the over elimination avoidance criterion prevents this to occur, i.e., the estimate of expected candidate set size $\tilde{L}_t$ may be inaccurate, i.e., $\tilde{L}_t = \abs{\mathcal{S}_t} > L$.
This observation implies that the proof plot in Theorem~\ref{thm:mp_se_upper} should be further refined in Theorem~\ref{thm:main_regret}.
\begin{proof}[Proof of Theorem~\ref{thm:main_regret}]
	We first assume that all suboptimal arm's eliminations happen smoothly, that is, whenever we can discern a suboptimal arm via the elimination condition, we can eliminate it and the over elimination avoidance criterion does not prevent us, i.e., $\tilde{L}_t <\abs{\mathcal{S}_t}$.

	The condition $T > \xi\max_{k\in [N]}\exp({1/(64m_k^2\mu_k^2)})$ corresponds to sample complexity's maximal operation in Corollary~\ref{cor:m_sample_complexity}, that is, \(\frac{49m_k^2}{\mu_k^2}\log\frac{T}{\xi} > \frac{1}{4\mu_k^4}\) for all arms $k\le N$.

	Then, the whole learning procedure is the same as \texttt{MP-SE-SA}-KC, except that we need to assign some time slots to perform UE for estimating reward capacity (specifically, those $\hat{\nu}_{k,t}$).  Notice that the number of UE rounds $\iota_t$ is less than the number of IE rounds $\tau_t$ (including the exploitation rounds). From Corollary~\ref{cor:m_sample_complexity}'s sample complexity result, for each arm, $\frac{49m_k^2}{\mu_k^2}\log\frac{T}{\xi}$ rounds of UE and IE would provide an accurate estimate of reward capacity with probability of at least $1-2\xi/T$.

	Thus, the additional cost under this assumption is at most    \begin{equation*}
		\sum_{k=1}^N  w_k \frac{49m_k^2}{\mu_k^2}\log\frac{T}{\xi}  + \frac{2K\xi}{T}\cdot hT \le
		\sum_{k=1}^N \frac{49m_k^2 w_k}{\mu_k^2}\log\frac{T}{\xi} + 2\xi Kh,
	\end{equation*}
	where $w_k \coloneqq  f(\bm{a}^*) - m_k\mu_k + \mu_1$ stands for the highest compound cost of applying IE and UE for an arm $k \le N$.

	Next, we relax the assumption that all eliminations happen smoothly. In that case, when the estimate of reward mean is accurate enough for eliminating some suboptimal arms, the over elimination avoidance criterion may put off the elimination until the expected candidate set $\tilde{L}_t$ is less than $\abs{\mathcal{S}_t}$, i.e., $\tilde{L}_t < \abs{\mathcal{S}_t}$.

	The additional periods caused by the elimination's impediment is for accumulating IE and UE observations to improve the estimate accuracy of reward capacity. Notice that Corollary~\ref{cor:m_sample_complexity} shows that at most $\frac{49m_k^2}{\mu_k^2}\log\frac{2T}{\xi}$ rounds of UE and IE would provide a good estimate of reward capacity. Thus, the total cost of such put-offs is still less than $\sum_{k=1}^N \frac{49m_k^2 h_k}{\mu_k^2}\log\frac{T}{\xi} + 2\xi Kh_k$.

	To make the separators proof technique of Theorem~\ref{thm:mp_se_upper} applicable, we consider a \textit{virtual rearrangement} of those additional time slots caused by those delayed elimination.
	That is, we virtually replace them to the start of Algorithm~\ref{alg:main} to accumulate observations in advance.
	After those rearrangement explorations (say totally $Y$ time slots), all elimination can proceed smoothly. The only difference from its known capacity counterpart (\texttt{MP-SE-SA}-KC) is that these time indexes now become $Y+t$. The corresponding regret after those rearrangement rounds is upper bounded as Theorem~\ref{thm:mp_se_sa_known_upper}'s Eq.(\ref{eq:mp_se_kas_upper}).


	Finally, summing up the $Y$ time slots of shifted explorations and the remaining rounds
	concludes the regret upper bound as follows.
	\begin{equation*}
		\begin{split}
			\ERT \le \sum_{k=L+1}^K \frac{342\gamma^2 m_k g_{L,k}}{\Delta_{L,k}}\xoverline{\log}\left(\frac{T\Delta_{L,k}^2}{18\gamma^2}\right)+ \sum_{k=1}^N \frac{49m_k^2 w_k}{\mu_k^2}\log\frac{T}{\xi} + 2\xi Kh
			+ \frac{4(L-1)h}{\Delta_{L-1,L}^2}.
		\end{split}
	\end{equation*}
\end{proof}

\section{\texttt{ETC-UCB} Algorithm and Its Regret Upper Bound}\label{app:etc_ucb}

\subsection{\texttt{ETC-UCB} Algorithm}\label{appsub:etc_ucb_alg}
We present the \texttt{ETC-UCB} algorithm in Algorithm~\ref{alg:etc_ucb}.
Its procedures are presented in Algorithm~\ref{alg:procedures}.
The ETC (explore-then-commit) phase is from Line~\ref{line:exploration_start} to Line~\ref{line:exploration_end} and
the UCB (upper confidence bound) phase is from Line~\ref{line:ucb_start} to Line~\ref{line:ucb_end}.
In each exploration round of the ETC phase, the algorithm implements IE (individual exploration) and UE (united exploration) once for each arm $k$ in the whole arm set $[K]$.
In its UCB rounds, the algorithm chooses actions according to each arm's UCB index.

\begin{algorithm}[htb]
	\caption{\texttt{ETC-UCB}}
	\label{alg:etc_ucb}
	{\bfseries Input:} Arm set $[K]$, plays $N$, time horizon $T$, and parameter
	$\xi\in (0,\infty).$\\
	{\bfseries Initialization:} $t, \tau_t , \iota_t \gets 1, \hat{\bm{\mu}}_t, \hat{\bm{\nu}}_t\gets \bm{0}\in \mathbb{R}^K, \bm{m}_t^l,{\bm{m}}_t,\bm{n}_t\gets \bm{1}\in\mathbb{N}^K, \bm{m}_t^u\gets(N,\dots, N).$
	\begin{algorithmic}[1]
		\WHILE[{ETC phase}]{$\mathcal{S}\neq\emptyset$\label{line:exploration_start}}
		\STATE{$\mathcal{S}\gets\{k\in[K]:m_{k,t}^l\neq m_{k,t}^u\}.$}
		\STATE{\textsc{Individual Exploration}($\mathcal{S}, \hat{\bm{\mu}}_t, \tau_t, t$).}
		\STATE{\textsc{United Exploration}($\mathcal{S}, \hat{\bm{\nu}}, {\bm{m}}_{t}^l, {\bm{m}}_{t}^u,\xi, T,\iota_t, t$).}
		\ENDWHILE \label{line:exploration_end}
		\STATE{$m_{k,t} \gets \round{\hat{\nu}_{k,t} / \hat{\mu}_{k,t}}$ for all $k$ in $[K]$.}
		\STATE{$n_{k,t}\gets \tau_t$ for all $k$ in $[K]$.}
		\WHILE[UCB phase]{$t\le T$\label{line:ucb_start}}
		\STATE{$\text{UCB}_{k,t} \gets \hat{\mu}_{k,t} + \sqrt{\frac{2\log t}{n_{k,t}}}$ for all $k$ in $[K]$.
		}
		\STATE{Sort $\{\text{UCB}_{k,t}, k \in \mathcal{S}\}$ via a descending ordering $\sigma$, such that $\text{UCB}_{\sigma(k),t}$ is the $k$th largest.}
		\STATE{$\hat{L}_t\gets \argmin_n\left\{n:\sum_{k=1}^n m_{\sigma(k),t}\ge N\right\}.$}
		\STATE{\textsc{Exploitation}($\mathcal{S},{\bm{m}}_t^l,\text{\bf UCB}_{t},\hat{L}_t,\sigma(\cdot), \bm{n}_t, t$).}
		\ENDWHILE\label{line:ucb_end}
	\end{algorithmic}
\end{algorithm}

\subsection{Regret Upper Bound of \texttt{ETC-UCB}}\label{appsub:etc_ucb}
\begin{theorem}\label{thm:etc_ucb_bound}
	The \emph{\texttt{ETC-UCB}} in Algorithm~\ref{alg:etc_ucb} has the regret upper bound,
	\begin{equation}\label{eq:etc_ucb_upper}
		\begin{split}
			\ERT \le \sum_{k=L+1}^K\frac{8\Delta_{1,k} m_k\log T}{\Delta_{L,k}^2} + \frac{8\Delta_{1,L} m_L\log T}{\Delta_{L-1,L}^2}  + \sum_{k=1}^K \frac{49m_k^2 w_k}{\mu_k^2}\log{T}
			+ 6KN.
		\end{split}
	\end{equation}
	where $L$ is the smallest number of top arms that can cover all $N$ plays in Eq.(\ref{eq:critical_top_arms}).
\end{theorem}
\begin{proof}[Proof of Theorem~\ref{thm:etc_ucb_bound}]


	The regret analysis contains two parts of the ETC phase and the UCB phase.
	The ETC phase (in Line~\ref{line:exploration_start}-\ref{line:exploration_end}) repeatedly applies IE and UE to accumulate observations so as to accurately estimate reward capacities. We apply the sample complexity result in Theorem~\ref{cor:m_sample_complexity} to bound the number of IEs and UEs (let \(\delta\gets 2/T\)). Thus the total cost in the ETC phase is 
	upper bounded as follows
	\begin{equation}\label{eq:etc_ucb_etcphase}
		\begin{split}
			\sum_{k=1}^K  w_k \frac{49m_k^2}{\mu_k^2}\log{T}  + \frac{2K}{T} NT
			\le
			\sum_{k=1}^K \frac{49m_k^2 w_k}{\mu_k^2}\log{T} + 2 KN.
		\end{split}
	\end{equation}

	Next,  with known capacities, we prove the regret cost in the UCB phase (in Line~\ref{line:ucb_start}-\ref{line:ucb_end}). 
	We first assume that for all arm \(k\) and time slots \(t\) in the UCB phase, 
	their ``per-load'' reward mean \({\mu}_{k}\) is \textit{always} inside the UCB index's corresponding
	the confidence interval \((\hat{\mu}_{k,t} -\sqrt{2\log t/n_{k,t}}, \hat{\mu}_{k,t} + \sqrt{2\log t/ n_{k,t}})\).
	With this assumption, we show that the number of times that a suboptimal arm \(k\) is played is at most \(\frac{8\log T}{\Delta_{L,k}^2}\). Because when \(n_{k,t} > \frac{8\log T}{\Delta_{L,k}^2}\), we have 
	\[
		\sqrt{\frac {2\log t}{n_{k,t}}} < \frac{\Delta_{L,k}}{2}. 
	\]
	
	If this suboptimal arm \(k\) is pulled when \(n_{k,t} > \frac{8\log T}{\Delta_{L,k}^2}\), it UCB index should be greater than the least favored arm \(L\)'s UCB index.
	However, this is impossible: \[
		\hat{\mu}_{k,t} + \sqrt{\frac{2\log t}{n_{k,t}}} \le \mu_k + 2\sqrt{\frac{2\log t}{n_{k,t}}} \le \mu_k + \Delta_{L,k}
		\le \mu_L \le \hat{\mu}_{L,t} + \sqrt{\frac{2\log t}{n_{L,t}}}.
	\]

	So, for these suboptimal arms, the total cost is upper bounded by \[
		\sum_{k=L+1}^K m_k\Delta_{1,k}\frac{8\log T}{\Delta_{L,k}^2} = \sum_{k=L+1}^K\frac{8\Delta_{1,k} m_k\log T}{\Delta_{L,k}^2},
	\]
	where the per play cost \(m_k\Delta_{1,k}\) considers the worst case that the best arm \(1\) is missed.

	Especially, when the number of times of pulling the least favored arm \(L\) is greater than \(\frac{8\log T}{\Delta_{L-1,L}^2}\),
	the algorithm (if chooses arm \(L\)) can identify it as the least favored arm and only assign \(\bar{m}_L\) number of plays to it. So, the additional cost caused by arm \(L\) is upper bounded as \(
		\frac{8\Delta_{1,L} m_L\log T}{\Delta_{L-1,L}^2}.
	\)

	We then prove that the expected total number of times that an arm's ``per-load'' reward mean is outside the confidence interval is finite:
	\[
		\begin{split}
			\E\left\{ \sum_{k\in [K]}\sum_{t\le T} \1{\mu_k \not\in \left(\hat{\mu}_{k,t} -\sqrt{\frac{2\log t}{n_{k,t}}}, \hat{\mu}_{k,t} + \sqrt{\frac{2\log t}{n_{k,t}}}\right)} \right\}
			\le 2K \sum_{t\le T} t^{-2} \le 4K,
		\end{split}
	\]
	where the first inequality holds for applying the Hoeffding's inequality as follows \[
		\P\left( \mu_k \not\in \left(\hat{\mu}_{k,t} -\sqrt{\frac{2\log t}{n_{k,t}}}, \hat{\mu}_{k,t} + \sqrt{\frac{2\log t}{n_{k,t}}}\right)\right) 
		\le  2 t^{-2}.
	\]

	We sum up the above costs in the UCB phase as follows \begin{equation}\label{eq:etc_ucb_ucbphase}
		\sum_{k=L+1}^K\frac{8\Delta_{1,k} m_k\log T}{\Delta_{L,k}^2} + \frac{8\Delta_{1,L} m_k\log T}{\Delta_{L-1,L}^2} + 4KN.
	\end{equation}
	



	Finally, from Eq.(\ref{eq:etc_ucb_etcphase}) and Eq.(\ref{eq:etc_ucb_ucbphase}), we obtain \texttt{ETC-UCB}'s regret upper bound as follows
	\begin{equation*}
		\begin{split}
			\ERT \le \sum_{k=L+1}^K\frac{8\Delta_{1,k} m_k\log T}{\Delta_{L,k}^2} + \frac{8\Delta_{1,L} m_k\log T}{\Delta_{L-1,L}^2}  + \sum_{k=1}^K \frac{49m_k^2 w_k}{\mu_k^2}\log{T}
			 + 6KN
		\end{split}
	\end{equation*}
	which confirms our statement in Appendix~\ref{sec:se_algorithm}'s Design Overview.
\end{proof}

\section{Addition Evaluation}\label{app:addition-simulation}

\subsection{Real World Application in 5G \& 4G Base Station Selection}\label{app:real_world_simulation}

In this section, we consider a real-world 5G \& 4G base station selection application
and show how our algorithms can be applied to it.
Since 2019, 5G base stations started to serve consumers and will coexist with
4G base stations for a long time.
5G and 4G base stations' performance were measured in~\citet{narayanan2020first}.
They shown 5G station's throughput (THR) is about \(8\) times higher than 4G stations',
and 5G station's round-trip time (RTT) latency is 4 times shorter than 4G stations'.
From \citet{narayanan2020first}'s results, we consider
a real-world scenario which contains two 5G base stations (underlined) and eighteen 4G base stations (in total \(K=20\))
and eighteen smartphones (\(N=18\)).
Their parameters are in Table~\ref{tab:5g-4g}.
Each base station is regarded as one arm,
and each smartphone phone is represented as a play.
Base stations' RTT latencies' reciprocals are mapped to arms' ``per-load'' Bernoulli reward means.
A station's throughput (THR) is rounded to their closed integer as the arm's finite reward capacity.

\begin{table}[htp]
	\centering
	\caption{The 5G \& 4G Base Station Selection Environment}
	\label{tab:5g-4g}
	\begin{tabular}{|c||cccccccccc|}
		\hline
		RTT (100ms)   & \underline{{1.2}} & \underline{{1.1}} & 4.2 & 4.9 & 4.5 & 3.4 & 5.0 & 4.2 & 5.1 & 3.9 \\ \hline
		THR (100Mbps) & \underline{{8.2}} & \underline{{8.1}} & 1.2 & 1.2 & 1.4 & 1.1 & 1.3 & 1.2 & 1.1 & 1.4 \\
		\hline
		RTT (100ms)   & 4.8               & 5.7               & 3.7 & 4.7 & 3.2 & 5.1 & 4.4 & 5.1 & 4.9 & 4.1 \\ \hline
		THR (100Mbps) & 1.0               & 1.1               & 1.2 & 1.0 & 1.3 & 1.2 & 1.0 & 1.1 & 1.3 & 1.2 \\
		\hline
	\end{tabular}
\end{table}

We apply our three algorithms \texttt{OrchExplore}, \texttt{MP-SE-SA} (\(\gamma=0.1\)), and \texttt{ETC-UCB} to the scenario.
Their performance is in Figure~\ref{fig:5g-4g-selection}.
Other implicitly-learning-capacity algorithms
--- regard
each \(N\)-play allocation (action)
as an independent arm ---
is infeasible in this scenario.
Because the total number of these combinatorial actions is greater than \(10^9\)!
Figure~\ref{fig:5g-4g-selection} shows all of our three algorithms achieve the sub-linear regret performance.
From the total throughput aspect, the \texttt{OrchExplore} algorithm outperforms \texttt{MP-SE-SA} in a moderate degree,
while both are much better than the \texttt{ETC-UCB} two-phase algorithm.

\begin{figure}[htb]
	\centering
	\subfloat[Regret\label{subfig:5g-4g-regret}]{\includegraphics[width=0.25\columnwidth]{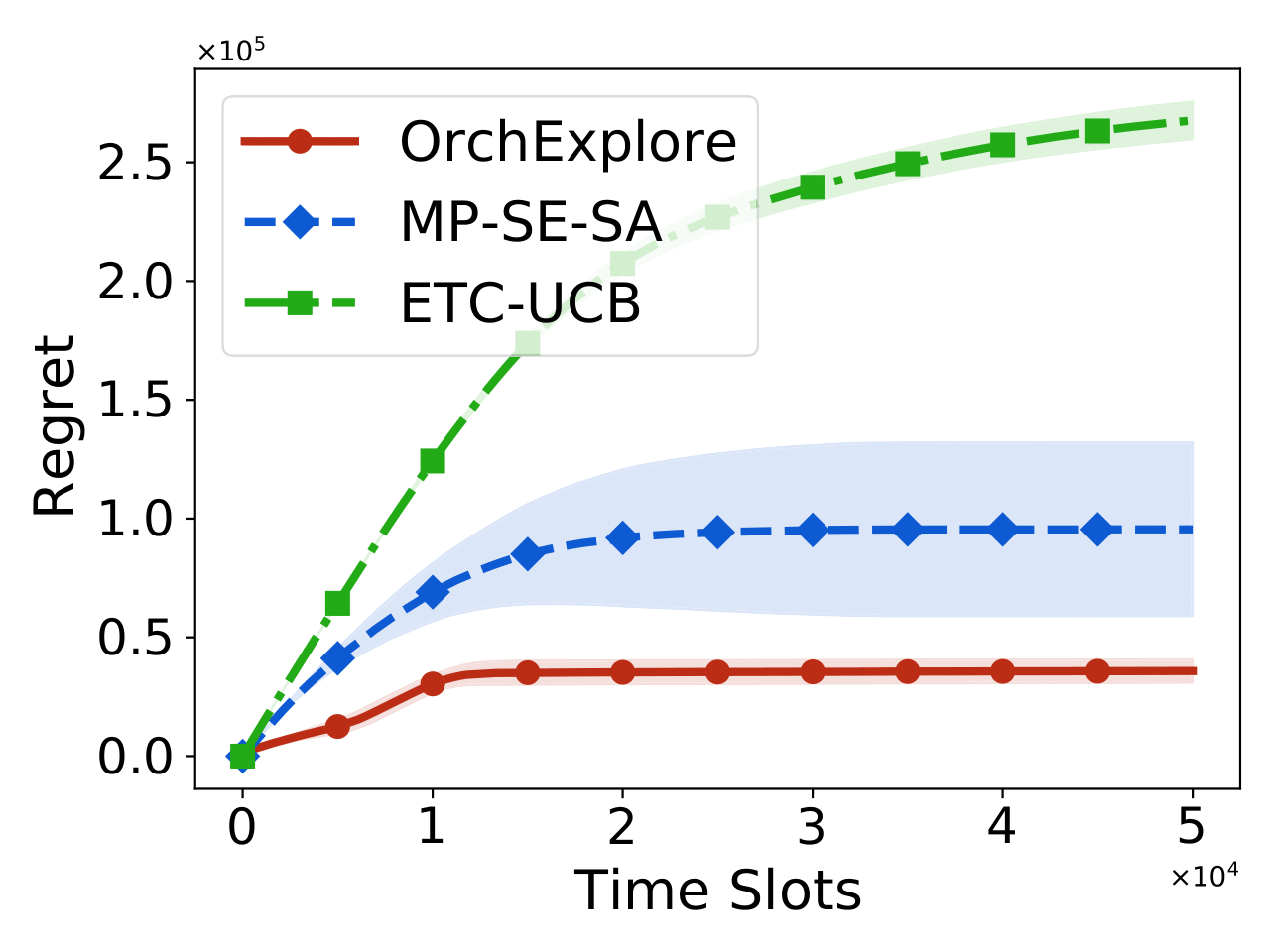}}
	\subfloat[Total Throughput\label{subfig:5g-4g-reward}]{\includegraphics[width=0.25\columnwidth]{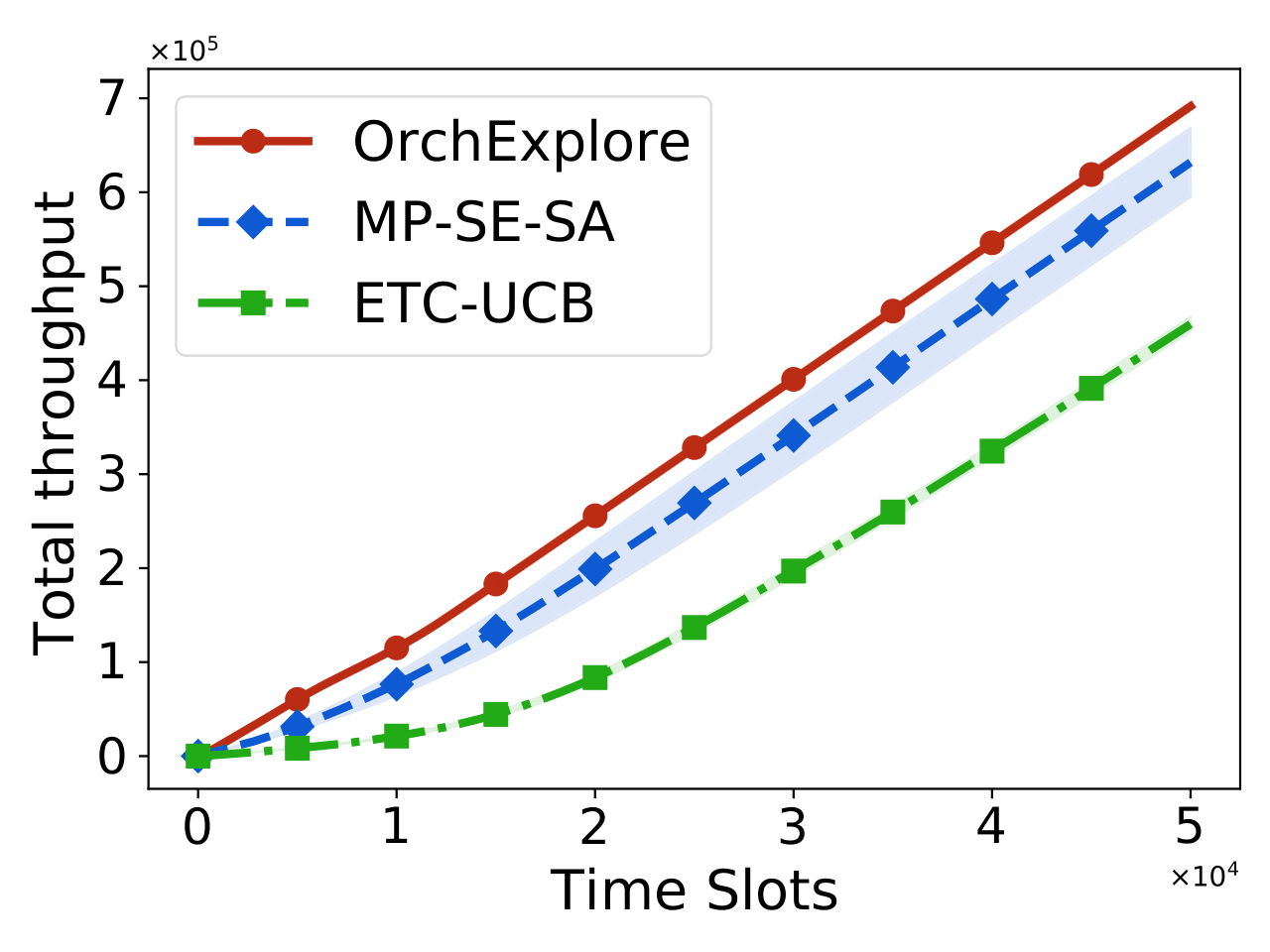}}
	\caption{The 5G \& 4G Base Station Selection}
	\label{fig:5g-4g-selection}
\end{figure}

\subsection{In Gaussian Distributions with \(1/2\) Variance}\label{appsub:guassian-simulation}
In Figure~\ref{fig:eval-Gaussian}, we present the simulation results of Gaussian ``per-load'' reward case
under the same parameters as Section~\ref{sec:simulation}.
It is a complement of Section~\ref{sec:simulation}'s Bernoulli ``per-load'' reward evaluations.
The Gaussian reward causes larger variance than the Bernoulli case.
Their average regret performance is similar.
That validates Section~\ref{sec:simulation}'s evaluation insights.
\begin{figure}[htb]
	\centering
	\subfloat[\texttt{OrchExplore} \textit{vs.} \texttt{MP-SE-SA} \textit{vs.} \texttt{ETC-UCB}\label{subfig:fine_grained_update-Gaussian}]{\includegraphics[width=0.25\columnwidth]{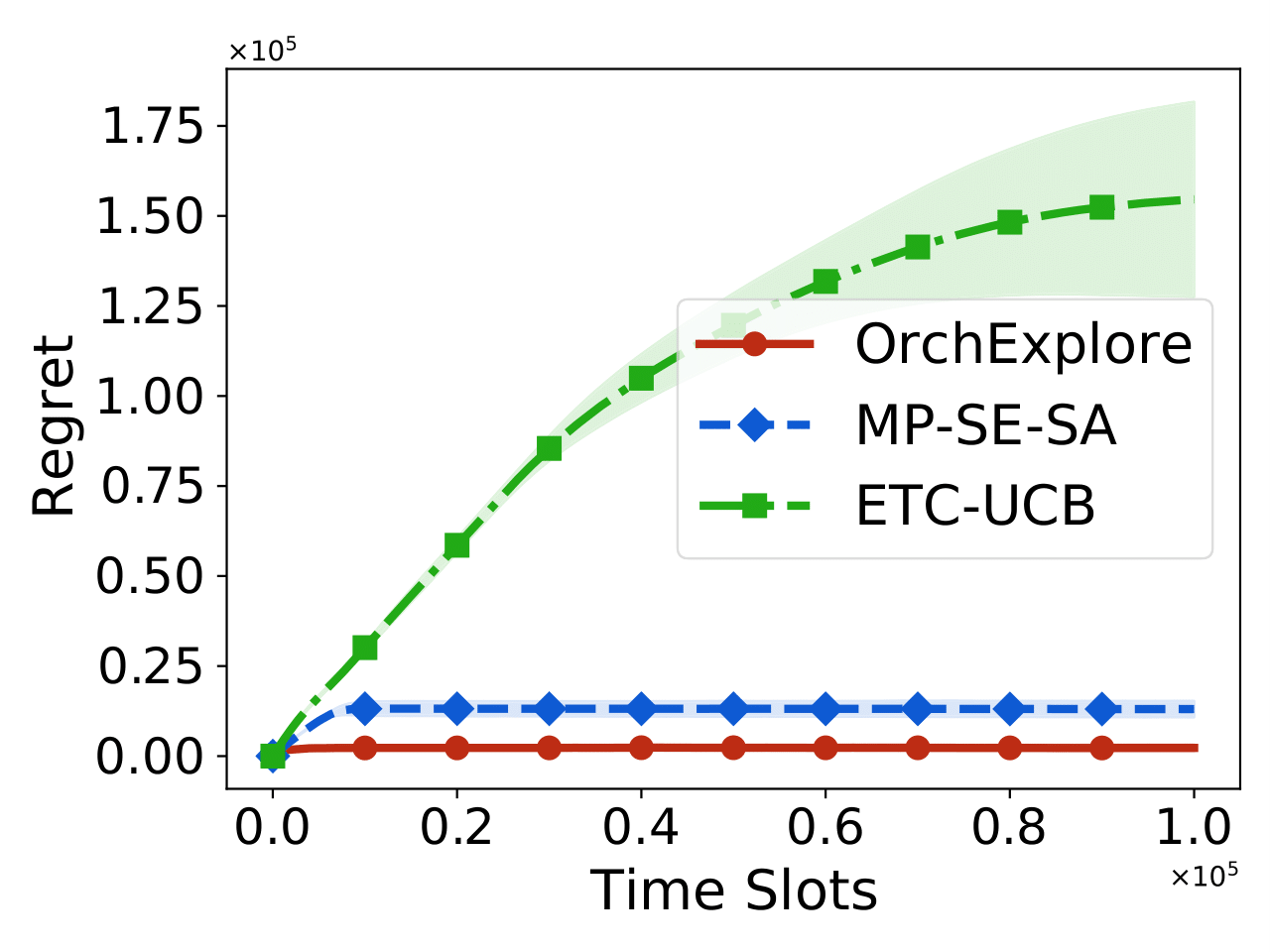}}
	\subfloat[Improvement of UCI\label{subfig:hfd_vs_uci-Gaussian}]{\includegraphics[width=0.25\columnwidth]{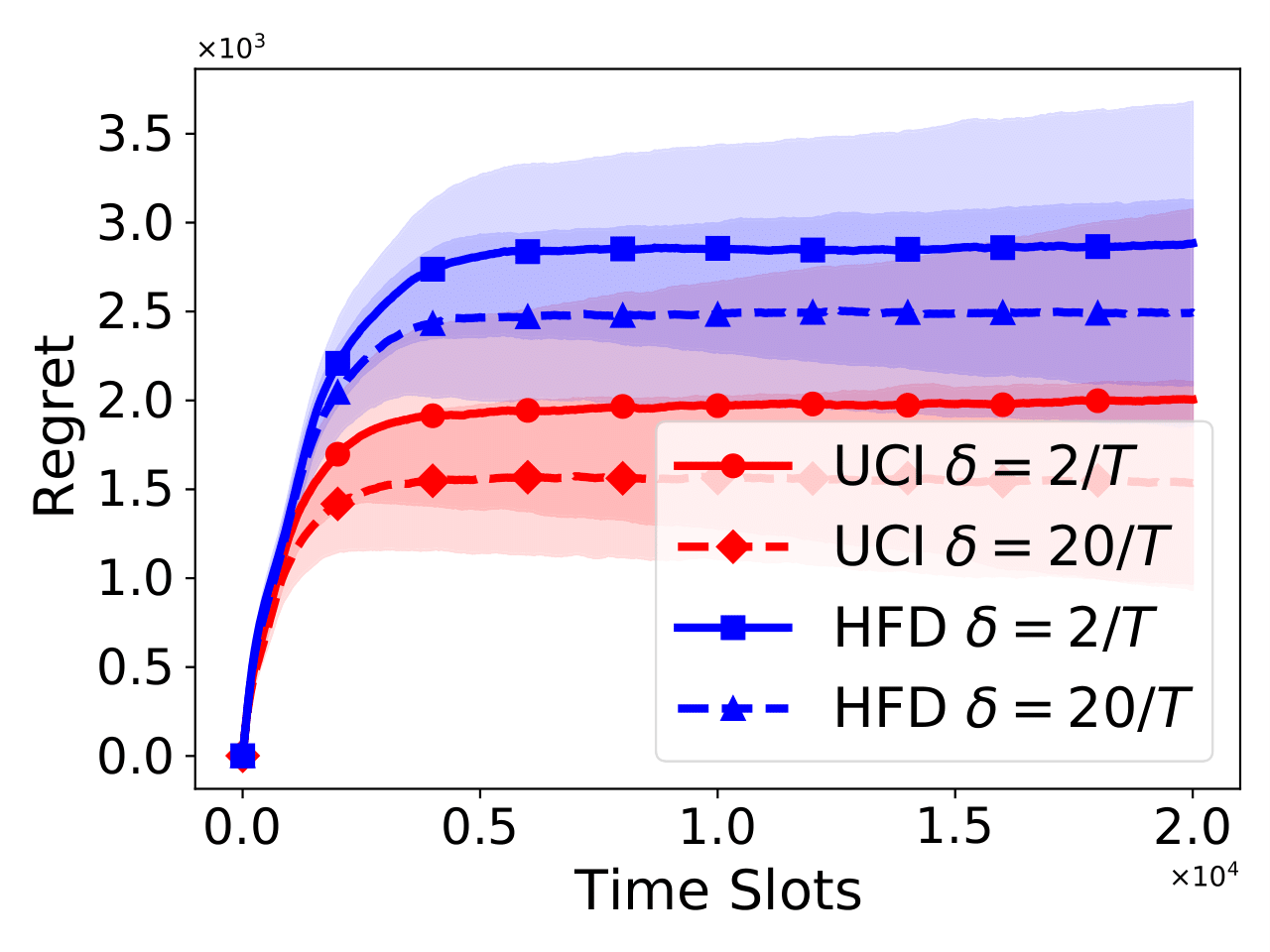}}
	\subfloat[Price of learning $m_k$\label{subfig:know_or_not-Gaussian}]{\includegraphics[width=0.25\columnwidth]{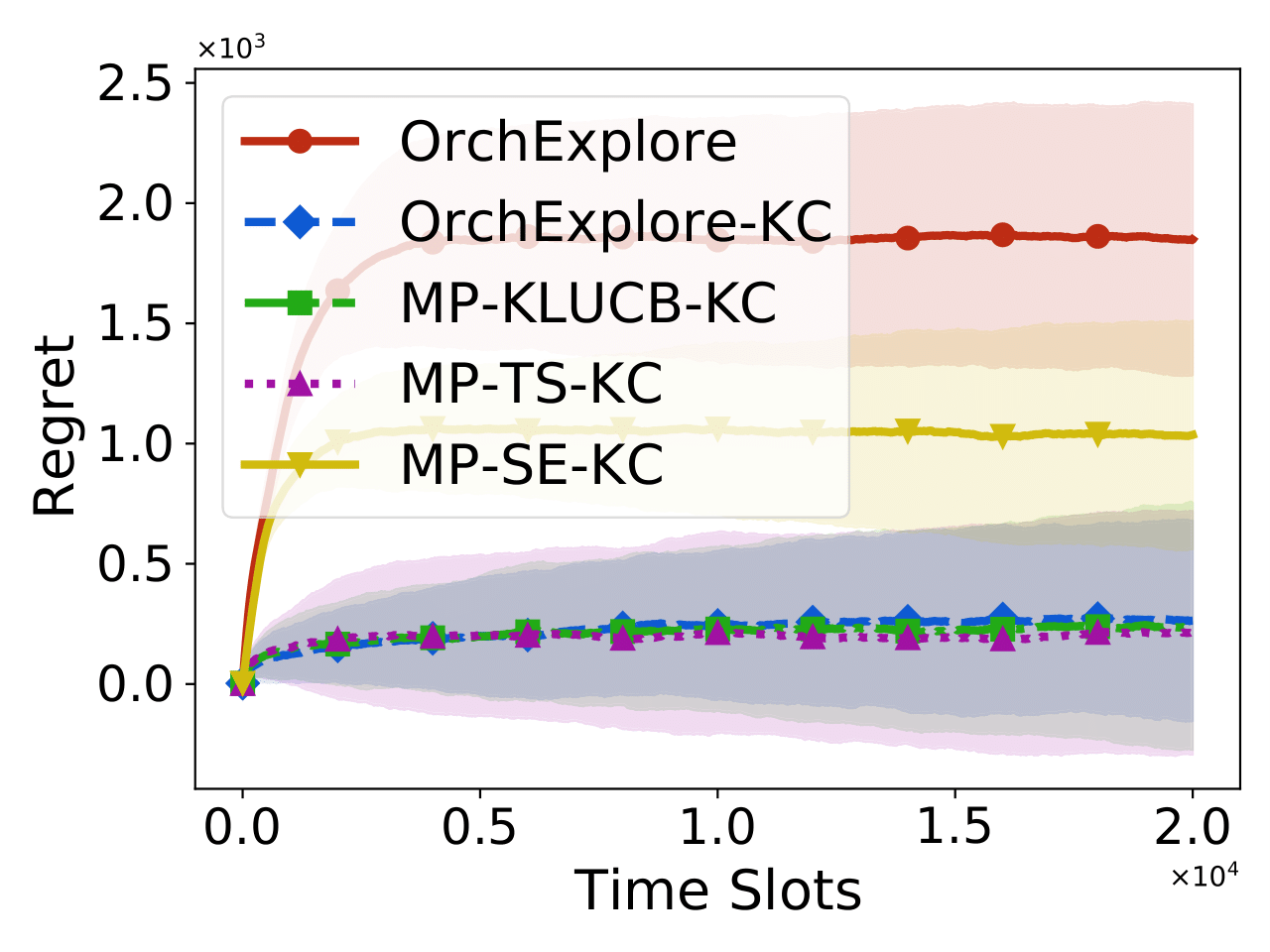}}
	\subfloat[Implicitly learn $m_k$\label{subfig:learn_or_not-Gaussian}]{\includegraphics[width=0.25\columnwidth]{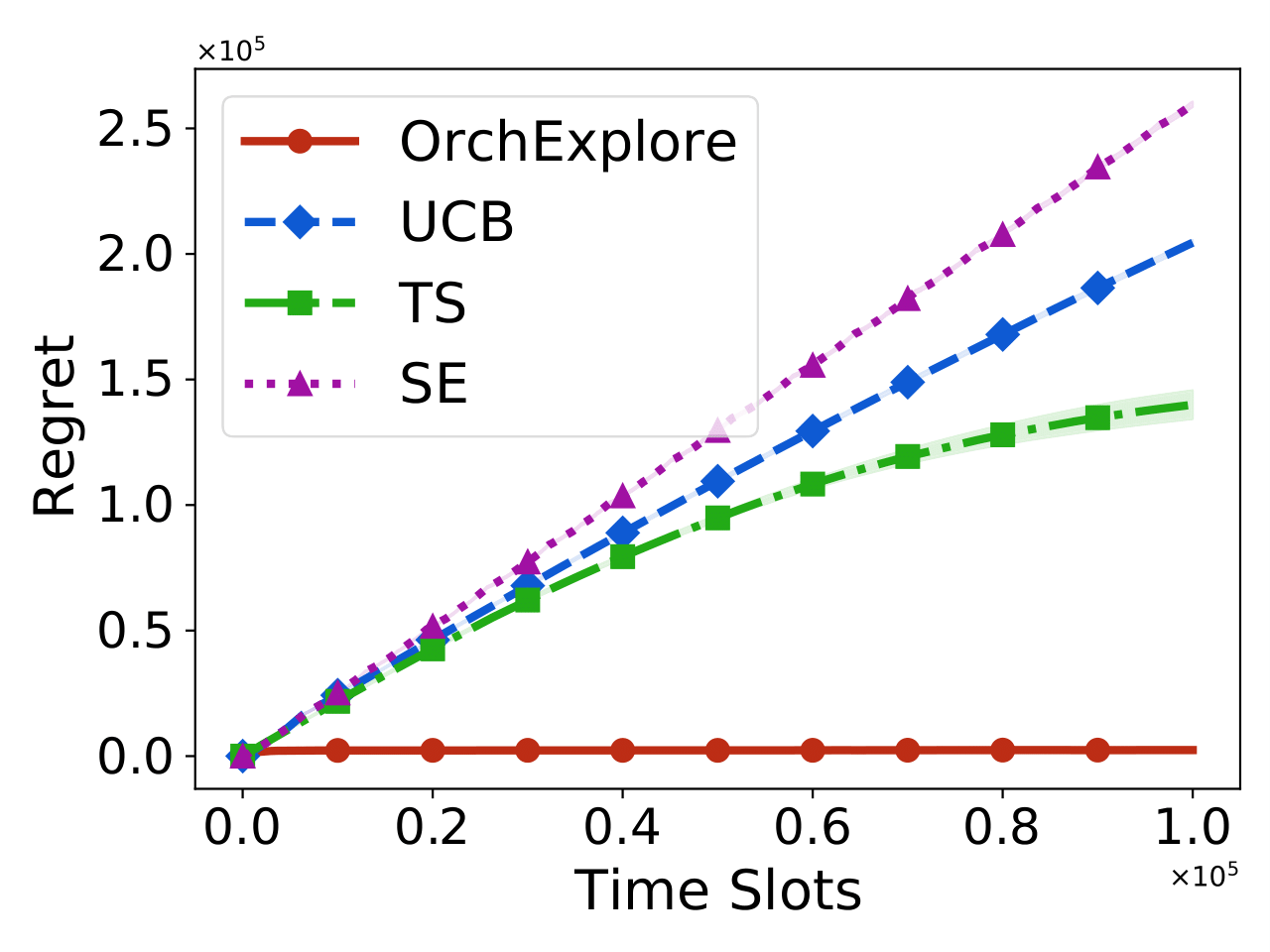}}
	\caption{Evaluation under Gaussian Distributions with \(1/2\) Variance}
	\label{fig:eval-Gaussian}
\end{figure}

\section{Hoeffding's Inequality Based Confidence Interval Design}\label{sec:hfd_ci}


If replacing the uniform concentration inequality in Lemma~\ref{lma:uci_basic} with Hoeffding's inequality,
we can obtain the following three results. Each of them corresponds to our UCI results. There are two key differences:
(1)the $\phi(x,\delta)$ function of UCI is replaced by $\rho(x, \delta)$ defined in Lemma~\ref{lma:ici_hfd};
(2) Lemma~\ref{lma:ici_hfd} is an instantaneous confidence interval only holding for one single pair of $(\tau_{k,t},\iota_{k,t})$.
Their proofs are almost the same as Appendix~\ref{app:capacity_proof}'s.

\begin{lemma}\label{lma:ici_hfd}
	Denote the function
	$
		\rho(x,\delta)\triangleq \sqrt{\log (2/\delta)/ 2x},
	$
	When $\rho(\tau_{k,t},\delta) + \rho(\iota_{k,t},\delta) < \mu_k$, the event
	\begin{equation*}
		\left\{m_k \in
		\left[\frac{\hat{\nu}_{k,t}}{\hat{\mu}_{k,t} + \rho(\tau_{k,t},\delta) + \rho(\iota_{k,t},\delta)},
			\frac{\hat{\nu}_{k,t}}{\hat{\mu}_{k,t} - \rho(\tau_{k,t},\delta) - \rho(\iota_{k,t},\delta)}\right] \right\}
	\end{equation*}
	holds with probability of at least $1-\delta$.
\end{lemma}

\begin{lemma}
	For any arm $k$, if
	\[
		\ceil{\frac{\hat{\nu}_{k,t}}{\hat{\mu}_{k,t} + \rho(\tau_{k,t},\delta) + \rho(\iota_{k,t},\delta)}} =
		\floor{\frac{\hat{\nu}_{k,t}}{\hat{\mu}_{k,t} - \rho(\tau_{k,t},\delta) - \rho(\iota_{k,t},\delta)}},
	\]
	then
	the probability of correctly estimating $m_k$ is at least $1-\delta$, i.e.,
	$
		\mathbb{P}({\hat{m}_{k,t}} = m_k) \ge 1 - \delta.
	$
\end{lemma}

\begin{corollary}
	For any arm $k$, if $\tau_{k,t}$ and $\iota_{k,t}$ satisfy
	\[
		\tau_{k,t}, \iota_{k,t} \geq
		(49m_k^2/2\mu_k^2)\log (2/\delta),
	\]
	then it hold that
	$
		\mathbb{P}({\hat{m}_{k,t}} = m_k) \ge 1 - \delta.
	$
\end{corollary}

We note that the above sample complexity upper bound only guarantees for one pair of $(\tau_{k,t},\iota_{k,t})$ while Lemma~\ref{cor:m_sample_complexity}'s is for \emph{all} pairs of $(\tau_{k,t},\iota_{k,t})$.
When comparing them, we need to convert Lemma~\ref{lma:ici_hfd} to uniform version, that is, replacing $\rho(x, \delta)$
with $\rho(x, \delta/T)$. In Figure~\ref{fig:width_comparison}, we compare function $\phi(t,T^{-1})$ and function $\rho(t, T^{-2})$'s decreasing rate. That implies our UCI has a sharper concentration.

\begin{figure}[htb]
	\centering
	\includegraphics[width=0.35\textwidth]{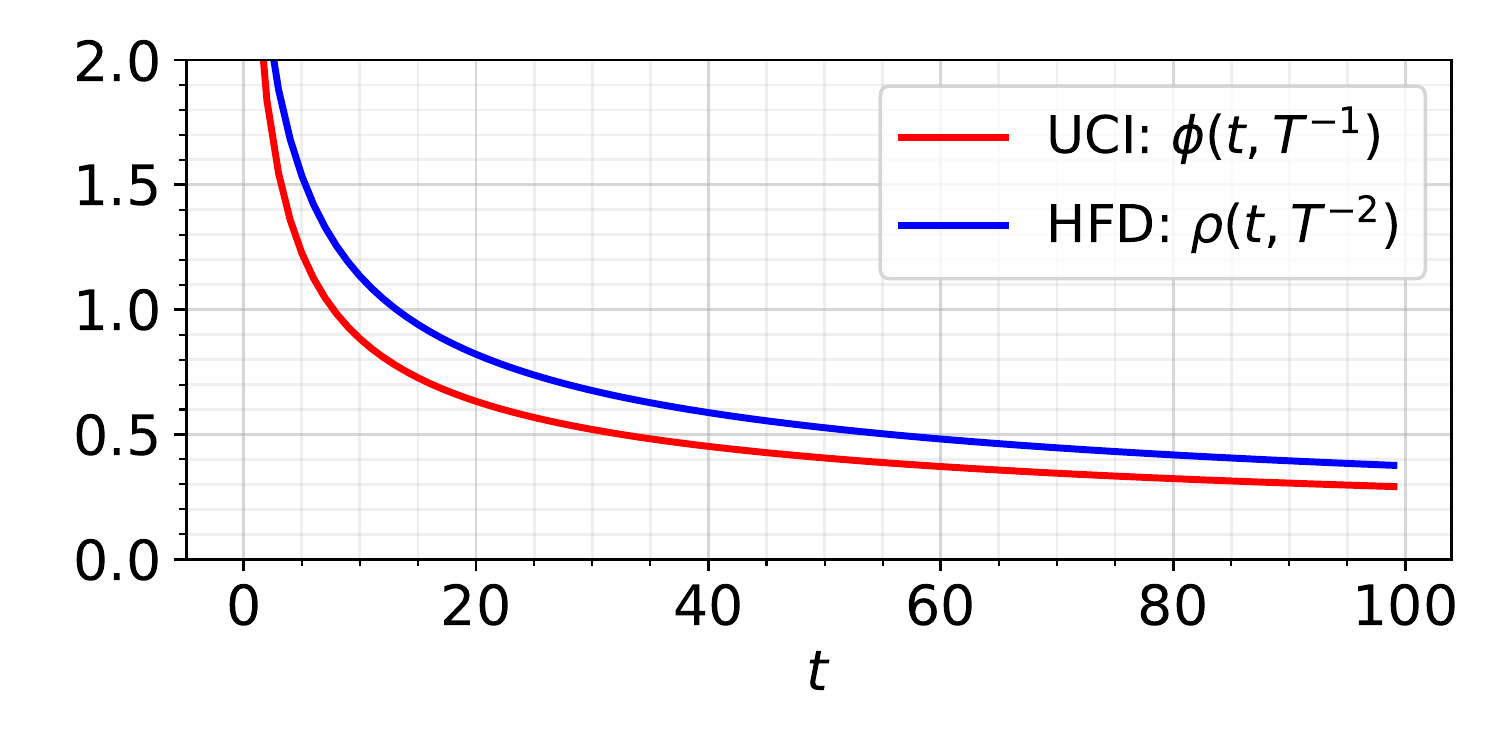}
	\caption{Our UCI's $\phi(t,T^{-1})$ v.s. the one based on Hoeffding's inequality (HFD)'s $\rho(t,T^{-2})$ ($T=10^{6}$).}
	\label{fig:width_comparison}
\end{figure}

\end{document}